\documentclass[twoside]{article}

\newif\ifdraft
\drafttrue
\draftfalse

\newif\ifaccepted
\acceptedfalse

\newif\ifarxiv
\arxivtrue

\ifarxiv
    \newlength\titlebox 
 \twocolumn 
    \usepackage{mdframed}
\else
  \ifaccepted
    \usepackage[accepted]{aistats2022}
    \draftfalse
    \else
    \usepackage{aistats2022}
  \fi
\fi

\usepackage[round]{natbib}

\usepackage{dsfont}
\usepackage[usenames,dvipsnames,svgnames,table]{xcolor}
\usepackage{cancel}
\usepackage{graphicx}

\ifdraft\ifarxiv\else
\fi\fi

\usepackage{amsmath,amssymb}
\usepackage{amsthm}
\usepackage{ifthen}
\usepackage{hyperref}

\usepackage{bm}

\usepackage{algorithm}
\usepackage{subcaption}
\usepackage{listings}

\usepackage{letltxmacro}
\newcommand*{\SavedLstInline}{}
\LetLtxMacro\SavedLstInline\lstinline
\DeclareRobustCommand*{\lstinline}{
  \ifmmode
    \let\SavedBGroup\bgroup
    \def\bgroup{
      \let\bgroup\SavedBGroup
      \hbox\bgroup
    }
  \fi
  \SavedLstInline
}
\newcommand{\code}[1]{\texttt{\lstinline[mathescape,classoffset=1,keywordstyle=\color{black},basicstyle=\color{black},classoffset=0,keywordstyle=\color{black}]{#1}}}
\newcommand{\lstcommentcolor}[1]{\textcolor{darkorange}{#1}}

\usepackage[capitalize]{cleveref}
\crefname{listing}{Algorithm}{Algorithms}
\Crefname{listing}{Algorithm}{Algorithms}

\theoremstyle{definition}

\newcommand{\mynewtheorem}[3]{
\ifthenelse{\equal{#1}{theorem}}{
    \newtheorem{my#1}{#2}
}{
    \newtheorem{my#1}[mytheorem]{#2}
    \ifthenelse{\equal{#3}{}}{
        \crefname{my#1}{#2}{#2s}
    }{
        \crefname{my#1}{#2}{#3}
    }
}
}

\mynewtheorem{theorem}{Theorem}{}
\mynewtheorem{lemma}{Lemma}{}
\mynewtheorem{corollary}{Corollary}{Corollaries}
\mynewtheorem{definition}{Definition}{}
\mynewtheorem{assumption}{Assumption}{}
\mynewtheorem{proposition}{Proposition}{}
\mynewtheorem{remark}{Remark}{}
\mynewtheorem{example}{Example}{}
\mynewtheorem{exercise}{Exercise}{}

\newenvironment{theorem}
  {\pushQED{\qed}\mytheorem}
  {\popQED\endmytheorem}
\newenvironment{lemma}
  {\pushQED{\qed}\mylemma}
  {\popQED\endmylemma}
\newenvironment{corollary}
  {\pushQED{\qed}\mycorollary}
  {\popQED\endmycorollary}
\newenvironment{proposition}
  {\pushQED{\qed}\myproposition}
  {\popQED\endmyproposition}

\newenvironment{definition}
  {\pushQED{\qed}\mydefinition}
  {\popQED\endmydefinition}
\newenvironment{assumption}
  {\pushQED{\qed}\myassumption}
  {\popQED\endmyassumption}
\newenvironment{remark}
  {\pushQED{\qed}\myremark}
  {\popQED\endmyremark}
\newenvironment{example}
  {\pushQED{\qed}\myexample}
  {\popQED\endmyexample}

\definecolor{darkred}{rgb}{.5,0,0}
\definecolor{darkgreen}{rgb}{0,.5,0}
\definecolor{darkblue}{rgb}{0,0,.5}
\definecolor{darkorange}{rgb}{.8,.4,0}
\ifdraft
\newcommand{\todo}[1]{\textcolor{darkorange}{(\emph{TODO: #1})}}
\newcommand{\comment}[1]{\textcolor{gray}{(\emph{#1})}}
\newcommand{\warning}[1]{\textcolor{red}{(\emph{WARNING: #1})}}
\newcommand{\quest}[1]{\textcolor{darkgreen}{(\emph{Q: #1})}}

\else
\newcommand{\todo}[1]{}
\newcommand{\comment}[1]{}
\newcommand{\warning}[1]{}
\newcommand{\quest}[1]{}

\fi

\newcounter{alphoversetcount}
\newcommand{\alphnextref}{\stepcounter{alphoversetcount}\text{(\alph{alphoversetcount})}}
\newcommand{\alphoverset}[1]{\overset{\alphnextref{}}{#1}}
\newcommand{\resetalph}{\setcounter{alphoversetcount}{0}}

\newenvironment{alphalign*}{
\csname align*\endcsname\resetalph{}
}{
\csname endalign*\endcsname\resetalph{}
}

\newcommand{\eg}{\emph{e.g.}, }

\DeclareMathOperator*{\expect}{\mathbb{E}}
\DeclareMathOperator*{\argmin}{argmin}
\DeclareMathOperator*{\arginf}{arginf}
\DeclareMathOperator*{\argmax}{argmax}

\newcommand{\Naturals}{{\mathbb N}_1}
\newcommand{\Nonnegints}{{\mathbb N}_0}
\newcommand{\Reals}{{\mathbb R}}

\newcommand{\indicator}[1]{\left[\!\left[#1\right]\!\right]}

\newcommand{\tuple}[1]{(#1)}
\newcommand{\innerprod}[1]{\langle #1 \rangle}
\newcommand{\probsimplex}[1]{\Delta_{\text{p}}(#1)}
\newcommand{\positive}[1]{\left(#1\right)_{\!+}}

\newcommand{\RE}[2]{\text{RE}(#1\|#2)}

\newcommand{\eps}{\varepsilon}

\DeclareMathOperator*{\iso}{iso}

\newcommand{\rate}{\eta}
\newcommand{\loss}{\ell}

\newcommand{\invrate}{\rate^{-1}}

\newcommand{\QQ}{q}
\newcommand{\diam}{D}

\newcommand{\ptx}{\bm{x}}

\newcommand{\ptxset}{\mathcal{X}}

\newcommand{\barT}{\mathcal{T}_{\emptyset}}

\newcommand{\breg}{B}

\newcommand{\regret}{\mathcal{R}}

\DeclareMathOperator*{\projection}{\ensuremath{\Pi}}
\newcommand{\proj}[2]{\projection_{#1}\left[#2\right]}

\begin{document}

\ifarxiv
    \title{Isotuning With Applications To Scale-Free Online Learning}
    \author{Laurent Orseau \and Marcus Hutter}
    \date{\{lorseau,mhutter\}@google.com\\
    DeepMind, London, UK}
    \maketitle
\else
    \twocolumn[
    \aistatstitle{Isotuning with Applications to Scale-Free Online Learning}
    \aistatsauthor{Laurent Orseau \And Marcus Hutter}
    \aistatsaddress{DeepMind, London, UK \And  DeepMind, London, UK} ]
\fi

\begin{abstract}
We extend and combine several tools of the literature to design fast, adaptive, anytime and scale-free online learning algorithms.
Scale-free regret bounds must scale linearly with the maximum loss, both toward large losses and toward very small losses.
Adaptive regret bounds demonstrate that an algorithm can take advantage of easy data and potentially have constant regret.
We seek to develop fast algorithms that depend on as few parameters as possible,
in particular they should be anytime and thus not depend on the time horizon.
Our first and main tool, \emph{isotuning}, is a generalization of the idea of
designing adaptive learning rates that balance the trade-off of the regret. 
We provide a simple and versatile theorem that can be applied to a wide range of settings, and competes with the best balancing in hindsight within a factor 2.
The second tool is an \emph{online correction}, which allows us to obtain centered bounds for many algorithms, to prevent the regret bounds from being vacuous when the domain is overly large or only partially constrained.
The last tool, \emph{null updates}, prevents the algorithm from performing overly large updates, which could result in unbounded regret, or even invalid updates.
We develop a general theory to combine all these tools and apply it to several standard algorithms.
In particular, we (almost entirely) restore the adaptivity to small losses of FTRL for unbounded domains, design and prove scale-free adaptive guarantees for a variant of Mirror Descent (at least when the Bregman divergence is convex in its second argument), extend Adapt-ML-Prod to scale-free guarantees, and provide several additional contributions about Prod, AdaHedge, BOA and Soft-Bayes.
\end{abstract}

\section{INTRODUCTION}

\warning{DRAFT MODE}

In online convex optimization~\citep{hazan2016oco}, at each round $t=1, 2, \dots$, the learner selects a point $\ptx_t\in\ptxset\subseteq \Reals^N$ and incurs a loss
$\loss_t(\ptx_t)$, with $\loss_t:\ptxset\to\Reals$ convex.
We also use $\loss_t\in\Reals^N$ as a loss vector, with incurred loss $\innerprod{\ptx_t, \loss_t}$ when considering online linear optimization.
The objective of the learner is to minimize its cumulative loss over $T$ rounds (with $T$ unknown in advance), and 
the difference with the cumulative loss incurred by the best constant point $\ptx^*\in\ptxset$ in hindsight is called the regret.
Consider the following simplified regret bound (\eg online gradient descent, \citet{Zin03}) featuring a time-varying learning rate $\{\rate_t\}_{t\in[T]}$:
\begin{align}\label{eq:regret_intro}
    \regret_T(\ptx^*) \leq \frac{\QQ}{\rate_T} + \sum_{t\in[T]} \rate_t a_t\,. 
\end{align}
We would like to design a sequence of (adaptive) learning rates $\rate_1, \rate_2, \dots$ without knowledge of the time horizon $T$ and with the $a_t$ discovered sequentially.
The usual process of the researcher is to first replace the learning rate with a constant one
which minimizes the trade-off $\QQ/\rate + \rate \sum_t a_t$, that is, $\rate = \sqrt{\QQ/\sum_t a_t}$ in hindsight.
The second step is to make it time varying: $\rate_t = \sqrt{\QQ/\sum_{s\leq t} a_s}$.
The last step is to analyze the regret, which often involves ad-hoc lemmas such as $\sum_t a_t/\sqrt{\sum_{s\leq t} a_s} \leq 2\sqrt{\sum_t a_t}$ (\eg \citet[Lemma 3.5]{auer2002adaptive}, \citet[Lemma 3]{orabona2018solo}, \citet[Lemma 4.8]{pogodin20first}).
Instead, following the ideas of \citet{bartlett2007adaptive,derooij2014follow}, 
we tune the learning rate so as to sequentially \emph{equate} the two terms of the regret --- in this particular case, $\QQ/\rate_T = \sum_t \rate_t a_t$.
Hence, instead of having to bound terms of the regret in interaction with a time-varying learning rate, we now only have to bound $2\QQ/\rate_T$, that is, only the last learning rate.
It can be shown that this provides a bound that is within a factor 2 of the bound that uses the constant optimal learning rate for the same observed quantities, independently of the regime of the regret.

In this paper we are particularly interested in designing learning algorithms with the following properties:
(i) anytime: the horizon $T$ is unknown,
(ii) scale-free: the losses are in $\Reals^N$ for $N\in\Naturals$ are only sequentially revealed, the range of the losses is unknown in advance, the regret bound must scale linearly with the maximum loss (whether large or very small) or the loss range (on the simplex),
(iii) adaptive: the regret bound should feature a main term like $\sqrt{\sum_t \|\loss_t\|^2}$ (instead of $L\sqrt{T}$ with $L=\max_t \|\loss_t\|$) so that a smaller bound can be shown for easy problems and the algorithm converges faster,
(iv) fast: we are only interested in algorithms that have $O(N)$ computation complexity per round.

While a desirable property, scale-free bounds and algorithms are not so easy to obtain (see for example the discussions in \citet{orabona2018solo,mhammedi2020freerange}):
in particular even a constant additive term makes the bound vacuous for very small 
losses, as $LT$ may be much smaller than 1.
Prod and AdaHedge are possibly the first adaptive scale-free algorithms for the probability simplex~\citep{cesabianchi2007prod,derooij2014follow}.
Important progress has been made by \citet{orabona2018solo},
who propose a Follow The Regularized Leader (FTRL) algorithm~\citep{cesabianchi2006prediction,shalev2007online} with a scale-free bound for any convex regularizer for both bounded and unbounded domains.
Its main drawback is that while it is adaptive to small loss vectors $\loss_t$ with a bound of the form $O(\sqrt{\sum_t \|\loss_t\|_*^2})$ for constrained domains,
this adaptivity is lost for unconstrained domains and the bound degrades to $O(L\sqrt{T})$.
\citet{orabona2018solo} also propose a scale-free variant of Mirror Descent (MD)~\citep{beck2003mirror}, but it requires a bound on the maximum Bregman divergence between any two points, which excludes the most interesting cases (`AdaHedge'-like and unconstrained `gradient descent'-like)---they also provide lower bounds on this algorithm demonstrating its failure cases.
(See more discussion in \cref{sec:related}.)

In this paper we combine and improve several tools of the literature (local regret balancing, online correction, null updates) to help design algorithms that fit our requirements detailed above.
In particular, we generalize the local balancing ideas of \citet{bartlett2007adaptive,derooij2014follow} into the isotuning framework,
which tools allow us to simplify and often strengthen the analysis of various algorithms, as well as analyze more adaptive learning rates --- see in particular the isoSoft-Bayes analysis in \cref{apdx:soft-bayes}, which would likely have been difficult without isotuning.
We derive a new FTRL algorithm with scale-free regret bounds, both for constrained and unconstrained domains and restore \emph{almost} entirely the adaptivity to small losses for the unconstrained case:
Instead of a $O(L\sqrt{T})$ additive term, we obtain a $O(L\sqrt{\tau})$ one,
where $\tau$ is the last step where the loss exceeds the \emph{sum of all} previously observed losses.
While in the worst case this term is still $O(L\sqrt{T})$, it remains $O(1)$ as long as losses do not keep increasing exponentially
(\eg i.i.d.\! losses).
In combination with isotuning, this enhancement is obtained by performing \emph{null updates}, which prevents the algorithm from performing large updates when the scale of the losses suddenly increases.
Such skipping of the updates has been used recently to tackle overly long feedback delays for multi-armed bandits \citep{thune2019delays,zimmert2020delays} --- but not to obtain scale-free properties.
We also derive a Mirror-Descent algorithm with the same properties as our FTRL variant for both constrained and unconstrained domains,
at least for Bregman divergences that are convex in their second argument---this includes the `AdaHedge' and unconstrained gradient descent cases.
This is made possible by generalizing the usage of the \emph{online correction} 
of \citet{orseau2017softbayes},
which `centers' the regret bound and replaces the term $\breg_R(\ptx^*, \ptx_T)$ with $\breg_R(\ptx^*, \ptx_1)$.
With the same tools we also derive new (anytime, adaptive) scale-free regret bounds for Prod~\citep{cesabianchi2007prod} (\cref{sec:prod}), Adapt-ML-Prod~\citep{gaillard2014secondorder} (\cref{sec:mlu-prod}),
BOA~\citep{wintenberger2016bernstein} (\cref{rmk:boa} in \cref{apdx:mlu-prod}).
We also provide a more adaptive bound for the Soft-Bayes algorithm~\citep{orseau2017softbayes} (\cref{apdx:soft-bayes}),
for which isotuning eliminates the need to track the set of `good' experts.
We also design a variant of AdaHedge that uses a Mirror-Descent-like update rather than the original FTRL-like update (\cref{apdx:adahedge}),
and even slightly improve the regret bound.
All algorithms retain their original time complexity of $O(N)$ per round.

To do so, we first develop a general theory (\cref{sec:theory}) and build our set of tools. We start with isotuning (\cref{sec:isotuning}), 
and demonstrate its use on a few examples, 
then introduce the online correction (\cref{sec:generic_ob1}
while using gradient descent as a rolling example,
followed by null updates (\cref{sec:null_updates}), before applying these tools to several algorithms.
We finish with more discussion about the literature.

\section{NOTATION AND BACKGROUND}

A table of notation is given on page \pageref{sec:notation_table}.
The set of positive integers is $\Naturals$,
and we write $[T] = \{1, 2, \dots, T\}$.
Define $\positive{x} = \max\{0, x\}$.
The scalar product of two vectors $a$ and $b$ is denoted $\innerprod{a, b}$.
Define $\probsimplex{N} = \{\ptx\in[0, 1]^N: \|\ptx\|_1=1\}$ be the probability simplex of $N$ coordinates, with $N\in\Naturals$.
Let $\ptxset\subseteq \Reals^N$ be a non-empty closed convex set.
Let $\ptx^* \in\ptxset$ be a fixed point chosen in hindsight.
Let $\ptx_t \in \ptxset$ for all $t\in\Naturals$ be the sequence of points chosen by the learner.
In online convex optimization, the loss function at step $t$ is $\loss_t:\ptxset\to\Reals$ and the instantaneous regret is $r_t = \loss_t(\ptx_t) - \loss_t(\ptx^*)$,
while in online linear optimization
the loss is a vector $\loss_t \in\Reals^N$ and the instantaneous regret is 
$r_t = \innerprod{\ptx_t - \ptx^*, \loss_t}$.
The cumulative regret up to step $T$ is $\regret_T(\ptx^*) = \sum_{t\in[T]} r_t$, which is the quantity of interest to minimize.
Let $\phi: \ptxset\to\Reals$ be a convex function, used for
analyzing the regret.
We abbreviate $\phi_t \equiv \phi(\ptx_t)$.
Let $\|\cdot\|$ be some norm on $\Reals^N$, then:
$\|\cdot\|_*$ is its dual norm;
$\diam = \sup_{x, y, \in\ptxset}\|x-y\|$ is the diameter of $\ptxset$;
$R:\ptxset\to\Reals^+$ is a regularizer function that is differentiable and 1-strongly convex with respect to $\|\cdot\|$
and $R^*$ is its Fenchel conjugate;
$\breg_R$ is a Bregman divergence associated with $R$, that is, 
$\forall (\ptx, \ptx')\in\ptxset^2,\breg_R(\ptx, \ptx') = R(\ptx) - R(\ptx') - \innerprod{\ptx - \ptx', \nabla R(\ptx')}.$
The relative entropy (Kullback-Leibler divergence)
of $(a, b)\in\probsimplex{N}^2$
is $\RE{a}{b} = \sum_{i\in[N]} a_i \ln (a_i/b_i)$.
The indicator function is $\indicator{test}$ and is equal to 1 if $test$ is true, 0 otherwise.
Let $\rate_t\geq 0$ for all $t\in\Naturals$ be the \emph{learning rate}.
Let $\QQ > 0$, which will be a parameter of most algorithms.
We will also often --- but not always --- use the following assumptions:
\begin{assumption}[Generic isotuning assumptions]\label{ass:generic_isotuning}
\begin{align*}
 &\QQ > 0\,, 
 \ \Delta_0 = 0\,, \\
 &\forall t\in\Naturals:
  \delta_t \geq 0\,,\ 
  \Delta_t = \frac{\QQ}{\rate_t} = \sum_{s\in[t]}\delta_s\,.
  \qedhere
\end{align*}
\end{assumption}
These assumptions are used for AdaHedge~\citep{derooij2014follow} for a specific definition of $\delta_t$ for the Hedge algorithm and with $q=\ln N$.
In what follows we will provide a more general definition of $\delta_t$ to accommodate
a wide range of situations.

\section{GENERAL THEORY}\label{sec:theory}

We describe isotuning, then online correction, and finally null updates,
and we provide generic regret bounds that we will use in our applications.

\subsection{Isotuning}\label{sec:isotuning}

As explained in the introduction, isotuning is based on the idea of sequentially balancing a trade-off between the two terms of the regret.
We first introduce a non-sequential form of this idea,
as depicted in \cref{fig:isopoint}, which we will build upon for the sequential form, and also re-use as-is when dealing with null updates.
We define the $\iso$ operator for convenience, and then give a very simple isopoint lemma that we will use multiple times.

\begin{definition}[Iso operator]
Let $\mathcal{S}\subseteq\Reals$ be a convex set.
Define the operator $\iso_{x\in\mathcal{S}}(f(x), g(x)) = f(s)$
where $s\in\mathcal{S}$ is such that $f(s) = g(s)$
for $f$ continuous and non-decreasing
(resp. non-increasing)
on $\mathcal{S}$, and $g$ continuous and non-increasing
(resp. non-decreasing)
on $\mathcal{S}$.
Then also $s=\arg\iso_\mathcal{S}(f, g)$ with ties broken in favour of small $s$.
If $s$ does not exist the result is undefined.
\end{definition}

\begin{lemma}[Isopoint]\label{thm:isopoint}
Let $\mathcal{S}\subseteq \Reals$ be convex, and let $f$
be continuous non-decreasing on $\mathcal{S}$ and let $g$
be continuous non-increasing on $\mathcal{S}$.
If $\iso_{s\in\mathcal{S}}(f(s), g(s))$ exists then
for all $x\in\mathcal{S}$,
\begin{align*}
    \min\{f(x),\ &g(x)\}\ \leq \\
        &\iso_{s\in\mathcal{S}}(f(s),\ g(s))\ \leq\ \max\{f(x),\ g(x)\}\,.
        \qedhere
\end{align*}
\end{lemma}
\begin{proof}
Let $s$ be such that $f(s) = g(s)$.
Then by the monotonicity of $f$ and $g$, if $x \leq s$, then
$f(x) \leq f(s)=g(s) \leq g(x)$, and conversely for $x > s$.
\end{proof}

\begin{remark}
if $x^= = \arg\iso(f, g)$ and $x^* = \argmin f+g$
then $f(x^=)+g(x^=) \leq 2 \max\{f(x^*), g(x^*)\}$.
\end{remark}

\begin{figure}
    \centering
    \includegraphics[width=0.75\linewidth]{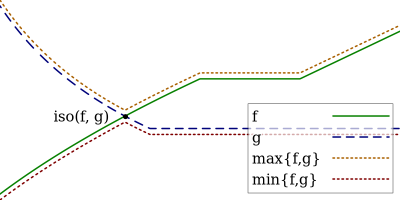}
    \caption{If $f$ increases and $g$ decreases, then 
    for all $x:\min\{f(x), g(x)\}\ \leq\ \iso(f, g)\ \leq\ \max\{f(x), g(x)\}$.}
    \label{fig:isopoint}
\end{figure}

Next, we consider a sequential form of the isopoint operator and of the isopoint lemma, for the special case
\footnote{The `more general' case of an increasing $f$ can still be obtained by
a change of variables.}
where $f(x) = x$ for $x \geq 0$.

\begin{definition}[Isotuning sequence]\label{def:isotuning_seq}
A pair $\tuple{X, g}$ is an \emph{isotuning sequence}
if the following conditions hold
\footnote{These conditions have been simplified at v3 of the paper.}
for all $t\in\Naturals$:
\begin{itemize}
    \item[(i)] $g_t: [0,\infty) \to [0, \infty]$ is continuous non-increasing,
    \item[(ii)] $\lim_{x\to\infty}g_t(x) < \infty$,
    \item[(iii)] $X_0 = 0$,
    \item[(iv)] $X_t\ =\ X_{t-1} + g_t(X_t)$, \qedhere
\end{itemize}
\end{definition}

Condition (ii) ensures that $X_t$ exists and is unique,
as shown in \cref{apdx:isotuning}.
Condition (iv) is an isotuning update: $X_t~=~\iso_{y\in\Reals}(y,\ X_{t-1} + g_t(y))$.

\begin{remark}
$\forall T:\ X_T + \sum_{t=1}^T g_t(X_t)\ =\ 2X_T$.
\end{remark}

We now give our central theorem, which bounds the isotuning sequence within a factor 2 of the optimal tuning in hindsight, for all sequences of positive non-increasing $\{g_t\}_t$.

\begin{theorem}[Isotuning]\label{thm:isotune}
Let $\tuple{X, g}$ be an isotuning sequence and define, for all $T\in\Nonnegints$,
\begin{align*}
    M_T(x) = x + \sum_{t\in[T]} g_t(x)\,,
    \quad\text{and}\quad
    M^*_T = \inf_{x>0} M_T(x)\,,
\end{align*}
then for all $T \in\Nonnegints$, 
\begin{equation*}
    \tfrac12 M^*_T \leq \tfrac12 M_T(X_T) \leq X_T= \sum_{t\in[T]}g_t(X_t) \leq M^*_T\,.
    \qedhere
\end{equation*}
\end{theorem}
\begin{proof}
For the lower bound,
for all $t\in[T]$,
since $g_t$ is non-increasing and $X_T \geq X_t$, 
\begin{align*}
    2X_T &= X_T + \sum_{t\in[T]}g_t(X_t) \geq X_T + \sum_{t\in[T]} g_t(X_T) \\
    &= M_T(X_T) \geq M^*_T \,.
\end{align*}
For the upper bound, 
let $x^*_T = \arginf_{x \geq 0} M_T(x)$ and
let $\tau = \min\{t\in[T]: X_t \geq x^*_T\}$.
If $\tau$ does not exist then $X_T < x^*_T \leq M^*_T$, proving the claim.
Otherwise, since $g_t$ is non-increasing then for all $t\geq \tau$ we have 
$g_t(X_t) \leq g_t(x^*_T)$ and so
\begin{align*}
    X_T &= X_{\tau-1} + \sum_{t=\tau}^T g_t(X_t) \leq x^*_T + \sum_{t=\tau}^T g_t(x^*_T) \\
    &\leq M_T(x^*_T) = M^*_T\,.
    \qedhere
\end{align*}
\end{proof}

The r.h.s.\! of \cref{thm:isotune} can be written as follows.
\begin{mdframed}
For all $T\in\Nonnegints$,
if $X_0 = 0$
and $X_t = X_{t-1} + g_t(X_t)$
for all $t\in[T]$
with $g_t:\Reals\to[0, \infty]$ continuous non-increasing, then
\begin{align}\label{eq:isotuning_upper_bound}
    X_T = \sum_{t \in [T]} g(X_t)\ \leq\ \inf_{x \geq 0} \left\{x + \sum_{t\in[T]} g_t(x)\right\}
\end{align}
\end{mdframed}

\begin{example}
Consider the regret bound~\cref{eq:regret_intro} of the introduction.
We use \cref{ass:generic_isotuning} and so $\Delta_T = \QQ/\rate_T = \sum_t \delta_t$.
Simply define $\delta_t = \rate_t a_t$, then:
\begin{align*}
    \regret_T(\ptx^*) \leq \Delta_T + \sum_{t\in[T]} \delta_t = 2\Delta_T\,.
\end{align*}
Define $g_t(x) = \QQ a_t/x$ such that $\delta_t = g_t(\Delta_t)$ for all $t\in[T]$, 
and $g_t(0) = \infty$,
then by \cref{def:isotuning_seq}, $\tuple{\Delta, g}$ is an isotuning sequence
(we can extend trivially $g_t(\cdot)=0$ for $t > T$).
Thus, by \cref{thm:isotune} (and \cref{eq:isotuning_upper_bound}) where $X_t \leadsto \Delta_t$ we obtain immediately
\begin{equation*}
    \Delta_T \leq M^*_T
    = \inf_{y>0}\left\{y + \sum_{t\in[T]} \frac{\QQ a_t}{y}\right\} = 2\sqrt{\QQ\sum_{t\in[T]}a_t}\,.
    \qedhere
\end{equation*}
\end{example}

\begin{example}
To demonstrate the adaptivity of isotuning, assume instead that the regret bound is
\begin{align*}
    \regret_T(\ptx^*) \leq \frac{\QQ}{\rate_T} + \sum_{t\in[T]} \sqrt{\rate_t} a_t\,.
\end{align*}
For example, such a bound can be extracted from Eq. 10 and Prop. 7 by \citet{lattimore2020monitoring}, in the context of partial monitoring.
Then we also just set $\delta_t = \sqrt{\rate_t} a_t = \sqrt{\QQ/\Delta_t}a_t$
and use  \cref{ass:generic_isotuning}, so by \cref{thm:isotune} 
(and \cref{eq:isotuning_upper_bound}),
\begin{align*}
    \regret_T&(\ptx^*)\ \leq\ 2\Delta_T\ \leq\ 2M^*_T \\
    &= 2\inf_{y>0}\left\{y + \sum_{t\in[T]} \sqrt{\frac{\QQ}{y}}a_t\right\} 
    = 2c\QQ^{\tfrac13}\bigg(\sum_{t\in[T]}a_t\bigg)^{\tfrac23}
\end{align*}
where $c=2^{-\frac23} + 2^{\frac13}$.
Observe that with a change of variables we also have
\begin{align*}
    \regret_T&(\ptx^*) \leq
     2\inf_{\rate >0}\left\{\frac{\QQ}{\rate} + \sum_{t\in[T]} \sqrt{\rate}a_t\right\} 
\end{align*}
which again shows that the regret is within a factor 2 of the optimal constant learning rate in hindsight.
\end{example}

\subsection{Example: AOGD Revisited}\label{sec:aogd}

As mentioned earlier, our isotuning tools are inspired by the analysis of \citet{bartlett2007adaptive}, who propose an adaptive online gradient descent algorithm that automatically and sequentially adapts to the strong convexity of the loss functions,
by introducing an additional regularizer.
We show how to use our generic isotuning tools
to recover this adaptive result, while leading to a simplified analysis
and a slightly better and simpler bound, by tuning only the learning rate of online gradient descent without needing an additional regularizer.
See \cref{apdx:aogd} for additional information.

Assume that for all $t\in[T], \loss_t$ is $\alpha_t$-strongly convex, that is,
for all $(x, y) \in\ptxset^2$, 
$\loss_t(y) \geq \loss_t(x) + \innerprod{y -x, \nabla\loss_t(x)} + \frac{\alpha_t}{2}\|y-x\|_2^2$,
We write $\nabla_t \equiv \nabla \loss_t(\ptx_t)$.
Define $\alpha_{1:t} = \sum_{s\in[t]} \alpha_t$,
and $r_t = \loss_t(\ptx_t) - \loss_t(\ptx^*)$.
We use the standard online gradient descent update with projection onto $\ptxset$: $\ptx_{t+1} = \proj{\ptxset}{\ptx_t - \rate_t \nabla_t}$.

Asuming $1/\rate_0 = 0$ and $\max_t \|\ptx_t-\ptx^*\|_2^2 \leq \diam^2$ and taking $\QQ = \diam^2$,
then by \cref{lem:aogd} (see \cref{apdx:aogd})
\begin{align*}
    2\,\regret_T(\ptx^*) \leq
    \underbrace{\QQ\left(\frac{1}{\rate_T} - \alpha_{1:T}\right)}_{(A)} + 
    \underbrace{\sum_{t\in[T]} \rate_t \|\nabla_t\|^2}_{(B)}\,.
\end{align*}
The term $(A)$ is decreasing with the learning rate, while the term $(B)$ is increasing with it.
Hence, to balance the two terms, we define $\delta_t = \rate_t \|\nabla_t\|_2^2$ and $\Delta_T = \sum_{t\in[T]} \delta_t$,
and we also define $\rate_T$ (for all $T$) such that $\Delta_T = \QQ(\frac{1}{\rate_T} - \alpha_{1:T})$.
Thus,
\begin{align*}
    \regret_T(\ptx^*) \leq \Delta_T\,.
\end{align*}
Stepwise, this corresponds to the update
\begin{align*}
    \Delta_t &= \Delta_{t-1} + \delta_t\,, \\
    \QQ\left(\frac{1}{\rate_t} - \alpha_{1:t}\right) &= \QQ\left(\frac{1}{\rate_{t-1}} - \alpha_{1:t-1}\right) + \rate_t\|\nabla_t\|_2^2
\end{align*}
which is solved for $\invrate_{t} = \psi(\invrate_{t-1}+\alpha_t, 4\|\nabla_t\|_2^2/\QQ)$
with $\psi(a, b) = (a+ \sqrt{a^2 + b})/2$; the learning rate $\rate_t$ is thus nonnegative non-increasing with $t$.
To use \cref{thm:isotune}, we need to express $\delta_t$ as a function of $\Delta_t$ (as in \cref{def:isotuning_seq}, condition (iv)),
using $\Delta_t = \QQ(\frac{1}{\rate_t} - \alpha_{1:t})$:
\begin{align*}
    \delta_t = \rate_t \|\nabla_t\|_2^2 = \frac{\QQ\|\nabla_t\|_2^2}{\Delta_t + \QQ\alpha_{1:t}}
    \ =:\  g_t(\Delta_t)\,.
\end{align*}
We can now apply \cref{thm:isotune} (see \cref{eq:isotuning_upper_bound})
on the isotuning sequence $\tuple{\Delta, g}$:
\begin{align*}
    \regret_T(\ptx^*) &\leq \inf_{y \geq 0} \left\{y + \sum_{t\in[T]} \frac{\QQ\|\nabla_t\|_2^2}{y + \QQ\alpha_{1:t}}\right\}\,.
\end{align*}
Finally, making a change of variables $y = z\QQ$
and recalling that $\QQ = \diam^2$,
\newcommand{\isoaogdregret}{
\begin{align*}
    \regret_T(\ptx^*) &\leq  \inf_{z \geq 0} \left\{z\diam^2 + \sum_{t\in[T]} \frac{\|\nabla_t\|_2^2}{z + \alpha_{1:t}}\right\}\,.
\end{align*}}
\isoaogdregret
We have thus recovered (and slightly improved) the adaptivity to sequential strong-convexity of the original work (although in a simpler form): if $\alpha_{1:t}=\Omega(t)$ we can take $y=0$ and get a regret bound in $O(G^2\log T)$ where $G = \max_{t\in[T]}\|\nabla_t\|_2$, 
but even if $\alpha_{1:T} = 0$ we also can take $y = \diam\sqrt{\sum_{t\in[T]}\|\nabla_t\|_2^2}$
to obtain a regret bounded by $2y$.
All intermediate regimes can be retrieved as well.
See \cref{apdx:aogd} for more information.

\subsection{A Refined Bound}

\cref{thm:isotune} is fairly general and can accommodate many cases,
in particular since formulas of the form $X \leq \inf_{y} F(y)$ can easily be simplified by changes of variables, or by picking some good enough value for $y$.
However, the special case $g_t(x) = \frac{a_t}{x}$ is frequent enough that is deserves a slightly improved bound.
\begin{lemma}\label{lem:Delta_sqrt_short}
Let $\tuple{X, g}$ be an isotuning sequence where $g_t(x) = a_t/x$ with $a_t\geq0$ and $g_t(0) = \lim_{x\to 0^+} g_t(x)$, for all $t\in[T]$,
then
\begin{align*}
    X_T \leq \sqrt{2\sum_{t\in[T]}a_t}\,.
    &\qedhere
\end{align*}
\end{lemma}
The result follows by \cref{lem:Delta_sqrt} in \cref{apdx:isotuning_refined},
and it can be easily adapted to various other cases.
When applicable, this improved bound saves a factor $\sqrt{2}$ over \cref{thm:isotune}.

\subsection{Online Correction}\label{sec:generic_ob1}

The standard analysis for online gradient descent (\eg \citet{hazan2016oco}) introduces a factor 
$\max_{t\in[T]} \|\ptx_t - \ptx^*\|^2\leq \diam^2$ in the regret bound.
However, $\diam=\infty$ for unbounded domains.
Some algorithms, such as FTRL, manage to have regret bounds that are centered on a fixed point (such as $\ptx_1$ or 0) and feature a term $\|\ptx_1 - \ptx^*\|^2$ or $\|\ptx^*\|^2$ instead.
(Even for bounded convex domains this is an improvement, see \cref{ex:md_ogd} in \cref{apdx:md}).
With constant learning rates, Mirror Descent is mostly equivalent to FTRL, but when using adaptive learning rates
Mirror Descent usually does not provide centered bounds, and features a leading factor of $\max_{\ptx_1, \ptx_2\in\ptxset} \breg_R(\ptx_1, \ptx_2)$ instead.
For bounded domains where the Bregman divergence is also bounded, this may be good enough and scale-free guarantees can be obtained~\citep{orabona2018solo}, but this factor can be infinite otherwise.

In this section we use the correction rule of \citet{orseau2017softbayes}
(see also \citet{pmlr-v48-gyorgy16,pmlr-v119-fang20a})
to derive general centered regret bounds for a large class of algorithms beyond FTRL.
While this correction rule was used for only one algorithm, we show that its usage can be extended to a much larger class, by contrast to other correction rules, such as the exponential correction of \citet{gyorfi2007sequential,gaillard2014secondorder}
(see also \citet{chen2021impossible,huang2021delays} for other specific corrections).
We use \cref{ass:generic_isotuning}.

When deriving centered bounds, an additional difficulty appears: the learning rate becomes `off-by-one'~\citep[see also \cref{lem:generic_ob1} below]{mcmahan2017survey}, that is, 
at step $t$ only the \emph{previous} learning rate $\rate_{t-1}$ may be used, 
which means that the update cannot take into account, and adapt to the feedback received at step $t$.

At each step, we assume that the point $\ptx_t$ is first updated to $\ptx_{t'}$
(in a problem-dependent way)
using the \emph{previous} learning rate $\rate_{t-1}$. Intuitively,
\begin{align}\label{eq:generic_ob1_generic_update}
    \ptx_{t'} = \code{update}_t(\ptx_t, \rate_{t-1})\,.
\end{align}
for some algorithm-specific $\code{update}$ function.
Then the following online correction (equivalent to that of ~\citet{orseau2017softbayes}) is applied:
\begin{align}\label{eq:generic_ob1_ocr}
    \ptx_{t+1} = \ptx_{t'} \frac{\Delta_{t-1}}{\Delta_{t}} +
    \frac{\delta_t}{\Delta_t}\ptx_1\,.
\end{align}

\begin{lemma}[Online correction]\label{lem:ocr_convexity}
Let $f: \ptxset\to\Reals$ be a convex function.
Then, under \cref{ass:generic_isotuning}, for all $t\in[T]$:
\begin{equation*}
    \Delta_t f(\ptx_{t+1}) \leq \Delta_{t-1}f(\ptx_{t'}) + \delta_t f(\ptx_1)\,,
\end{equation*}
and holds with equality if $f$ is affine.
\end{lemma}
\begin{proof}
Follows from $\Delta_{t-1} + \delta_t = \Delta_t$ and the convexity of $f$,
then multiplying by $\Delta_t$ on both sides.
\end{proof}

Also observe that necessarily $\ptx_{t+1}\in\ptxset$ if $\ptx_{t'}\in\ptxset$ and $\ptx_1\in\ptxset$.

Recall that $\phi:\mathcal{X}\to\Reals$ is a convex function used to analyze the regret of the algorithm, and that $r_t$ is the instantaneous regret.
Define for all $t\in\Naturals$,
\begin{align}\label{eq:delta*t}
    \delta^*_t &= \positive{\frac{1}{\rate_{t-1}}(\phi(\ptx_{t'}) - \phi(\ptx_t)) + r_t}\,,
\end{align}
then we require
\begin{align}\label{eq:generic_deltat_ob1}   
    \delta_t \geq \delta^*_t \,.
\end{align}
In general, this inequality should be as tight as possible.
\footnote{
$\delta_t$ reduces after cancellations to the \emph{mixability gap} of AdaHedge~\citep{derooij2014follow} in the special case where
 $\phi(\ptx)=\RE{\ptx^*}{\ptx}$ and $\QQ=\ln N$. 
The main difference is that we explicitly incorporate the whole instantaneous regret $r_t$ inside $\delta_t$ so as to obtain a generic regret bound.}
See \cref{alg:generic_ob1}.
With these conditions we can provide a generic regret bound in this `off-by-one' setting.

\begin{algorithm}[htbp!]
\begin{lstlisting}
def generic_OCO_off_by_1($\ptx_1\!\in\!\ptxset, \QQ\!>\!0$):
  $\Delta_0 = 0$
  for t = 1, 2, ...:
    predict $\ptx_t$; observe $\loss_t$
    $\rate_{t-1} = \QQ/\Delta_{t-1}$ # off-by-one learning rate
    $\ptx_{t'}$ = update$_t$($\ptx_t, \rate_{t-1}$)  # off-by-one update
    set $\delta_t$ s.t. $\delta_t \geq \positive{\frac{1}{\rate_{t-1}}(\phi(\ptx_{t'}) - \phi(\ptx_t)) + r_t}$  # found by analysis
    $\Delta_t = \Delta_{t-1} + \delta_t$ # isotuning
    $\ptx_{t+1} =\frac{\Delta_{t-1}}{\Delta_{t}} \ptx_{t'} + \frac{\delta_t}{\Delta_t}\ptx_1$ # online correction
\end{lstlisting}
\caption{Generic online convex optimization algorithm sketch in the off-by-one setting.
The update of $\rate_t$ usually has a closed form since $\delta_t$ depends on $\rate_{t-1}$ only.
}
\label{alg:generic_ob1}
\end{algorithm}

\begin{lemma}[Generic regret bound, off-by-one setting]\label{lem:generic_ob1}
Consider \cref{ass:generic_isotuning}.
Let $\phi:\ptxset\to\Reals$ be a convex function,
and assume that $\delta_t$ satisfies \cref{eq:generic_deltat_ob1}.
The regret of \cref{alg:generic_ob1} is bounded by
\begin{align*}
    \regret_T(\ptx^*) &\leq
    \left(1+ \frac{\phi_1}{\QQ} - \frac{\phi_{T+1}}{\QQ}\right)\Delta_T\,.
    \qedhere
\end{align*}
\end{lemma}
\begin{proof}
From \cref{lem:ocr_convexity}, since $\phi$ is convex, for all $t$,
\begin{align*}
    \Delta_t \frac{\phi_{t+1}}{\QQ} &\leq \Delta_{t-1}\frac{\phi_t}{\QQ} + \delta_t\frac{\phi_1}{\QQ} + \Delta_{t-1}\left(\frac{\phi_{t'}}{\QQ}-\frac{\phi_t}{\QQ}\right) \\
    \Leftrightarrow\quad
    r_t &\leq 
     \Delta_{t-1}\frac{\phi_t}{\QQ} - \Delta_t \frac{\phi_{t+1}}{\QQ}
     + \delta_t\frac{\phi_1}{\QQ} \\
     &\quad+ \underbrace{\Delta_{t-1}\frac{\phi_{t'}-\phi_t}{\QQ} + r_t}_{\leq \delta_t} \\
     &\leq \Delta_{t-1}\frac{\phi_t}{\QQ} - \Delta_t \frac{\phi_{t+1}}{\QQ}
     +\left(1+\frac{\phi_1}{\QQ}\right)\delta_t\,.
\end{align*}
Summing over $t$ and telescoping gives:
\begin{align*}
    \sum_{t\in[T]} r_t &\leq
    \Delta_0\frac{\phi_1}{\QQ} - \Delta_T\frac{\phi_{T+1}}{\QQ}
    + \left(1+\frac{\phi_1}{\QQ}\right)\sum_{t\in[T]}\delta_t
\end{align*}
and the result follows from \cref{ass:generic_isotuning}.
\end{proof}

Since, in the off-by-one setting, $\delta_t$ is a function of $\rate_{t-1}$ and thus of $\delta_{t-1}$ rather than $\Delta_t$, 
we provide a simple result to deal with this offset.
\begin{lemma}[Isotuning offset]\label{lem:isotuning_offset}
Assume that for all $t\in\Naturals$, $\Delta_t = \sum_{s\in[t]}\delta_s$
with $\delta_t \geq 0$.
For all $t\in\Naturals$, let $\hat{g}_t: \Reals \to [0, \infty]$ be continuous nonnegative non-increasing functions such that
$\delta_t \leq \hat{g}_t(\Delta_t - b_t)$ with $b_t \in\Reals$.
Let $\tuple{\hat\Delta, \hat{g}}$ be an isotuning sequence.
Then, for all $T\in\Naturals$, 
\begin{equation*}
    \Delta_T \ \leq\ \hat\Delta_T + \max_{t\in[T]}\positive{b_t}\,.
     \qedhere
\end{equation*}
\end{lemma}
\begin{proof}
Assume that $\Delta_{t-1} \leq \hat\Delta_{t-1} + \max_{s< t} \positive{b_t}$,
which is satisfied for $t=1$.
Then
\begin{alphalign*}
    \Delta_t 
    &\alphoverset{=}
    \min\{\Delta_t,\ \Delta_{t-1} + \delta_t\} \\
    &\alphoverset{=}
    \min\{\Delta_t,\ \Delta_{t-1} + \hat{g}(\Delta_t - b_t)\} \\
    &\alphoverset{\leq} \iso_{x\in\Reals}\big( x,\ \Delta_{t-1} + \hat{g}(x - b_t)\big) \\
    &\alphoverset{\leq} \max\{\hat\Delta_t + b_t,\ \Delta_{t-1} + \hat{g}_t(\hat\Delta_t)\} \\
    &=  \max\{\hat\Delta_t + b_t,\ \hat\Delta_{t-1} +
    \hat{g}_t(\hat\Delta_t) + \max_{s<t}\positive{b_s}\} \\
    &\alphoverset{=}  \max\{\hat\Delta_t + b_t,\ \hat\Delta_{t} + \max_{s<t}\positive{b_s}\} \\
    &= \hat\Delta_t + \max_{s\leq t} \positive{b_s}\,,
\end{alphalign*}
with
\alphnextref{} is by definition of $\Delta_t$,
\alphnextref{} is by assumption on $\delta_t$,
\alphnextref{} is by \cref{thm:isopoint},
\alphnextref{} is also by \cref{thm:isopoint} and taking $x = \hat\Delta_t + b_t$,
\alphnextref{} is because $\tuple{\hat\Delta, \hat{g}}$ is an isotuning sequence.
Finally, the result holds by induction.
\end{proof}

\begin{remark}
Such offsets can also be handled directly with \cref{thm:isotune} via a change of variables (see \cref{ex:x*_issues}).
\end{remark}

\begin{corollary}[Off-by-one]\label{lem:generic_ob1_DeltaT}
Consider the conditions of \cref{lem:isotuning_offset}, with $\delta_t \leq \hat{g}_t(\Delta_{t-1})$ for all $t\in[T]$. Then 
\begin{equation*}
    \Delta_T \leq \hat\Delta_T + \max_{t\in[T]} \delta_t\,.
    \qedhere
\end{equation*}
\end{corollary}
\begin{proof}
Follows from \cref{lem:isotuning_offset} where $b_t=\delta_t \geq 0$.
\end{proof}

Note that \cref{lem:generic_ob1_DeltaT} does \emph{not} require that $\delta_t$ satisfies \cref{eq:generic_deltat_ob1}.

For bounded domains, even with \emph{unknown diameter}, we can already easily obtain scale-free adaptive regret bounds.
For example, consider online gradient descent
using \cref{alg:generic_ob1}
with the (off-by-one) update rule
$\ptx_{t'} = \ptx_{t} - \rate_{t-1}\nabla_t$.
Following \cref{eq:generic_deltat_ob1}
with $\phi(\ptx) = \|\ptx-\ptx^*\|^2/2$
(for some unknown $\ptx^*\in\ptxset$)
we can take $\delta_t = \rate_{t-1}\|\nabla_t\|^2/2 = \QQ\|\nabla_t\|^2/(2\Delta_{t-1})$.
Define $\hat{g}_t(z) = \QQ\|\nabla_t\|^2/(2z)$
such that $\delta_t \leq \hat{g}_t(\Delta_{t-1})$ for all $t\in[T]$,
and $\hat{g}_t(z) = \infty$ for $z \leq 0$.
Then from \cref{lem:generic_ob1,lem:generic_ob1_DeltaT,thm:isotune} we obtain
\begin{multline*}
    \regret_T(\ptx^*) \leq 
    \left(1+\frac{\phi_1}{\QQ}\right)
        \left(
        \sqrt{\QQ\sum_{t\in[T]}\|\nabla_t\|^2}
        +\max_{t\in[T]} \delta_t
        \right)
\end{multline*}
From the update rule, we also have $\delta_t = \|\ptx_{t'}-\ptx_t\|\|\nabla_t\|/2$,
so for bounded domains $\delta_t \leq \diam \max_t \|\nabla_t\|/2$ even if $\diam$ is unknown.
Using $\phi_1=\|\ptx_1-\ptx^*\|^2/2$ we thus have
\begin{multline*}
    \regret_T(\ptx^*) \leq \\ \left(1+\frac{\|\ptx^*-\ptx_1\|^2}{2\QQ}\right)\left(\sqrt{\QQ\sum_t\|\loss_t\|^2}
    + \frac{\diam}{2}\max_t \|\loss_t\|\right)
\end{multline*}
which fulfills our requirements of being centered, scale-free and adaptive to small losses.
More discussion is given in \cref{ex:md_ogd}.
With unknown $\diam$, interesting choices for $\QQ$ are $\QQ=1$, but also $\QQ=\|\mathbf{1}\|^2$,
$\QQ=\ln T$ or $\QQ=\|\mathbf{1}\|^2\ln T$.

\subsection{Null Updates}\label{sec:null_updates}

The additive term $\max_t \delta_t$
of \cref{lem:generic_ob1_DeltaT} is unfortunately not well bounded if $\diam$ is overly large or infinite (\eg, for partially constrained domains).
For example, with $\ptxset=\Reals^N$, both $\|\ptx_{t'}-\ptx_t\|$ and $\rate_{t-1}\|\nabla_t\|^2/2$
may be arbitrarily large, because $\rate_{t-1}$ may be too large compared to 
$\|\nabla_t\|_2$.
And enforcing $\rate_{t-1} \leq 1$ would not lead to scale-free bounds for either small or large loss scales. Note that small loss scales can matter for example when performing gradient descent on an almost flat surface.

At this point, we could use the loss clipping trick of \citet{cutkosky2019hints,mhammedi2019squint}
with off-by-one isotuning to obtain a bound very similar to the SOLO FTRL bound
(hence with the additive $O(L\sqrt{T})$ term in unbounded domains).
Instead, we consider another simple trick that will have additional benefits:
updates that are too large, leading to large instantaneous regret, are 
simply skipped, similarly to skipping updates when delays are too large for bandit algorithms~\citep{thune2019delays,zimmert2020delays}.
This corresponds to performing an update with a zero loss: $\ptx_{t'}=\ptx_t$, 
which we call a \emph{null update},
and, by contrast to using the doubling trick~\citep{cesabianchi1997expert}
(see also \citet{mhammedi2019squint}),
the algorithm does not need to be additionally restarted.
The instantaneous regret on null updates are dealt with separately (outside of $\Delta_T$),
as demonstrated in the following generic result.
We call $\barT$ the set of steps where null updates are performed.

\begin{theorem}[Generic regret bound, null updates]\label{thm:generic_regret_null}
Consider the assumptions of \cref{lem:generic_ob1}, but 
modify \cref{alg:generic_ob1} as follows:
When $t\in\barT$ set $\ptx_{t'} = \ptx_t$, and set $\delta_t \geq 0$.
Then the regret of the algorithm is at most
\begin{equation*}
    \regret_T(\ptx^*)
    \leq 
    \left(1+\frac{\phi_1}{\QQ} - \frac{\phi_{T+1}}{\QQ}\right)\Delta_T
    + \sum_{t\in\barT} (r_t - \delta_t)\,.
    \qedhere
\end{equation*}
\end{theorem}
\begin{proof}
Similar to the proof of \cref{lem:generic_ob1}, except that for $t\in\barT$, $\phi_{t'}=\phi_t$, and in this case we may not have $\delta_t \geq r_t$
so we extract $r_t$, and add and subtract $\delta_t$.
\end{proof}
Note that when $t\in\barT$, it is important to set $\delta_t > 0$ to increase $\Delta_t$ (and thus decrease $\rate_t$) and avoid triggering repeated null updates.

Consider the function $\hat{g}_t(\cdot)$ used in \cref{lem:generic_ob1_DeltaT}.
We perform a null update at step $t$ when the instantaneous regret term coming from the update is larger than all previous ones added up, that is, 
when $\hat{g}_t(\Delta_{t-1}) \geq c\Delta_{t-1}$ for some $c > 0$.
Define the set of steps where a null update is performed:
\begin{equation*}
    \barT = \{t\in[T] : \hat{g}_t(\Delta_{t-1}) >  c\Delta_{t-1}\}\,.
\end{equation*}

For $t\notin\barT$, 
since $\delta_t \leq \hat{g}_t(\Delta_{t-1}) \leq c\Delta_{t-1}$, using \cref{thm:isopoint},
\begin{align}
    \delta_t\leq
    \hat{g}_t(\Delta_{t-1}) 
    &= \min\{\hat{g}_t(\Delta_{t-1}),\ c\Delta_{t-1}\} \notag\\
    &\leq \iso_{y\geq 0}(\hat{g}_t(y),\ c y)\,.
    \label{eq:isopoint_deltat_notinbarT}
\end{align}
This iso value has the particularity of being independent of $\Delta_{t-1}$ and thus of the learning rate $\rate_{t-1}$.
For $\tau=\max \barT$, the inequality is reversed, hence,
\begin{align}\label{eq:isopoint_Deltatau}
    c\Delta_{\tau-1} \leq \iso_{y\geq0}(\hat{g}_\tau(y),\ c y)\,.
\end{align}
Now, what value for $\delta_t$ should we choose when $t\in\barT$ is a null update step?
A good choice is again 
\begin{align*}
    \forall t\in\barT:&&
    \delta_t =& \iso_{y\geq0}(\hat{g}_\tau(y),\ c y)\,,
\intertext{which ensures that}
    \forall t\in\barT:&&
    c\Delta_{t-1} \leq {}& \delta_t\hphantom{(\Delta_{t-1})} \leq  \hat{g}_t(\Delta_{t-1})\,, \\
    \forall t\notin\barT:&&
    \delta_t \leq {}& \hat{g}_t(\Delta_{t-1}) \leq c\Delta_{t-1}\,,
\end{align*}
where the second line follows from \cref{thm:isopoint}.
We now provide a generic theorem that uses this particular choice of $\delta_t$.

\begin{lemma}[Isotuning with null updates]\label{thm:generic_DeltaT_null}
Under the same conditions as \cref{lem:generic_ob1_DeltaT}, and additionally
let $c > 0$ and 
for $t\in\barT$ set $\delta_t = \iso_{y \geq 0}(\hat{g}_t(y),\ c y)$.
Then for all $T\in\Naturals$ and with $\tau = \max \barT$,
\begin{align*}
    \Delta_T &\leq \hat\Delta_T
    + \max_{t\in[T]} \ \iso_{y \geq 0}(\hat{g}_t(y),\ c y)\,, &\text{(i)}\\
    \sum_{t\in\barT}\delta_t &\leq (\tfrac{1}{c}+1)\iso_{y \geq 0}(\hat{g}_\tau(y),\ c y)\,.&\text{(ii)}
    &\qedhere
\end{align*}
\end{lemma}
\begin{proof}
(i) For $t\in\barT, \delta_t=\iso_{y\geq0}(\hat{g}_t(y),\ c y) \leq \max\{\hat{g}_t(\Delta_{t-1}),\ c\Delta_{t-1}\}=\hat{g}_t(\Delta_{t-1})$ by \cref{thm:isopoint},
hence for all $t\in[T]$, $\delta_t \leq \hat{g}_t(\Delta_{t-1})$,
and we can then use \cref{lem:generic_ob1_DeltaT}.
Moreover, for $t\notin\barT, \delta_t \leq \hat{g}_t(\Delta_{t-1})
=\min\{\hat{g}_t(\Delta_{t-1}),\ c \Delta_{t-1}\}
\leq \iso_{y\geq0}(\hat{g}_t(y),\ c y)$,
and thus $\max_t \delta_t \leq \max_t \iso_{y\geq0}(\hat{g}_t(y),\ c y)$.

(ii) The result follows from
$\sum_{t\in\barT} \delta_t \leq \Delta_\tau = \Delta_{\tau-1} + \delta_\tau$
and \cref{eq:isopoint_deltat_notinbarT,eq:isopoint_Deltatau}.
\end{proof}
\begin{remark}
\Cref{thm:generic_DeltaT_null} can be straightforwardly generalized from $c y$ to any sequence of continuous non-decreasing functions $\{f_t(y)\}_t$, including constant functions.
\end{remark}

Continuing our online gradient descent example,
recall that $\hat{g}_t(y) = \QQ\|\nabla_t\|^2/(2y)$,
and thus $\iso_{y>0}(\QQ\|\nabla_t\|^2/(2y), y) = \sqrt{\QQ/2}\|\nabla_t\|$.
So if we perform null updates with $c=1$, that is,
when $\rate_{t-1}\|\nabla_t\|^2/2 > \Delta_{t-1}$
(or, equivalently, when $\sqrt{\QQ/2}\|\nabla_t\| > \Delta_{t-1}$),
then by \cref{thm:generic_DeltaT_null} (and \cref{lem:Delta_sqrt_short}) we have 
\begin{align*}
    \Delta_T \leq
    \sqrt{\QQ\sum_{t\in[T]}\|\nabla_t\|^2} + \sqrt{\QQ/2}\max_t \|\nabla_t\|\,.
\end{align*}
Thus, now the additive term is nicely bounded and does not depend on the diameter $\diam$.
It remains to bound the regret incurred on null update rounds
as per \cref{thm:generic_regret_null}:
\begin{align*}
    \sum_{t\in\barT} r_t &\leq \sum_{t\in\barT} \|\ptx^* -\ptx_t\|\|\nabla_t\| \\
    &=\sum_{t\in\barT} \|\ptx^* -\ptx_t\|\sqrt{2/\QQ}\delta_t \\
    &\leq 2\max_{t\in\barT}\|\ptx^* -\ptx_t\| \|\nabla_\tau\|
\end{align*}
where we used \cref{thm:generic_DeltaT_null} (ii) on the last line.
Finally, 
$\max_{t\in\barT}\|\ptx^* -\ptx_t\|$ is bounded by $D$,
but also by $\|\ptx^* - \ptx_1\| + \sqrt{2\QQ\tau}$
by \cref{lem:md_travel} (in \cref{apdx:md}),
where $\tau = \max \barT$.

Putting it all together we have
\begin{multline*}
    \regret_T(\ptx^*) \leq  \\
    \left(1+\frac{\|\ptx_1-\ptx^*\|^2}{2\QQ}\right)\left(\sqrt{2\QQ\sum_{t\in[T]}\|\nabla_t\|^2}
    + \sqrt{\frac{\QQ}{2}}\max_t \|\nabla_t\|
    \right) \\
    + 2\min\{\diam,\ \|\ptx_1-\ptx^*\|+ \sqrt{2\QQ\tau}\}\|\nabla_\tau\|\,.
\end{multline*}

Now the main difference with the SOLO-FTRL bound~\citep{orabona2018solo}
is that we have the term $\sqrt{\tau}\|\nabla_\tau\|$ instead of $\sqrt{T}\max_{t\in[T]}\|\nabla_t\|$.
Naturally, we wish to understand how much of an improvement this is.
First of all, since $\tau \leq T$, our bound is never worse.
Furthermore, since $\delta_\tau \geq \Delta_{\tau-1} \geq \sum_{t\in\barT\setminus\{\tau\}}  \delta_t$ we have $\|\nabla_\tau\| \geq \sum_{t\in\barT\setminus\{\tau\}} \|\nabla_t\|$.
This means that a null update is triggered only when the loss is larger
than the cumulative losses on all previous null updates.
We can draw a few observations: 
(i) unless losses grow exponentially forever (which makes learning rather difficult), $\tau$ has a finite value even as $T$ grows;
(ii) if losses are i.i.d.\! from a bounded distribution, large losses are observed exponentially fast and thus $\tau$ is necessarily small;
(iii) when performing gradient descent on a convex surface, the largest losses are usually observed during the first few steps and generally decreasing thereafter, then $\tau=O(1)$;
(iv) it is not obvious from the analysis of SOLO-FTRL whether the $\sqrt{T}$ term can be reduced.

The most likely `bad' scenario is when gradients with very small norm $\|\ell_1\|$ are observed for $T-1$ steps and on the $T$-th step a large loss $L$ is incurred
(possibly $\ptx_1$ starts on a plateau and the loss function is only quasiconvex).
Moreover, we will prove in \cref{apdx:ftrl} that $\|\nabla_\tau\|^2 \geq \sum_{t < \tau} \|\nabla_t\|^2$.
This means that in this scenario, to trigger a null update, we must have $\|\ell_T\| \geq \sqrt{T}\|\ell_1\|$, even if only one null update has been triggered before.
Hence, as long as losses do not grow faster than a factor $\sqrt{T}$, no null update is triggered.

Hence, for most situations, the adaptivity to small quadratic losses $\sqrt{\sum_t \|\nabla_t\|^2}$ is restored.

\section{APPLICATIONS}

We apply our tools to several algorithms: scale-free FTRL, scale-free Mirror Descent, and scale-free Prod, scale-free Adapt-ML-Prod, and also a slight generalization of AdaHedge (\cref{apdx:adahedge}) and a refinement of Soft-Bayes (\cref{apdx:soft-bayes}).
We now consider linear losses $\loss_t\in\Reals^N$.
For the first two applications,
the instantaneous regret is $r_t = \innerprod{\ptx_t - \ptx^*, \loss_t}$.

\subsection{Scale-Free FTRL}\label{sec:ftrl}

In this section we show how the results we obtain for online gradient descent in \cref{sec:null_updates}
can be generalized to FTRL with any regularizer that is 1-strongly convex to a given norm $\|\cdot\|$.
As discussed in \cref{sec:null_updates}, we use null updates to restore the adaptivity to small losses for unbounded domains, that was lost for the SOLO FTRL regret bound whose regret is dominated by the $O(L\sqrt{T})$ term,
where $L=\max_{t\in[T]} \|\loss_t\|_*$.

We first consider FTRL with isotuning but without null updates
(similar to AdaFTRL~\citep[appendix B.1]{orabona2018solo}), then we construct a meta-algorithm that filters out the exceedingly large losses before sending the losses to the subalgorithm, but still updates $\Delta$. 
The resulting (compound) algorithm is called isoFTRL (see \cref{alg:FTRLnull}).

\begin{algorithm}
\begin{lstlisting}
def isoFTRL($\ptxset, R, \QQ>0$):
  $L_0 = 0$; $\Delta_0 = 0$
  for t = 1, 2, 3, ...:
    $\rate = \QQ/\Delta_{t-1}$ # $\lstcommentcolor{\rate_{t-1}}$
    $\ptx_t = \argmin_{\ptx\in\ptxset}\left\{ \innerprod{\tilde L_{t-1}, \ptx} + \frac{R(\ptx)}{\rate} \right\}$
    predict $\ptx_t$;  observe $\loss_t$
    if $\tfrac12\rate\|\loss_t\|_*^2 \leq \Delta_{t-1}$: # or $\lstcommentcolor{\sqrt{\QQ/2}\|\loss_t\|_* \leq \Delta_{t-1}}$
      $\delta_t = \tfrac12\rate\|\loss_t\|_*^2$ # from (*\lstcommentcolor{\cref{eq:generic_deltat_ob1}}*)
      $\tilde{\loss}_t = \loss_t$
    else:
      $\delta_t = \sqrt{\tfrac12\QQ}\|\loss_t\|_*$ # isopoint
      $\tilde{\loss}_t = 0$  # null update
      
    $\tilde L_{t} = \tilde L_{t-1} + \tilde{\loss}_t$
    $\Delta_t = \Delta_{t-1} + \delta_t$ # isotuning
\end{lstlisting}
\caption{FTRL with null updates and isotuning}
\label{alg:FTRLnull}
\end{algorithm}

First we recall (and slightly generalize) an intermediate bound for AdaFTRL.

\begin{lemma}[AdaFTRL, \citet{orabona2018solo}, Lemma 5]\label{thm:ftrl_isotuning}
Consider \cref{ass:generic_isotuning},
and
$\forall t:\delta_t \geq \breg_{R^*}(-\rate_{t-1}L_t, -\rate_{t-1}L_{t-1})/\rate_{t-1}$.
Then the regret of FTRL with isotuning
to any $\ptx^*\in\ptxset$ is bounded by
\begin{align*}
    \regret_T(\ptx^*) \leq \left(1+\frac{R(\ptx^*)}{\QQ}\right)\Delta_T\,.
    &\qedhere
\end{align*}
\end{lemma}
We now describe isotuning FTRL with null updates, see \cref{alg:FTRLnull}.
Recall that the regularizer $R$ is 1-strongly convex with respect to the norm $\|\cdot\|$.
Define $\hat{g}_t(y) = \frac{\QQ}{2y}\|\loss_t\|_*^2\,$,
and $\hat{g}_t(y) =\infty$ for $y \leq 0$.
From \cref{sec:null_updates} we take $c=1$, and recall that
$\barT = \{t\in[T]: \hat{g}_t(\Delta_{t-1}) > \Delta_{t-1}\}$ 
is the set of steps where a null update is performed,
then we run FTRL with modified losses 
$\tilde{\loss}_t = \loss_t\indicator{t\notin\barT}$.
Furthermore, define
\begin{align*}
\delta_t =
\begin{cases}
    \iso_{y > 0}\left(\QQ\frac{\|\loss_t\|_*^2}{2y},\ y\right)=  \sqrt{\tfrac12\QQ}\|\loss_t\|_* &\text{ when } t\in\barT\,, \\
    \hat{g}_t(\Delta_{t-1})  &\text{ otherwise.}
\end{cases} 
\end{align*}
Then $\delta_t \geq \breg_{R^*}(-\rate_{t-1}\tilde L_t, -\rate_{t-1}\tilde L_{t-1})/\rate_{t-1}$
by proposition 2 of \citet{orabona2018solo}.

Then we can show the following bound. The proof is in \cref{apdx:ftrl}.
\begin{theorem}\label{thm:FTRLnull}
The regret of isoFTRL (\cref{alg:FTRLnull}) is bounded by
\begin{multline*}
    \regret_T(\ptx^*) \leq \\ 
    \left(\sqrt{\QQ}+\frac{R(\ptx^*)}{\sqrt{\QQ}}\right)
    \left(\sqrt{\sum_{t\in[T]} \|\loss_t\|_*^2} + \tfrac{1}{\sqrt2}\max_t \|\loss_t\|_*\right)\\
    + 2\min\left\{\diam,\ \sqrt{2R(\ptx^*)} + 2\sqrt{2\QQ\tau}\right\}
    \|\loss_\tau\|_* \,,
\end{multline*}
where $\tau=\max \barT$ is the last step a null update is performed, and is such that 
$\|\loss_\tau\|_*^2\ \geq\ \sum_{t< \tau}\|\loss_t\|_*^2$.
\end{theorem}

See the remarks at the end of \cref{sec:null_updates}.

\subsection{Scale-Free Mirror Descent}\label{sec:md}

\citet{orabona2018solo} provide two lower bounds showing that their scale-free
Mirror Descent (meant for bounded maximum Bregman divergence between any two points) fails with (super)linear regret: the first one is for the quadratic regularizer (gradient descent), and the second one is for the entropic regularizer for which the Bregman divergence is the relative entropy, which is not bounded even though the regularizer is bounded.

We now propose a scale-free variant of the Mirror Descent algorithm
with null updates, isotuning and online correction,
for bounded and unbounded domains,
at least for Bregman divergences that are convex in their second argument,
including the quadratic and the entropic regularizers.
The two lower bound examples mentioned above do not apply to our algorithm, 
thanks to using the online correction to center the bound,
leading to a factor $\breg_R(\ptx^*, \ptx_1)$ instead of $\sup_{(\ptx, \ptx')\in\ptxset^2}\breg_R(\ptx, \ptx')$.
Choosing $\ptx_1=\argmin_{\ptx\in\ptxset} R(\ptx)$ ensures that 
$\phi_1 = \breg_R(\ptx^*, \ptx_1) \leq R(\ptx^*) - R(\ptx_1)$ similarly to FTRL,
but in general one should take
$\ptx_1 = \argmin_{\ptx\in\ptxset} \max_{\ptx^*\in\ptxset} \breg_R(\ptx^*, \ptx)$.
In particular, when $R(\ptx) = \tfrac12\|\ptx\|_2^2$, then by Jung's theorem
we have $\|\ptx^* -\ptx_1\|_2 \leq D/\sqrt{2}$ and thus we can take
$\QQ = (D/\sqrt{2})^2/2 = D^2/4$. See additional discussion in \cref{ex:md_ogd}.

The algorithm is described in \cref{alg:MD_unbounded}.
For null updates, we take $c=1$ also.

\begin{algorithm}
\begin{lstlisting}
def isoMD($\ptxset$, $R$, $\QQ >0$):
  $\Delta_0 = 0$ ; $\ptx_1 = \argmin_{\ptx\in\ptxset} \max_{\ptx'\in\ptxset}\breg_R(\ptx', \ptx)$
  for t = 1, 2, ...:
    predict $\ptx_t$; observe $\loss_t$
    $\rate = \QQ/\Delta_{t-1}$
    if $\tfrac12\rate\|\loss_t\|_*^2 \leq \Delta_{t-1}$: # or $\lstcommentcolor{\sqrt{\QQ/2}\|\loss_t\|_* \leq \Delta_{t-1}}$
      $\ptx_{t'} = \argmin_{\ptx\in\ptxset}\left\{\innerprod{\ptx, \loss_t} + \frac{\breg_R(\ptx, \ptx_t)}{\rate}\right\}$
      $\delta_t = \tfrac12\rate\|\loss_t\|_*^2$ # from (*\lstcommentcolor{\cref{eq:generic_deltat_ob1}}*)
    else:
      $\ptx_{t'} = \ptx_t$  # null update
      $\delta_t = \sqrt{\tfrac12 \QQ}\|\loss_t\|_*$ # isopoint
      
    $\Delta_t = \Delta_{t-1} + \delta_t$ # isotuning
    $\ptx_{t+1} =\frac{\Delta_{t-1}}{\Delta_{t}} \ptx_{t'} + \frac{\delta_t}{\Delta_t}\ptx_1$ # online cor.
\end{lstlisting}
\caption{Scale-free Mirror Descent for bounded and unbounded decision sets.}
\label{alg:MD_unbounded}
\end{algorithm}

In this section we take $\phi_t = \breg_R(\ptx^*, \ptx_t)$, 
and recall that $\phi(\cdot)$ must be convex.
\footnote{This is satisfied at least by $\breg_R(a, b) = \|a-b\|_2^2/2$ and
$\breg_R(a, b) = \sum_i a_i \ln(a_i/b_i)$.}

\begin{theorem}\label{thm:MDnull}
With $\phi_1 = \breg_R(\ptx^*, \ptx_1)$, the regret of isoMD (\cref{alg:MD_unbounded}) is bounded by
\begin{multline*}
    \regret_T(\ptx^*) \leq \\
    \left(\sqrt{\QQ}+\frac{\phi_1}{\sqrt{\QQ}}\right)
    \left(\sqrt{\sum_{t\notin\barT} \|\loss_t\|_*^2}
        + \tfrac{1}{\sqrt{2}}\max_{t\in[T]}\|\loss_t\|_*\right) \\
     + 2\min\{\diam,\ \|\ptx^* - \ptx_1\| + \sqrt{2\QQ\tau}\}\|\loss_\tau\|_*
     \,,
\end{multline*}
where $\tau=\max \barT$ is the last step a null update is performed, and is such that 
$\|\loss_\tau\|_*^2\ \geq\ \sum_{t< \tau}\|\loss_t\|_*^2$.
\end{theorem}

The proof is given in \cref{apdx:md}.
Also see the remarks at the end of \cref{sec:null_updates}.

\subsection{Scale-free Adaptive Prod}\label{sec:prod}

The Prod family of algorithms \citep{cesabianchi2007prod,gaillard2014secondorder,sani2014abprod,orseau2017softbayes}
are based on Mirror-Descent style multiplicative updates,
but their analysis does not fit either the MD nor the FTRL analysis framework.
Previous works have resorted to more or less ad-hoc analysis.
They can be analysed, however, with our generic off-by-one setting and online correction, and only the quantity of interest $\delta_t$ needs to be bounded as per \cref{eq:generic_deltat_ob1} on an algorithm-dependent basis,
taking $\phi(\ptx) = \RE{\ptx^*}{\ptx}=\sum_{i\in[N]}\ptx^*_i \ln (\ptx^*_i/\ptx_i)$.
Also see \cref{apdx:soft-bayes} for a slightly simplified sparse Soft-Bayes algorithm~\citep{orseau2017softbayes} and enhanced regret bound.

In this section, we demonstrate how to obtain a scale free
and adaptive version of the Prod algorithm~\citep{cesabianchi2007prod,gaillard2014secondorder} with a single learning rate.
This algorithm was originally analyzed with constant learning rates and a doubling trick with restarts~\citep{cesabianchi2007prod}.
While this single learning-rate algorithm is not so useful in itself in practice as it can be advantageously replaced with AdaHedge~\citep{derooij2014follow},
the example is technically useful as it shows how some constraints of the update rule
can be easily tackled with null updates, while demonstrating how the tools
we built in the previous sections lead to a straightforward analysis of the regret.
It also serves as a prelude to the next section where we analyse the multiple learning rate version. It also uses \cref{lem:log_approx} which tightens the leading 
constant factor by $\sqrt{2}$ on lower bounding $\log(1+x)$ compared to 
Lemma 1 of \citet{cesabianchi2007prod}.

We use \cref{ass:generic_isotuning}, and \cref{alg:generic_ob1} with null updates---but without a doubling trick or any restarting.
The decision set $\ptxset$ is the probability simplex of $N$ coordinates,
and we take $\phi(\ptx) = \RE{\ptx^*}{\ptx}$
and $\QQ = \phi_1 = \ln N$.

For all $t\in[T]$, let $\loss_t \in \Reals^N$, and define $\bar\loss_t = \sum_{i\in[N]} \ptx_{i, t}\loss_{i, t}$.

Define the prod update (off by one):
\begin{align}
    \ptx_{1} &= \left(\frac{1}{N}, \dots, \frac{1}{N}\right) \notag\\
    \forall i\in[N]: \ptx_{i, t'} &= \ptx_{i, t}(1+\rate_{ t-1}(\bar\loss_t - \loss_{i, t})) \label{eq:prod}
\end{align}

For all $t$, define $s_t = \max_i |\bar\loss_t - \loss_{i, t}|$ 
and $\hat{g}_t(x) = \frac{\QQ s_t^2}{2x}$,
with $\hat{g}_t(x) = \infty$ for $x \leq 0$.

As per \cref{eq:generic_deltat_ob1},
for $t\notin\barT$, we can define 
\begin{align*}
    \delta_t = \max_i \left(\frac{-1}{\rate_{t-1}}\ln(1+\rate_{t-1}(\bar\loss_t - \loss_{i, t})) + (\bar\loss_t - \loss_{i, t})\right)\,.
\end{align*}
By \cref{cor:rate_log_approx},
assuming $\rate_{t-1} |\bar\loss_t - \loss_{i, t}| < 1$
(which is also required for the update in \cref{eq:prod} to be valid),
\begin{align*}
    \delta_t \leq \frac{(\bar\loss_t - \loss_{i, t})^2/2}{\frac1{\rate_{t-1}}-|\bar\loss_t - \loss_{i, t}|} 
    \leq \frac{s_t^2/2}{\frac1{\rate_{t-1}} - s_t}
    = \frac{q s_t^2/2}{\Delta_{t-1} - q s_t}
\end{align*}
and thus
\begin{align}\label{eq:prod_deltat_bound}
    \delta_t &\leq \hat{g}_t(\Delta_{t-1} - \QQ s_t) = \hat{g}_t(\Delta_t - \delta_t - \QQ s_t)\,.
\end{align}
The condition $\rate_{t-1} |\bar\loss_t - \loss_{i, t}| < 1$
is implied by the condition $\hat{g}_t(\Delta_{t-1}-\QQ s_t) < \infty$,
so we choose to perform a null update ($t\in\barT)$ whenever $\hat{g}_t(\Delta_{t-1}-\QQ s_t) > 2(\Delta_{t-1} - \QQ s_t)/q$,
in which case we set
\begin{equation*}
    \forall t\in\barT:\quad \delta_t = \iso_{y>0}(\hat{g}_t(y), 2y/q)=s_t\,.
\end{equation*}
Using \cref{thm:isopoint}, all this ensures that
\begin{align*}
    \forall t\notin\barT:{}& \delta_t \leq \hat{g}_t(\Delta_{t-1} - \QQ s_t) \leq 
    s_t \textcolor{gray}{{}\leq 2(\Delta_{t-1} - \QQ s_t)/q}\,,\\
    \forall t\in\barT:{}& 2(\Delta_{t-1} - \QQ s_t)/q \leq{} \delta_t = s_t \leq \hat{g}_t(\Delta_{t-1} - \QQ s_t)\,.
\end{align*}
We call isoProd the resulting algorithm.

\begin{theorem}
Let $S = \max_{t\in[T]} s_t$.
The regret of isoProd is bounded by
\begin{align*}
    \regret_T(\ptx^*)
    \leq
    2\sqrt{\ln (N)\sum_t s_t^2} + 2S(1+\ln N)\,.
    &\qedhere
\end{align*}
\end{theorem}
\begin{proof}
By \cref{thm:generic_regret_null} we have $\regret_T(\ptx^*) \leq 2\Delta_T + \sum_{t\in\barT}(r_t-\delta_t)
\leq 2\Delta_T$ since $r_t \leq s_t = \delta_t$ for $t\in\barT$.
Let $\tuple{\hat\Delta, \hat{g}}$ be an isotuning sequence,
then by \cref{lem:isotuning_offset}
and \cref{eq:prod_deltat_bound} we know that
$\Delta_T \leq \hat\Delta_T + \max_t \{\delta_t + \QQ s_t\}
\leq \hat\Delta_T + S + \QQ S$.
Finally by \cref{lem:Delta_sqrt_short} we have $\hat\Delta_T \leq \sqrt{\QQ\sum_t s_t^2}$.
\end{proof}

\begin{remark}
Alternatively, we could use the update
\begin{align*}
    \ptx_{i, t'} &= \ptx_{i, t}\frac{1-\rate_{t-1}\loss_{i, t}}{1-\rate_{t-1}\bar\loss_t}
\end{align*}
which leads to $\delta_t = \max_i \rate_{t-1}\frac{\loss_{i, t}^2/2}{1-\rate_{t-1}\loss_{i, t}}$
and thus
$\regret_T(\ptx^*) = O(\sqrt{\QQ\sum_t \max_i \loss_{i, t}^2})$
and then use \cref{lem:simplex_translate_losses} to sequentially translate the losses arbitrarily.
\end{remark}

\subsection{Scale-free Adapt-ML-Prod}\label{sec:mlu-prod}

\citet{gaillard2014secondorder} propose the Adapt-ML-Prod algorithm, for losses in $[0, 1]$ 
with multiple learning rates,
to obtain a regret bound which depends only on the excess losses $\sqrt{\sum_t (\bar\loss_t - \loss_{i, t})^2}$ to the best expert $i$ in hindsight,
where $\bar\loss_t$ is the loss incurred by the algorithm and $\loss_{i, t}$ is the loss incurred by the expert $i$ at step $t$
and show interesting consequences of this bound, in particular it is bounded
by a constant in expectation when losses are i.i.d.
By contrast, AdaHedge~\citep{derooij2014follow,erven2011adaptive} can tackle losses with unknown range, but has a regret bound featuring the quantity $\sqrt{\sum_t \sum_{i\in[N]}\ptx_{i, t}(\bar\loss_t - \loss_{i, t})^2}$ instead.
We extend Adapt-ML-Prod to tackle unbounded and unknown loss ranges using isotuning, online correction, and null updates---or equivalently we extend isoProd of the previous section to using one learning rate per expert.
\footnote{\label{foot:boa}
\citet{wintenberger2016bernstein} 
proposes an adaptive algorithm based on exponential
weights that is claimed to adapt to the maximum excess losses to the best expert in hindsight (see p.4, top, def. of $\ell_{j, t}$, and p.13, Theorem 3.3).
Unfortunately, 
the main results rely on the following wrong claim to hold (proof of Theorem 1.1):
``as $X$ is centered, we can bound [\dots] $1-\expect[X\mathds{1}_{X<-1/2}] \leq \expect[1+X\mathds{1}_{X<-1/2}]$''.
It is not clear whether this can be fixed simply, in particular 
because the preconditions of Lemma 4 in \citet{cesabianchi2007prod} need to be  satisfied. See \cref{rmk:boa} in \cref{apdx:mlu-prod} instead.
}

\citeauthor{gaillard2014secondorder} consider losses in $[0, 1]$, and their algorithm predicts with positive weights $\propto \ptx_{i, t}\rate_{i, t-1}$ and suffers the loss 
\begin{align}\label{eq:mlprod_barloss}
    \bar\loss_t = \sum_{i\in[N]}\frac{\ptx_{i, t}\rate_{i, t-1}}{\sum_{j\in[N]}\ptx_{j, t}\rate_{j, t-1}}\loss_{i, t}\,.
\end{align}
For all $T\in\Naturals, \forall i\in[N]$,
define $r_{i, t}=\bar\loss_t - \loss_{i, t}$ and $\regret_{i, T} = \sum_{t\in[T]} r_{i, t}$,
also $V_{i, T} = \sum_{t\in[T]} r_{i, t}^2$,
and $s_t = \max_{i\in[N]}|r_{i, t}|$.
Assuming $N\geq 2$ and $T\geq 2$, we can rewrite the regret bound of Adapt-ML-Prod~\citep[Corollary 4]{gaillard2014secondorder}
(with losses in $[0, 1]$)
to the best expert $i$ in hindsight as:
\begin{multline*}
    \hspace{-1em}
    \regret_{i, T} \leq  
    \left[\sqrt{\left(1+V_{i, T}\right)\ln N} + 3\ln N\right]
    \!\left[4 + \frac{1+\ln \ln T}{\ln N}\right].
\end{multline*}
Note, for $T=10^{40}$ and $N \geq 4$, $(1+\ln \ln T)/\ln N < 4$.

Adapt-ML-Prod is meant for losses in $[0, 1]$ and can be used for losses in any $[a, b]$ if an upper bound on $b-a$ is known.
However, when the loss range is not known beforehand, as for isoProd we perform a null update when the update may lead to invalid (negative) weights.
We obtain an algorithm with a regret guarantee that is just as good as when the loss range is known beforehand.

As for Adapt-ML-Prod, our algorithm uses one learning rate per coordinate, 
and the tools of the previous sections will apply almost always coordinate-wise.
For all $i\in[N]$, we first define the update $\ptx_{i, t'}$ for Adapt-ML-Prod:
\begin{align}
    \ptx_{i, 1} &= 1\,, \notag\\
    \ptx_{i, t'} &= \ptx_{i, t}(1+\rate_{i, t-1}(\bar\loss_t - \loss_{i, t}))\,,
    \label{eq:mlprod}
\end{align}
and then apply the online correction of \cref{eq:generic_ob1_ocr}.
As for Prod, the problem with \cref{eq:mlprod} is that it is invalid when $\ptx_{i, t'}$ becomes negative.
We prevent this issue by performing null updates ($t\in\barT$) whenever
$\rate_{i, t-1}|\bar\loss_t-\loss_{i, t}| \geq 1/2$ for some $i\in[N]$,
that is, $\QQ |\bar\loss_t - \loss_{i, t}| \geq \Delta_{i, t-1}/2$.
Note that to ensure that $\sum_i \ptx_{i, t'} = \sum_i \ptx_{i, t}$ for all rounds $t$
(as will be shown in the proof of \cref{thm:mlu-prod}),
we perform a null update \emph{on all coordinates at the same time},
which leads to an additional difficulty compared to isoProd.

We use \cref{ass:generic_isotuning}
and we take $\phi_i(\ptx_{i, t}) = -\ln\ptx_{i, t}$ for all $i\in[N]$, and $\ptxset = [0, \infty)^N$.
Set $\QQ = \ln N$.
Recalling that $r_{i, t} = \bar\loss_t - \loss_{i, t}$ and $s_t = \max_i |r_{i, t}|$, define from \cref{eq:generic_deltat_ob1}
\begin{align}
    t\notin\barT:\delta_{i, t} &= \delta^*_{i,t}  \notag \\
    &=r_{i, t} - \frac{1}{\rate_{i, t-1}}\ln(1+\rate_{i, t-1}r_{i, t})\,, \label{eq:isomlprod_deltat}\\
    t\in\barT:\delta_{i, t} &= s_t\,. \notag
\end{align}
On a null update we set $\delta_{i, t} = s_t$,
so that $\delta_{i, t} \geq r_{i, t}$ conveniently absorbs the instantaneous regret;
note that we do not set $\delta_{i, t} = |\bar\loss_t - \loss_{i, t}|$
to avoid triggering null updates of the same scale $N$ times, which would lead to an additive term $O(N\max_t |r_{i, t}|)$ in the bound instead of $O(\max_{j, t} |r_{j, t}|)$.

Observe that $\phi_{i, 1} = 0$, so from \cref{thm:generic_regret_null},
since also $\delta_{i, t} \geq r_{i, t}$ for $t\in\barT$,
we can simplify:
\begin{align}\label{eq:mluprod_regret_simple}
    \regret_{i, T} \leq \left(1+\frac{\ln \ptx_{i, T+1}}{\ln N}\right)\Delta_{i, T}\,.
\end{align}
We call the resulting algorithm isoML-Prod (see \cref{alg:mlu-prod})
and we prove an adaptive scale-free bound.
The proof is given in \cref{apdx:mlu-prod}.

\begin{algorithm}
\begin{lstlisting}
def isoML_Prod():
  $\forall i\in[N]: \Delta_{i, 0} = 0,  \ptx_{i, 1} = 1$
  for t = 1, 2, ...:
    $\forall i\in[N]: \tilde{\ptx}_{i, t} = \ptx_{i, t}/\Delta_{i, t-1}$
    Predict normalized $\tilde{\ptx}$; observe $\loss_t$
    $\bar\loss_t = \sum_{i\in[N]} \tilde\ptx_{i, t}\loss_{i, t}\, /\, \sum_{i\in[N]} \tilde\ptx_{i, t}$
    
    for each $i\in[N]$:
      if $\max_{j\in[N]} \QQ |\bar\loss_t - \loss_{j, t}| < \Delta_{j, t-1}/2$:
        $r = \bar\loss_t - \loss_{i, t}$
        $\rate = \ln N / \Delta_{i, t-1}$
        # Adapt-ML-Prod (core) update:
        $\ptx'_{i, t} = \ptx_{i, t}(1+ \rate r)$
        $\delta_{i, t} = r - \frac{1}{\rate}\ln(1+\rate r)$ # isotuning
      else:  # $\lstcommentcolor{t\in \barT}$
        $\ptx'_{i, t} = \ptx_{i, t}$ # null update
        $\delta_{i, t} = \max_{j\in[N]}|\bar\loss_t - \loss_{j, t}|$
      $\Delta_{i, t} = \Delta_{i, t-1} + \delta_{i, t}$ # isotuning
      $\ptx_{i, t+1} = \ptx'_{i, t}\frac{\Delta_{i, t-1}}{\Delta_{i, t}} +\frac{\delta_{i, t}}{\Delta_{i, t}}\ptx_{i, 1}$ # online cor.
\end{lstlisting}
\caption{The isoML-Prod algorithm, which uses one learning rate per coordinate and does not have prior knowledge of the range of the losses.
}
\label{alg:mlu-prod}
\end{algorithm}

\begin{theorem}\label{thm:mlu-prod}
Assume $N\geq 2$.
We consider losses $\loss_{i, t}\in \Reals$ for all experts $i\in[N]$ and time steps $t\in\Naturals$.
Let $S = \max_{t\in[T]} s_t$.
Then for all $i\in[N]$, of isoML-Prod is bounded by
\begin{multline*}
    \hspace{-1em}
    \regret_{i, T} \leq
    \left[\sqrt{V_{i, T}\ln N} + S(2 + 3\ln N) \right]
    \left[2+ \frac{\ln(1+ C)}{\ln N}\right]
\end{multline*}
where simultaneously $C \leq T$,
$C \leq \tau + \ln \frac{ST}{s_\tau}$,
and also $C = o(1) + \ln ST$.
\end{theorem}

\section{RELATED WORK}\label{sec:related}

\paragraph{Null updates.}
Skipping updates was used in the bandit literature to avoid a large penalty due to long delays \citep{thune2019delays,zimmert2020delays}.
Concurrently to this paper,
\footnote{Both papers were submitted to AISTATS 2022, but we refrained from uploading our paper to arXiv before receiving the reviews to maintain anonymity.}
\citet{huang2021delays} extend this line of work
to obtain scale-free (but only partially adaptive) guarantees in the same context.

\paragraph{Online correction.}
Our use of the online correction comes from \citet{orseau2017softbayes}, which is based on the Fixed Share rule~\citep{herbster1998tracking}.
We have been informed after a first version of this paper that it is similar to the primal stabilization technique of \citet{pmlr-v119-fang20a}, who also propose a more general stabilization technique
in the dual space which does not require the Bregman divergence to be convex in its second argument.
Other correction terms appear in the literature for more specific purposes, \eg \citet{gaillard2014secondorder,chen2021impossible,huang2021delays}.

\paragraph{Loss clipping.}
To tackle the off-by-one issue, \citet{cutkosky2019hints} proposes a correction term of the loss scales, at the expense of an additive term $\max_t \|\ptx_t - \ptx^*\| \max_t \|\loss_t\|$, and requires a special treatment for $t=0$ akin to a null update.
In itself, this correction is not enough and must be combined with an additional mechanism such as the restarting of \citet{mhammedi2019squint} to obtain scale-free regret bounds.
The SOLO-FTRL bound can easily be recovered with this combination, but not apparently that of \cref{thm:FTRLnull} and, by contrast to null updates, it does not help dealing with constraint violations of the update rule as for Prod (\cref{sec:prod}), or the offsets in the Prod (\cref{sec:prod}) and AdaHedge (\cref{apdx:adahedge}) bounds.
Null updates, combined with isotuning and in particular \cref{lem:generic_ob1_DeltaT}, avoid the need to restart the algorithm altogether.

\paragraph{Isotuning.}
As mentioned before, the isotuning assumptions of \cref{ass:generic_isotuning} come from the AdaHedge update scheme~\citep{derooij2014follow}, and were adapted to FTRL in AdaFTRL only for domains with bounded diameter \emph{and} bounded Bregman divergence \citep{orabona2018solo}.
A form of isotuning also appears in \citet{bartlett2007adaptive} (see \cref{sec:aogd}), and the AdaHedge analysis has been followed by other works such as \citet[Section 5.2]{flaspohler2021delay}.
These works all use an ad-hoc analysis of the $\Delta_T$ quantity, 
but none provide the versatile result of \cref{thm:isotune}
(see the improved Soft-Bayes analysis in \cref{apdx:soft-bayes} for a prime example of this versatility),
nor the additional generic results such as \cref{lem:isotuning_offset} or \cref{lem:generic_ob1_DeltaT}.
Another ad-hoc isotuning lemma appears in \citet[Lemma 6.1]{campolongo2020temporal} --- but see \cref{apdx:campolongo}.
The off-by-one issue is trivially dealt within AdaHedge because the additive term $\max_t \delta_t$ is bounded by the loss range (see \cref{apdx:adahedge}),
 and is only partially dealt within AdaFTRL by requiring bounded domains
--- although the analysis could still be tightened using the isotuning tools, see \cref{apdx:campolongo}.
See also the comparison between isotuning and the more standard learning rate $\rate_t \propto 1/\sqrt{\sum_t \|\loss_t\|^2}$ in \cref{apdx:seq-opt}.

\paragraph{Unbounded domains.}
Let $L = \max_{t\in[T]}\|\loss_t\|, U_T = L\sum_{t\in[T]} \|\loss_t\|,V_T = \sum_{t\in[T]} \|\loss_t\|^2$.
For unbounded domains, FTRL-SOLO has regret $\mathcal{R}_1 = O((\sqrt{\QQ} + \frac{\|\ptx^*\|^2}{\sqrt{\QQ}})\sqrt{V_T} + L\sqrt{\QQ T})$,
which we improve to 
$\mathcal{R}_2 = O((\sqrt{\QQ}+\frac{\|\ptx^*\|^2}{\sqrt{\QQ}})\sqrt{V_T} + L\sqrt{\QQ\tau})$ where $\tau$ is the last step where $\|\loss_\tau\|^2 \geq \sum_{t<\tau}\|\loss_t\|^2$.
FreeRange \citep{mhammedi2020freerange} manage to obtain (almost)
$\mathcal{R}_3 = O(\|\ptx^*\|\sqrt{V_T\log (T\|\ptx^*\|)} + \sqrt{U_T} + L\|\ptx^*\|^3 + L\|\ptx^*\|\ln T)$, and they give a matching lower bound (in particular $U_T$ cannot be replaced with $V_T$ without modifying the tradeoff elsewhere).
$\mathcal{R}_2$ is a strict improvement over $\mathcal{R}_1$ since $\tau$ is expected to be a small constant in most situations.
When defaulting to $\QQ=1$ ($\QQ=\ln T$ could also be considered), $\mathcal{R}_2$ and $\mathcal{R}_3$ are not comparable in general:
in the worst case
$\mathcal{R}_3 < \mathcal{R}_2$ by a less-than-constant factor $\|\ptx^*\|/\sqrt{\ln T}$ when $V_T = U_T = \Theta(L\sqrt{T})$ but, assuming that $\tau$ is a small constant, then $\mathcal{R}_2 < \mathcal{R}_3$ whenever $\|\ptx^*\| \leq \sqrt{\ln T}$ or $\|\ptx^*\| \geq \sqrt{V_T}/L$.
The remaining cases are not easy to compare since we may have $V_T \ll U_T$,
\eg if $\loss_t = 100/\sqrt{t}$
(thus $\tau=1)$
and $T=100\,000$,
then $L\sqrt{T} \approx 31\,000, \sqrt{U_T}\approx 2\,500, \sqrt{V_T} \approx 350$
(and observe that $\sqrt{\ln T} \approx 3.39$ and $\sqrt{V_T}/L \approx 3.48$).
The FreeRange bound also has a leading $\sqrt{24\ln T}$ factor and an additive
$48 L \|\ptx^*\|\ln T$ term which prevent constant regret for cases where $V_T =O(1)$.
Furthermore, FreeRange and FreeGrad results are derived only for the 2-norm, while FTRL and Mirror Descent can be used with any norm and, more generally, convex regularizers and Bregman divergences.

Hence restoring the main dependency on $V_T$ is desirable and non-trivial,
and it can lead to finite regret with stochastic losses, adaptation to smoothness and to strong convexity and has various applications~\citep[ICML 2020 parameter-free tutorial\footnote{\url{https://parameterfree.com/icml-tutorial}}]{gaillard2014secondorder}.

For bounded domains, see the discussion in \cref{ex:md_ogd}.

\paragraph{Squint+L and isoML-Prod.}
\citet{mhammedi2019squint} proposed the Squint+L algorithm for the simplex, 
which impressively mixes a continuous range of learning rates in closed form.
It enjoys a scale-free adaptive regret bound of $O(\sqrt{(\RE{\ptx^*}{\ptx_1}+\ln \ln T)\sum_i \ptx^*_i V_{i, T}})$ within $O(N)$ computation steps per round.
With $\ptx_1$ in the simplex, the bound for isoML-Prod can be shown to be $O(\sum_i \ptx^*_i (1+\frac{\ln \ln T}{-\ln \ptx_{i, 1}})\sqrt{\ln (1/\ptx_{i, 1})V_{i, T}})$
(assuming losses do not grow exponentially).
This is pretty close (using Cauchy-Schwarz, see \cref{rmk:isoprod_x*}, p.\pageref{rmk:isoprod_x*}) but not quite the Squint+L bound.
The main advantage of isoML-Prod is its far better numerical stability:
Even a careful (and complicated) implementation
\footnote{\url{https://bitbucket.org/wmkoolen/squint},
modified for Squint+L according to \citet[Algorithm 1]{mhammedi2019squint},
and $R - V$ replaced with $R-V/B_{\tau_1}$ in the weight update.} of Squint
fails with linear regret for losses drawn \emph{uniformly} in [0, 1e-3];
Squint+L works a little better but starts failing for [0, 1e-5] losses,
which defeats the scale-free guarantee.
By contrast, a straightforward implementation of isoML-Prod as per \cref{alg:mlu-prod}
on standard IEEE double precision floats works evenly (with near-constant regret)
even for losses drawn uniformly in [0, 1e-300].
This is because the algorithm itself is invariant to loss rescaling, by contrast to Squint+L.
Furthermore, whether Adapt-ML-Prod could be extended to have scale-free guarantees 
(and whether BOA could be fixed) without a doubling trick
was still an open problem (see also \cref{rmk:boa} in \cref{apdx:mlu-prod}).

\citet{chen2021impossible} propose a new algorithm with similar guarantees to Squint (but not Squint+L),
and use the techniques designed for the latter to derive partially scale-free regret bounds.
However, their term $V_T$ is lower bounded by 3 preventing the regret bound from being properly scale-free for small losses---by contrast to the Squint+L bounds and our bounds.
This appears to be due to constraints on the learning rate,
hence it seems plausible that null updates could be used to restore the scale-free property.
This algorithm also requires $O(\log T)$ learning rates and as many copies of the algorithm,
and the learning rates impractically depends on the initial guess $B_0$ of the scale of the losses.

\section{CONCLUSION}

Isotuning, null updates and the online correction can help design and analyze anytime, scale-free and adaptive online learning algorithms for constrained and unconstrained domains.
One particular advantage of our approach is that the algorithms do not need to be restarted.
All algorithms studied also run as fast as online gradient descent,
and the tools we developed do \emph{not} introduce any additional line search steps or logarithmic factors, neither in computation nor in the regret bounds.

\ifarxiv
\newcommand{\acknowledgments}[1]{
\subsubsection*{Acknowledgements}#1}
\fi
\acknowledgments{
The authors would like to thank Tor Lattimore, Andr\'as Gy\"orgy, and Pooria Joulani as well as anonymous reviewers for their useful feedback.
}

\bibliographystyle{unsrtnat}
\bibliography{biblio}

\appendix

\onecolumn
\ifarxiv\else
\aistatstitle{Isotuning With Applications To Scale-Free Online Learning: \\
Supplementary Materials}
\fi

\section{ISOTUNING: ADDITIONAL RESULTS AND REMARKS}\label{apdx:isotuning}

\begin{proposition}[Existence and uniqueness of $X_t$]\label{prop:Xt_exists_unique}
For all $t\in\Naturals$,
if assumptions (i) and (ii) of \cref{def:isotuning_seq} are satisfied,
then $X_t$ as defined in assumption (iv) exists and is unique.
\end{proposition}
\begin{proof}
Consider any $t\in\Naturals$.
Assume $X_{t-1} < \infty$, which is true for $t=1$.
Let $h_t(x) = X_{t-1} + g_t(x) - x$ for $x \geq 0$.
Since $g_t$ is non-increasing and continuous, 
then $h_t$ is strictly decreasing and continuous.
Since $h_t(0) \geq X_{t-1} \geq 0$
(by induction on $X_{t-1}$)
and $\lim_{x\to\infty}h_t(x) = -\infty$
(by condition (ii) and the assumption $X_{t-1} < \infty$), then by the intermediate value theorem $h_t$ has exactly one root, $X_t < \infty$.
Then the result holds by induction.
\end{proof}

\addtocounter{mytheorem}{1}
\begin{remark}
Since $g_t$ is nonnegative, necessarily $X_{t-1} \leq X_t$.
\end{remark}

\subsection{Tighter Isotuning Bound}\label{apdx:isotuning_refined}

The factor 2 guaranteed by \cref{thm:isotune}, although small, is a little loose (by at most a factor $\sqrt{2}$), in particular in the main region
of interest, around $O(\sqrt{T})$ regret.
For a class of cases, we can show a tighter constant with the following lemma, which merely uses convexity.
(See also \citet[Lemma 14]{gaillard2014secondorder}, but their lemma is suboptimal for values greater than 1, as it leads to a \emph{factor} of the largest $\delta_t$.)

\begin{lemma}\label{lem:isotuning_bound_convex}
Let $(T_1, T_2)\in\Nonnegints^2$ and $T_1 \leq T_2$,
and for all $t\in[T_1+1..T_2]$,
let $\delta_t \geq 0$ and $\Delta_t = \Delta_{t-1} + \delta_t$.
Let $f:[0, \infty)\to\Reals$ be a convex function (that is, $f'$ is increasing). Then
\begin{align*}
    \sum_{t= T_1+1}^{T_2} \delta_t f'(\Delta_{t-1})
    \ \leq\ f(\Delta_{T_2}) - f(\Delta_{T_1})
    \ \leq\ \sum_{t= T_1+1}^{T_2} \delta_t f'(\Delta_t)\,.
\end{align*}
If $f$ is concave ($f'$ is decreasing), then the inequalities are in the other directions.
\end{lemma}
\begin{proof}
By convexity of $f$ we have
\begin{align*}
    f(\Delta_t) &\leq f(\Delta_{t-1}) +(\Delta_t - \Delta_{t-1})f'(\Delta_t) \\
            &= f(\Delta_{t-1}) + \delta_t  f'(\Delta_t)\,,\\
    f(\Delta_t) &\geq f(\Delta_{t-1}) +(\Delta_t - \Delta_{t-1})f'(\Delta_{t-1}) \\
                &= f(\Delta_{t-1}) + \delta_t f'(\Delta_{t-1})\,,
\end{align*}
and the result follows by summing over $t$ from $T_1+1$ to $T_2$.
\end{proof}

See \cref{apdx:adahedge} for a use case of \cref{lem:isotuning_bound_convex} in the off-by-one setting.

\begin{lemma}\label{lem:Delta_sqrt}
Let $\tuple{X, g}$ be an isotuning sequence where for all $t:g_t(x) = \frac{a_t}{x}$ with $a_t \geq 0$. Then for all $T\in\Naturals$,
\begin{align*}
    \sqrt{A^2 + 2\sum_{t\in[T]}a_t} - A
    \ \leq\ X_T
    \ \leq\ \sqrt{2\sum_{t\in[T]}a_t}
\end{align*}
where $A=\max_{t\in[T]}a_t$.
\end{lemma}
\begin{proof}
For the upper bound, we use the upper bound of \cref{lem:isotuning_bound_convex} with $f(x)=x^2$
and $T_1=0$, then $f'(X_t)2X_t$ and thus
\begin{align*}
    X_T^2 \leq \sum_{t\in[T]} \frac{a_t}{X_t}2X_t\,.
\end{align*}
For the lower bound, we use the lower bound of \cref{lem:isotuning_bound_convex} with $f(x)=(x + A)^2$,
then $f'(X_{t-1}) = 2(X_{t-1} + A) \geq 2X_t$ and thus
\begin{align*}
    (X_T+A)^2 \ \geq\  (X_0+A)^2 + \sum_{t\in[T]} \frac{a_t}{X_t}2(X_{t-1}+A)
    \ \geq\ A^2 + 2\sum_{t\in[T]} \frac{a_t}{X_t}
    \,,
\end{align*}
and taking the square root and subtracting $A$ on both sides finishes the proof.
\end{proof}

\begin{example}\label{ex:refined_bound}
Take the regret bound \cref{eq:regret_intro} of the introduction.
With $\delta_t = a_t \rate_t= a_t\QQ/ \Delta_t$ we can take $f(y) = y^2$ so that from \cref{lem:isotuning_bound_convex} we have
\begin{align*}
    f(\Delta_T)-f(\Delta_0) &\leq \sum_{t\in[T]} \frac{a_t\QQ}{\Delta_t}\cdot 2\Delta_t = \QQ\sum_{t\in[T]} a_t\,, \\
    \Delta_T &\leq \sqrt{2\QQ\sum_{t\in[T]} a_t}\,,
\end{align*}
and thus $\regret_T(\ptx^*) \leq 2\sqrt{2\QQ\sum_{t} a_t}$,
which is tighter by a factor $\sqrt{2}$ than by using \cref{thm:isotune},
and is only a factor $\sqrt{2}$ away from the hindsight optimal constant learning rate.
We do not know of any sequential learning rate that achieves a better rate without prior knowledge of $T$.
\end{example}

\subsection{AOGD Revisited}\label{apdx:aogd}

The update of $\lambda_t$ satisfying $\lambda_t = C_t/(\alpha_{1:t} + \lambda_{1:t-1}+\lambda_t)$ of AOGD~\citep{bartlett2007adaptive} is indeed an instance of an isotuning update in disguise, since it is equivalent to
\begin{align*}
    \lambda_{1:t} &= \lambda_{1:t-1} + \frac{C_t}{\alpha_{1:t} + \lambda_{1:t}}\,,
\end{align*}
where $\lambda_{1:T} = \sum_{t\in[T]}\lambda_t$, and $\alpha$ is as in \cref{sec:aogd}.
The authors prove~\citep[Lemma 3.1]{bartlett2007adaptive} 
an analog of \cref{thm:isotune} for this particular definition of $\lambda$.
Recall their Theorem 1.1 for AOGD:
\begin{align*}
    \regret_T(\ptx^*)\ \leq\ 3\inf_{\lambda^*_1, \dots, \lambda^*_T \geq 0}
    \left\{\lambda^*_{1:T}\diam^2 + \sum_{t\in[T]} \frac{(\|\nabla_t\| + \lambda^*_t \diam)^2}{\lambda^*_{1:t} + \alpha_{1:t}}\right\}\,.
\end{align*}
\comment{The display before Eq. (3) on p.5 of the AOGD paper is confusing,
because the first occurrence of $G_t^2$ should probably be $\tilde G_t^2$,
which is the gradient of the function plus the additional regularizer.}
Our \cref{thm:isotune} generalizes the set of conditions for which local balancing is close to optimum, while simplifying the result.
In \cref{sec:aogd} we obtain for the isotuning variant of AOGD:
\isoaogdregret
This bound is both simpler and a little tighter than the one of the original work.
Observe in particular the presence of $\lambda^*_t\diam$ in the numerator of the first bound, which is due to the additional regularizer.
By contrast we apply isotuning directly to the learning rate --- which does not modify the gradient and thus does not introduce an extra term.
It can also be shown that the sequence of real-valued variables
$\lambda^*_1, \dots, \lambda^*_T$
that parameterizes the $\inf$ operation offers no advantage compared to a single variable $z$.
Note also the brevity of the analysis of the learning rate in \cref{sec:aogd}.
This generalized and simplified form allows us to tackle a wide range of situations, as witnessed by the multiple applications in this paper.

The following lemma is relatively standard and makes no use of isotuning per se.
The set $\ptxset$ is assumed to be convex, closed and bounded.
The projection $\proj{\ptxset}{x}$ is according to $\|\cdot\|_2$, that is
$
    \proj{\ptxset}{x} = \argmin_{y\in\ptxset}\|x - y\|_2^2\,.
$

\begin{lemma}\label{lem:aogd}
Assume the loss functions $\loss_t$ are $\alpha_t$-strongly convex for all $t\in[T]$,
and that $1/\rate_t - \alpha_t -1/\rate_{t-1} \geq 0$.
Then the regret of online gradient descent with a time varying learning rate $\rate_t$ is bounded by
\begin{align*}
    2\regret_T(\ptx^*) \leq \max_{t\in[T]}\|\ptx_t - \ptx^*\|_2^2\left(\frac{1}{\rate_T}-\alpha_{1:T} - \frac{1}{\rate_0}\right) + \sum_{t\in[T]}\rate_t\|\nabla_t\|_2^2 + \frac{1}{\rate_0}\|\ptx_1 - \ptx^*\|_2^2\,.
    &\qedhere
\end{align*}
\end{lemma}
\begin{proof}
Using the update rule on the first line and the definition of $\alpha_t$-strong convexity of $\loss_t$ on the second inequality we have
\begin{alphalign*}
    \|\ptx_{t+1} - \ptx^*\|_2^2 
    &= \|\proj{\ptxset}{\ptx_t - \rate_t \nabla_t}- \ptx^*\|_2^2 \\
    &\alphoverset{\leq} \|\ptx_t - \rate_t \nabla_t - \ptx^*\|_2^2 \\
    &= \|\ptx_t - \ptx^*\|_2^2 + \rate_t^2\|\nabla_t\|_2^2 + 2\innerprod{\ptx^* - \ptx_t, \rate_t \nabla_t} \\
    &\alphoverset{\leq}
    \|\ptx_t - \ptx^*\|_2^2 + \rate_t^2\|\nabla_t\|_2^2 + 2\rate_t
    (\loss_t(\ptx^*) - \loss_t(\ptx_t) + \frac{\alpha_t}{2}\|\ptx_t - \ptx^*\|_2^2)\,, \\
    \Leftrightarrow
    2(\loss_t(\ptx_t) - \loss_t(\ptx^*)) 
    &\alphoverset{\leq}
    \left(\frac{1}{\rate_t}-\alpha_t\right)\|\ptx_t - \ptx^*\|_2^2  - \frac{1}{\rate_t}\|\ptx_{t+1} - \ptx^*\|_2^2
    + \rate_t \|\nabla_t\|_2^2\,, \\
    2\sum_{t\in[T]}(\loss_t(\ptx_t) - \loss_t(\ptx^*))
    &\alphoverset{\leq}
    \sum_{t\in[T]}\left[
    \left(\frac{1}{\rate_t}-\alpha_t - \frac{1}{\rate_{t-1}}\right)\|\ptx_t - \ptx^*\|_2^2
    + \rate_t \|\nabla_t\|_2^2
    \right] \\
    &\quad - \underbrace{\frac{1}{\rate_{T}}\|\ptx_{T+1} - \ptx^*\|_2^2}_{\geq 0} + \frac{1}{\rate_0}\|\ptx_1 - \ptx^*\|_2^2\\
    &\leq \max_t\|\ptx_t - \ptx^*\|_2^2\left(\frac{1}{\rate_T}-\alpha_{1:T} - \frac{1}{\rate_0}\right) + \sum_{t\in[T]}\rate_t\|\nabla_t\|_2^2 + \frac{1}{\rate_0}\|\ptx_1 - \ptx^*\|_2^2\,,
\end{alphalign*}
where 
\alphnextref{} is by the Pythagorean identity since $\mathcal{X}$ is a closed and bounded convex set and $\ptx^* \in\mathcal{X}$,
\alphnextref{} follows from the $\alpha_t$-strong convexity of $\loss_t$,
\alphnextref{} is by dividing by $\rate_t$ and rearranging,
\alphnextref{} is by summing over $t$ and factoring by $\|\ptx_t - \ptx^*\|_2^2$.
\end{proof}

\subsection{Comparison to the Sequential Hindsight-Optimal Learning Rate}\label{apdx:seq-opt}

Recall \cref{def:isotuning_seq} and \cref{thm:isotune}.
For all $T\in\Naturals$, let $x^*_T = \arginf_{x \geq 0}\left\{ x + \sum_{t\in[T]}g_t(x)\right\}$.
Then $M_T(x^*_T) = M^*_T$.
We write $x^*_{1:T} \equiv (x^*_t)_{t\in[T]}$ and, abusing notation, define $M_T(x^*_{1:T}) = x^*_T + \sum_{t\in[T]}g_t(x^*_t)$.
We call $x^*_{1:T}$ the \emph{sequential hindsight-optimal} (inverse) `learning rate.'

It is quite common to use (a variant of) $x^*_{1:T}$ in adaptive learning rates (\eg~\citet{auer2002adaptive,cesabianchi2006prediction,orabona2018solo,pogodin20first,gyorgy2021delays,gaillard2014secondorder} and many others).
We can indeed prove nice and general upper and lower bounds for $M_T(x^*_{1:T})$ (\Cref{thm:x*1:T}), but as we will see later this learning rate is more problematic than the isotuning one in the off-by-one setting (\Cref{ex:x*_issues}) and has the disadvantage of requiring a closed-form solution (see \cref{rmk:closed_form_x*1:T}) --- which is often the case, but not always.

\begin{theorem}\label{thm:x*1:T}
For all $T\in\Naturals$, with $M_T$, $M^*_T$, $x^*_T$, and $M_T(x^*_{1:T})$ defined as above, 
\begin{equation*}
    M^*_T \quad\leq\quad M_T(x^*_{1:T})\quad\leq\quad x^*_T + M^*_T
    \quad \left(\leq 2M^*_T\right)
    \,.
    \qedhere
\end{equation*}
\end{theorem}
\begin{proof}
For the upper bound, 
\begin{align*}
    M^*_T &
    \geq M^*_{T-1} + g_T(x^*_T) 
    \geq M^*_{T-2} + g_{T-1}(x^*_{T-1}) + g_T(x^*_T) 
    \geq \dots \geq \cancel{M^*_0} + 
    \sum_{t\in[T]} g_t(x^*_t) \,,
\end{align*}
hence,
\begin{align*}    
    M_T(x^*_{1:T}) &= x^*_T + \sum_{t\in[T]} g_t(x^*_t)
     \leq  x^*_T + M^*_T \,. 
\end{align*}
And for the lower bound since $g_t$ is non-increasing:
\begin{align*}
    M_T(x^*_{1:T}) = x^*_T + \sum_{t=1}^T g_t(x^*_t) \geq x^*_T + \sum_{t=1}^T g_t(x^*_T) = M^*_T\,.
    &\qedhere
\end{align*}
\end{proof}

More concisely:
\begin{mdframed}
For all $t\in[T]$,
let $g_t:[0,\infty)\to[0,\infty]$ be continuous non-increasing, 
and let $x_t^* = \arginf_{x \geq 0}\{ x + \sum_{s\in[t]}g_s(x)\}$,
then
\begin{align}\label{eq:x*seq_bound}
    \sum_{t\in[T]}g_t(x^*_t)\  \leq\  \inf_{x \geq 0}\left\{ x + \sum_{t\in[T]}g_t(x)\right\}
    \ =\  x^*_T + \sum_{t\in[T]}g_t(x^*_T)\,.
\end{align}
\end{mdframed}
Note that keeping the $\inf$ form may be useful to \emph{choose} a convenient $x$.

\begin{example}
For the regret bound of \cref{eq:regret_intro} in the introduction, 
using the learning rate $x^*_{1:t}\equiv 1/\rate_t$,
with $M_T(x) = \QQ x + \sum_t a_t/x$, 
one can then use \cref{thm:x*1:T}
with $x^*_T = \arginf_{x \geq 0}M_T(x) = \sqrt{\QQ\sum_{t\in[T]} a_t}$
to obtain straightforwardly
(see also the well-known Lemma 3.5 from \citet{auer2002adaptive}),
\begin{align*}
    \regret_T(\ptx^*) \leq M_T(x^*_{1:T}) \leq 
    x^*_T + M^*_T = 
    3 \sqrt{\QQ\sum_{t\in[T]} a_t} \quad (= \tfrac32 M_T(x^*_T))\,.
\end{align*}
This factor 3 is tighter than the factor 4 obtained with \cref{thm:isotune} for the isotuning learning rate $\Delta_T$, but worse than the factor $2\sqrt{2}$ that can be obtained via \cref{lem:isotuning_bound_convex} also for isotuning
(see \cref{ex:refined_bound}),
that is, $M_T(\Delta_T) \leq \sqrt{2}M^*_T$.
It is known, however, that one can also obtain a factor $2\sqrt{2}$ by `manually' rebalancing $x^*_{1:T}$ and taking $\tilde{x}^*_t = x^*_t\sqrt2$ in place of $x^*_t$,
but this rebalancing is dependent on the regret bound.
\end{example}

\begin{example}
When $M^*_T = x^*_T + O(1)$ we can see that $M_t(x^*_{1:T})$ may reach its upper bound of $2M^*_T$ in \cref{thm:x*1:T}.
Indeed this is reached in particular for logarithmic `regret:'
\begin{align*}
    &g_t(x) = e^{-x}, \quad M_T(x) = x + Te^{-x}, \quad M^*_T = \ln T + 1\,, 
    \quad x^*_T = \ln T, \\
    &M_T(x^*_{1:T}) = x^*_T + \sum_{t\in[T]}g_t(x^*_t)
    = \ln T + \sum_{t\in[T]} e^{-\ln t} \geq 2\ln T - O(1) = 2M^*_T - O(1)\,.
    \qedhere
\end{align*}
\end{example}

Hence the sequential hindsight-optimal learning rate and the isotuning learning rate
behave rather similarly for the regret bound of the introduction.
However, in the off-by-one setting,
by contrast to $x^*_{1:T}$,
isotuning ensures that the additive term is $\max_t \delta_t$
(which was obtained using the isotuning-specific property $\Delta_{t-1} = \Delta_t - \delta_t$), which in some circumstances may be nicely bounded (see for example \cref{apdx:adahedge}).
This property of isotuning is what allows us to develop our theory in \cref{sec:generic_ob1} and \cref{sec:null_updates}.
In particular for null updates, we could bound the remaining additive term
$\max_{t\notin\barT} \delta_t \leq c\Delta_{t-1}$ and use \cref{eq:isopoint_deltat_notinbarT} to further bound with a value that is independent of the learning rate.

For the sequential hindsight-optimal learning rate in the off-by-one setting, the matter is not as nice, as shown in the following example.
\begin{example}[$x^*_{1:t}$ off-by-one]\label{ex:x*_issues}
Consider the introductory example of \cref{sec:null_updates},
where for isotuning we have
\footnote{Since $\invrate_0 = 0$, for bounded domains from \cref{eq:generic_deltat_ob1} we can take $\delta_1 \geq \positive{r_1}$, and note that $r_1 \leq \|\ptx_1 - \ptx^*\|_2 \|\nabla_1\|_2 \leq \diam\|\nabla_1\|_2$.}
$\max_t \delta_t \leq \diam\,\|\nabla_t\|_2$.
By contrast, using $x^*_t = \sqrt{\sum_{s\in[t]} a_s}$ (assuming $\QQ=1$)
in the off-by-one setting 
we do not know of a significantly better lemma than the following bound
(see also for example \citet[Eq. (25)]{gaillard2014secondorder}, \citet[Lemma 4.8]{pogodin20first}, \citet{duchi2011adaptive} and others):
\begin{align*}
    \sum_{t\in[T]} \frac{a_t}{\eps + x^*_{t-1}}
    \leq 
    \sum_{t\in[T]} \frac{a_t}{\eps + x^*_{t}} + \max_{t\in[T]} a_t\sum_{t\in[T]}\left(\frac{1}{\eps + x^*_{t-1}} -\frac{1}{\eps + x^*_t} \right)
    \leq 
    \sum_{t\in[T]} \frac{a_t}{x^*_{t}} + \max_{t\in[T]} \frac{a_t}{\eps}\,.
\end{align*}
Combining with \cref{thm:x*1:T} we obtain for $1/\rate_t = \eps + x^*_t =  \eps + \sqrt{\sum_{s\in [t]} a_s}$
for the regret bound of \cref{eq:regret_intro} in the off-by-one-setting
\begin{align*}
\regret_T(\ptx^*) \leq 
\frac{1}{\rate_T} + \sum_{t\in[T]} \rate_{t-1} a_t
\leq 
 \eps + x^*_T + \sum_{t\in[T]}\frac{a_t}{\eps + x^*_{t-1}} \leq
 3\sqrt{\sum_{t\in[T]} a_t}  + \max_t\frac{a_t}{\eps} + \eps \,.
\end{align*}
Hence in this bound the additive term depends on $\eps > 0$ which must be chosen in advance, and optimally should be $\max_t \sqrt{a_t}$ --- which would also give a scale-free bound.
However, this quantity is often unknown in advance.
To track $\eps$ over time, one may have to resolve to restarting the algorithm altogether (\eg \citet{mhammedi2019squint}).
It must be noted \citet[Lemma 3]{orabona2018solo} manage to set $\eps=0$ 
for a particular case of the functions $g_t$, but it remains to be seen whether this result can be extended to arbitrary continuous nonnegative non-increasing $g_t$.
By contrast, isotuning makes it easy to deal with the off-by-one issue, and offsets in general, in particular via the versatility of \cref{thm:isotune}
--- see also \cref{lem:isotuning_offset}.
Consider for example $\Delta_t = \sum_{s\leq t} \delta_t$, where $\delta_t = \min\{g_t(\Delta_{t-1}), b_t\}$, 
with $g_t$ continuous nonnegative non-increasing for all $t\in[T]$, then 
\begin{align*}
    \Delta_T\ \leq\ \inf_{y\geq 0}\left\{y + \sum_{t\in[T]} \min\{g_t(y),\ b_t\}\right\}
    +\max_t b_t\,. 
\end{align*}
Indeed, with $\delta_t = \min\{g_t(\Delta_t-\delta_t), b_t\} =: h_t(\Delta_t)$ and using \cref{thm:isotune} on the isotuning sequence $\tuple{\Delta, h}$,
with a simple change of variables $y = x -\max_t\delta_t$:
\begin{align*}
    \Delta_T &\ \leq\  \inf_{x\geq 0}\left\{x + \sum_{t\in[T]} \min\{g_t(x-\delta_t), b_t\}\right\} 
    \ \leq\ \inf_{x\geq 0}\left\{x + \sum_{t\in[T]} \min\{g_t(x-\max_t\delta_t), b_t\}\right\}\\ 
    &\ =\  \inf_{y\geq -\max_t\delta_t}\left\{y+\max_t \delta_t + \sum_{t\in[T]} \min\{g_t(y), b_t\}\right\}
    \ \leq\ \inf_{y\geq 0}\left\{y + \sum_{t\in[T]} \min\{g_t(y),b_t\}\right\}
    +\max_t b_t\,.
    \qedhere
\end{align*}
\end{example}

\begin{remark}[$x^*_{1:T}$ requires a closed form solution]\label{rmk:closed_form_x*1:T}
Another advantage of isotuning is that the learning rate is updated stepwise with $\Delta_t = \Delta_{t-1} + \delta_t$, rather than calculated explicitly as for $x^*_{1:t}$,
in particular in the off-by-one setting where $\delta_t$ depends on $\Delta_{t-1}$ and not on $\Delta_t$.
While this difference does not matter for most of the cases we have encountered in this paper, it does allow us to provide an adaptive bound for Soft-Bayes (\cref{apdx:soft-bayes}) that does not depend explicitly on the quantity $m$,
by contrast to the original work \citep{orseau2017softbayes} which uses a variant of $x^*_{1:t}$.
To see the difference concretely, consider the following hypothetical off-by-one regret bound, with $\alpha_t > 0$ and $a_t\in[0, 1]$ for all $t$,
\begin{align*}
    \regret_T(\ptx^*) \leq \frac{\QQ}{\rate_T} + \sum_{t\in[T]}(\rate_{t-1})^{\alpha_t} a_t\,.
\end{align*}
Here, because there is no closed form for the optimal constant learning rate in hindsight $x^*_t$ for each $t$, one would have to make strong and likely undesirable assumptions to design an adaptive learning rate based on $x^*_{1:t}$ and analyze the regret, or use a growing number of learning rates.
While a line search is still feasible to obtain an approximate closed-form since \cref{eq:x*seq_bound} is convex,
it may involve up to $t$ terms a each step, and thus may require $\Omega(Nt\log t)$ computation steps at each step $t$ when isotuning requires only $O(N)$.
By contrast, with isotuning one simply needs to set $\delta_t = (\rate_{t-1})^{\alpha_t} a_t= (\QQ/\Delta_{t-1})^{\alpha_t}a_t$, and use \cref{lem:generic_ob1_DeltaT,thm:isotune} to obtain a regret bound that is within a factor 2 of the (unknown!) optimal constant learning rate in hindsight:
\begin{align*}
    \regret_T(\ptx^*) \leq 
    2\Delta_T \leq 2\inf_{y > 0}\left\{y + \sum_{t\in[T]}\left(\frac{q}{y}\right)^{\alpha_t}a_t \right\}
    + 2\max_{t\in[T]} \delta_t\,.
\end{align*}
To bound the additive terms $\max_t \delta_t$, one can either set $\Delta_0 =  \max_t a_t$ if this quantity can be estimated, or use null updates.
\end{remark}

\section{SCALE-FREE FTRL: DETAILED PROOF}\label{apdx:ftrl}

Define $\tilde L_t = \sum_{s\in[t]} \tilde\loss_s$.
When $\tilde L_t=0$, as for SOLO-FTRL~\citep{orabona2018solo} we choose
$\ptx_t = \argmin_{\ptx\in\ptxset} R(\ptx)$.

First we need a few of lemmas.
The first one gives a generic lower bound on $\Delta_T$ when using null updates, 
then we provide upper and lower bounds on $\Delta_T$ for the problem at hand.
Finally, we upper bound the `travel distance' $\|\ptx_t -\ptx^*\|$ before proving the main theorem.

\begin{theorem}[Isotuning lower bound, null updates, $c=1$]\label{thm:generic_DeltaT_lower_null}
Consider the conditions of \cref{thm:generic_DeltaT_null},
but assume that for $t\notin\barT, \delta_t = \hat{g}_t(\Delta_{t-1})$
while for $t\in\barT, \delta_t = \iso_{y\geq0}(\hat{g}_t(y), y)$.
Then for all $T\in\Naturals$,
\begin{equation*}
    \iso_{y\geq0}\left(\sum_{t\in[T]}\hat{g}_t(y),\ y\right)
    \ \leq\ \Delta_T\,.
    \qedhere
\end{equation*}
\end{theorem}
\begin{proof}
Since for $t\in\barT, \delta_t = \iso_{y\geq0}(\hat{g}_t(y),\ y)$,
then by \cref{thm:isopoint},
\begin{align*}
    \min\left\{\hat{g}_t(\Delta_t), \Delta_t\right\}\ \leq\ \delta_t\ \leq\ \max\left\{\hat{g}_t(\Delta_t), \Delta_t\right\}
\end{align*}
and since $\delta_t \leq \Delta_t$ 
(by definition of $\Delta_t$)
then necessarily $\hat{g}_t(\Delta_t) \leq \delta_t$.
Then, recalling that $\hat{g}_t$ is non-increasing,
\begin{align*}
    \Delta_T\ =\ \sum_{t\in[T]}\delta_t
    \ \geq\ \sum_{t\notin\barT} \hat{g}_t\left(\Delta_{t-1}\right) + \sum_{t\in\barT} \hat{g}_t\left(\Delta_t\right)
    \ \geq\  \sum_{t\in[T]} \hat{g}_t\left(\Delta_T\right)\,.
\end{align*}
Using \cref{thm:isopoint} again on the inequality above we obtain
\begin{align*}
    \iso_{y \geq 0}\left(\sum_{t\in[T]}\hat{g}_t\left(y\right),\ y\right) \leq
    \max\left\{\Delta_T, \sum_{t\in[T]} \hat{g}_t\left(\Delta_T\right)\right\} = \Delta_T\,.
    &\qedhere
\end{align*}
\end{proof}

\begin{lemma}[$\Delta_T$ bounds]\label{lem:DeltaT_bounds_ftrl}
For \cref{alg:FTRLnull}, for all $T\in\Naturals, \Delta_T$ satisfies
\begin{align*}
    \sqrt{\frac{\QQ}{2}\sum_{t\in[T]}\|\loss_t\|_*^2}
    \quad \leq\quad \Delta_T
    \quad \leq\quad \sqrt{\QQ\sum_{t\in[T]}\|\loss_t\|_*^2} + \sqrt{\frac{\QQ}{2}}\max_{t\in[T]}\|\loss_t\|_*\,.
    &\qedhere
\end{align*}
\end{lemma}
\begin{proof}
Recalling that $\hat{g}_t(y) = \QQ\|\loss_t\|_*^2/(2y)$,
the lower bound follows by a straightforward application of \cref{thm:generic_DeltaT_lower_null},
while the upper bound follows from \cref{thm:generic_DeltaT_null} (i)
and \cref{lem:Delta_sqrt_short}.
\end{proof}

We now prove the same identity as \citet{orabona2018solo} about the 
maximum travel distance in the off-by-one setting.

\begin{lemma}[Upper bounding $\|\ptx^* - \ptx_t\|$]\label{lem:max_travel}
For \cref{alg:FTRLnull}, for all $\ptx\in\ptxset$,
\begin{align*}
    \|\ptx-\ptx_t\| &\leq \sqrt{2R(\ptx)} + 2\sqrt{2\QQ t}.
    \qedhere
\end{align*}
\end{lemma}
\begin{proof}
Recall that $\tilde L_t = \sum_{s\in[t]} \tilde\loss_s$, and that if $\tilde L_t = 0$, then $\ptx_t = \ptx_1$.
We decompose the update 
$\ptx_{t} = \argmin_{\ptx\in\ptxset} \left\{\innerprod{\tilde L_{t-1}, \ptx} + \frac{R(\ptx)}{\rate_{t-1}}\right\}$
into 
the unconstrained update
$\hat \ptx_t = \argmin_{\ptx\in\Reals^N} \left\{\innerprod{\tilde L_{t-1}, \ptx} + \frac{R(\ptx)}{\rate_{t-1}}\right\}$
and the projection onto $\ptxset$,
$\ptx_t = \argmin_{\ptx\in\ptxset}\breg_R(\ptx, \hat \ptx_t)$.
Hence by the generalized Pythagorean identity
(\eg \citet{herbster1998regressor} and references therein),
for all $\ptx\in\ptxset$,
\begin{align*}
    \breg_R(\ptx, \ptx_t)  \leq \breg_R(\ptx, \hat \ptx_t)\,.
\end{align*}
Furthermore, for all $\ptx\in\ptxset$,
\begin{alphalign*}
    \breg_R(\ptx, \hat \ptx_t) &= R(\ptx) - R(\hat \ptx_t) - \innerprod{\ptx - \hat \ptx_t, \nabla R(\hat \ptx_t)} \\
    &\alphoverset{\leq}  R(\ptx)  + \|\ptx-\hat \ptx_t\|\,\|\nabla R(\hat \ptx_t)\|_* \\
    &\alphoverset{\leq} R(\ptx) + 
    \sqrt{2\breg_R(\ptx, \hat \ptx_t)} \rate_{t-1}\|\tilde L_{t-1}\|_* \,,
\end{alphalign*}
where
\alphnextref{} since $R(\cdot) \geq 0$ by assumption and using H\"older's inequality (generalized to dual norms),
\alphnextref{} using the 1-strong convexity of $R$ with respect to $\|\cdot\|$
and $\nabla R(\hat \ptx_t) = - \rate_{t-1}\tilde L_{t-1}$.
This inequality is of the form $x^2 \leq c + bx$ with $x, b, c \geq 0$, which implies 
$x \leq \sqrt{c} + b$ and thus
\begin{align*}
    \|\ptx-\ptx_t\| &\leq \sqrt{2\breg_R(\ptx, \ptx_t)}
    \leq \sqrt{2\breg_R(\ptx, \hat \ptx_t)}
    \ \leq\ \sqrt{2R(\ptx)} + 2\rate_{t-1}\|\tilde L_{t-1}\|_*\,.
\end{align*}
Finally, 
using Cauchy-Schwarz on the first inequality,
and \cref{ass:generic_isotuning},
\begin{align*}
    \rate_{t-1}^2\|\tilde L_{t-1}\|_*^2 &
    \ \leq\  \rate_{t-1}^2 (t-1)\sum_{s < t}\|\tilde{\loss}_{s}\|_*^2
    \ \leq\ \rate_{t-1}(t-1)\sum_{s < t}\rate_{s-1}\|\tilde{\loss}_{s}\|_*^2
    \ \leq\ \rate_{t-1}(t-1)2\Delta_{t-1} \leq 2\QQ t\,,
\end{align*}
which finishes the proof.
\end{proof}

We are now ready to prove the main result.

\begin{proof}[Proof of \cref{thm:FTRLnull}]
We split the regret in two terms:
\begin{align*}
    \regret_T(\ptx^*)
    &= \sum_{t\in[T]} \innerprod{\ptx_t - \ptx^*, \loss_t}
    \ =\  \sum_{t\in[T]\setminus\barT} \innerprod{\ptx_t - \ptx^*, \loss_t}
    + \sum_{t\in\barT}\innerprod{\ptx_t - \ptx^*, \loss_t}\,.
\end{align*}
The first term corresponds to the regret of FTRL with a time-varying learning rate
running on the pseudo-losses $\tilde{\loss}_t$, as per \cref{thm:ftrl_isotuning}:
\begin{align*}
    \sum_{t\in[T]\setminus\barT} \innerprod{\ptx_t - \ptx^*, \loss_t} =
    \sum_{t\in[T]} \innerprod{\ptx_t - \ptx^*, \tilde{\loss}_t} &\leq 
    \left(1+\frac{R(\ptx^*)}{\QQ}\right)\Delta_T\,,
\end{align*}
and $\Delta_T$ is bounded from above in \cref{lem:DeltaT_bounds_ftrl}.

The second term corresponds to the losses incurred on null updates, which are not accounted for by FTRL on the pseudo losses:
\begin{alphalign*}
    \sum_{t\in\barT}\innerprod{\ptx_t - \ptx^*, \loss_t}
    &\alphoverset{\leq}
    \sum_{t\in\barT}\|\ptx_t - \ptx^*\|\,\|\loss_t\|_* 
    \ \leq\  \max_{t\in\barT}\|\ptx_t - \ptx^*\|
    \sum_{t\in\barT}\|\loss_t\|_* \\
    &\alphoverset{=} \max_{t\in\barT}\|\ptx_t - \ptx^*\|
    \sum_{t\in\barT}\sqrt{\frac{2}{\QQ}}\delta_t
    \ \leq\ \max_{t\in\barT}\|\ptx_t - \ptx^*\|
    \sqrt{\frac{2}{\QQ}}\Delta_{\tau}\\
    &\alphoverset{\leq}
    2\min\{\diam\,, \sqrt{2R(\ptx^*)} + 2\sqrt{2\QQ \tau}\}\|\loss_\tau\|_*\,.
\end{alphalign*}
where
\alphnextref{} follows from H\"older's inequality (generalized to dual norms),
\alphnextref{} by definition of $\delta_t$ for $t\in\barT$,
and
\alphnextref{} from \cref{lem:max_travel} and \cref{thm:generic_DeltaT_null} (ii).

\paragraph{Lower bounding $\|\loss_\tau\|_*$.}
Since $\tau\in\barT$, 
from \cref{eq:isopoint_Deltatau}
we have $\Delta_{\tau-1} \leq \sqrt{\QQ/2}\|\loss_\tau\|_*$, 
hence $\|\loss_\tau\|_*^2 > \frac{2}{\QQ}\Delta_{\tau-1}^2$,
and the claim is proven by using the lower bound in \cref{lem:DeltaT_bounds_ftrl}.
\end{proof}

\section{SCALE-FREE MIRROR DESCENT: DETAILED PROOF}\label{apdx:md}

The proof is very similar to the FTRL case in \cref{apdx:ftrl}.

First we show that $\delta_t$ satisfies \cref{eq:generic_deltat_ob1} for $t\notin\barT$:
\begin{align}
    \delta^*_t = \positive{\frac1{\rate_{t-1}}(\phi_{t'}-\phi_t) + r_t} 
    &= \positive{\frac{1}{\rate_{t-1}}(\breg_R(\ptx^*, \ptx_{t'}) - \breg_R(\ptx^*, \ptx_t))
    + \innerprod{\ptx_t - \ptx^*, \loss_t}} \notag\\
    &\leq \innerprod{\ptx_t - \ptx_{t'}, \loss_t} - \frac{1}{\rate_{t-1}}\breg_R(\ptx_{t'}, \ptx_t) 
    \label{eq:deltat_breg1}\\
    &\leq \|\ptx_t - \ptx_{t'}\|\,\|\loss_t\|_* - \frac{1}{2\rate_{t-1}}\|\ptx_t - \ptx_{t'}\|^2 \notag\\
    &\leq \max_{z\in\Reals}\left\{
    \|\loss_t\|_*z - \frac{1}{2\rate_{t-1}}z^2
    \right\}
    \ = \ \frac{\rate_{t-1}}{2}\|\loss_t\|_*^2\ =\ \delta_t\,.\notag
\end{align}
where on the second line we used the proof of Lemma 2 from \citet{orabona2018solo}, 
where we take $w_t \leadsto \ptx_t, w_{t+1}\leadsto \ptx_{t'}, u\leadsto \ptx^*, \lambda\leadsto 1$ and $\breg_{R_t}(\cdot, \cdot) \leadsto \breg_R(\cdot, \cdot)/\rate_{t-1}$, 
and recalling that the regularizer $R$ is 1-strongly convex with respect to the  norm $\|\cdot\|$,
that is, $\breg_R(a, b) \geq \|a-b\|^2/2$.

Also note that from \cref{eq:deltat_breg1} we can deduce for $t\notin\barT$,
\begin{align}\label{eq:delta_breg2}
    \delta_t \geq 
    \frac{1}{\rate_{t-1}}\breg_R(\ptx_t, \ptx_{t'})\geq 0\,.
\end{align}

\begin{remark}
We could actually take $\delta_t = \frac{1}{\rate_{t-1}}\breg_R(\ptx_t, \ptx_{t'})$
without changing the regret bound.
\end{remark}

Before proving the main result, as for the FTRL case we first bound the travel distance.

\begin{lemma}[{Bounding $\|x_t - x_1\|$}]\label{lem:md_travel}
For \cref{alg:MD_unbounded}, for all $t\in\Naturals$, 
\begin{equation*}
    \|\ptx_{t+1} - \ptx_1\| \leq \sqrt{2\QQ t}\,.
    \qedhere
\end{equation*}
\end{lemma}
\begin{proof}
Using \cref{lem:ocr_convexity} with $f(\ptx) = \|\ptx - \ptx_1\|^2$,
\begin{align*}\resetalph{}
    \Delta_t\|\ptx_{t+1} - \ptx_1\|^2
    &\leq \Delta_{t-1} \|\ptx_{t'} - \ptx_1\|^2\,, \\
    \sqrt{\Delta_t}\|\ptx_{t+1} - \ptx_1\|
    &\alphoverset{\leq} \sqrt{\Delta_{t-1}} (\|\ptx_{t'} - \ptx_t\| + \|\ptx_t - \ptx_1\|) \\
    &\alphoverset{\leq} \sqrt{\Delta_{t-1}} (\sqrt{2\breg_R(\ptx_t, \ptx_{t'})} + \|\ptx_t - \ptx_1\|) \\
    &\alphoverset{\leq} \sqrt{2\QQ\delta_t} + \sqrt{\Delta_{t-1}}\|\ptx_t - \ptx_1\| \\
    &\alphoverset{\leq}
    \sum_{s\in[t]} \sqrt{2\QQ\delta_s}\\
    &\alphoverset{\leq}
    \sqrt{t\sum_{s\in[t]} 2\QQ\delta_s} = \sqrt{2\QQ t\Delta_t}\,,\\
    \|\ptx_{t+1} - \ptx_1\|
    &\leq \sqrt{2\QQ t}\,,
\end{align*}\resetalph{}
where
\alphnextref{} follows by $\Delta_t = \QQ/\rate_t$,
\alphnextref{} by 1-strong convexity of $\breg_R$ with respect to $\|\cdot\|$,
\alphnextref{} 
using \cref{eq:delta_breg2} for $t\notin\barT$
with $\Delta_{t-1} = \QQ/\rate_{t-1}$,
and $\breg_R(\ptx_t, \ptx_t') = 0$ for $t\in\barT$,
\alphnextref{} by induction, 
\alphnextref{} by Cauchy-Schwarz.
\end{proof}

We can now prove the main result.

\begin{proof}[Proof of \cref{thm:MDnull}]
From \cref{thm:generic_regret_null}, since $\phi_{T+1} \geq 0$ we can simplify to
\begin{align*}
    \regret_T(\ptx^*) \leq \left(1+\frac{\phi_1}{\QQ}\right)\Delta_T + \sum_{t\in\barT}r_t\,.
\end{align*}
Define $\hat{g}_t(x) = \QQ\frac{\|\loss_t\|_*^2}{2x}$
and so $\iso_{y > 0}(\hat{g}_t(y),\ y) = \sqrt{\QQ/2}\|\loss_t\|_*$.
Then from \cref{thm:generic_DeltaT_null} (i) and \cref{lem:Delta_sqrt_short} we have
\begin{align*}
    \Delta_T &\leq 
    \hat\Delta_T
    + \sqrt{\frac{\QQ}{2}}\max_{t\in[T]}\|\loss_t\|_*\,,
    &
    \hat\Delta_T &\sqrt{\QQ\sum_{t\in[T]} \|\loss_t\|_*^2}
    \,.
\end{align*}

Since $\delta_t = \sqrt{\QQ/2}\|\loss_t\|_*$ for $t\in\barT$,
we have,
using \cref{thm:generic_DeltaT_null} (ii) on the inequality, 
\begin{align*}
    \sum_{t\in\barT} \|\loss_t\|
    &=\sqrt{\frac{2}{\QQ}} \sum_{t\in\barT} \delta_t 
    \leq 
    2\|\loss_\tau\|_*\,.
\end{align*}
Therefore, using H\"older's inequality we have
\begin{align*}
    \sum_{t\in\barT} r_t 
    &\ =\ \sum_{t\in\barT}\innerprod{\ptx_t - \ptx^*, \loss_t}
    \ \leq\  \max_{t\in\barT}\|\ptx^* - \ptx_t\| \sum_{t\in\barT}\|\loss_t\|_*
    \ \leq\  2\max_{t\in\barT}\|\ptx^* - \ptx_t\| \|\loss_\tau\|_* \\
    &\ \leq\  2\min\{\diam,\ \|\ptx^* - \ptx_1\| + \max_{t\in\barT}\|\ptx_1 - \ptx_t\|\}\|\loss_\tau\|_*\,.  \\
    &\ \leq\ 2\min\{\diam,\ \|\ptx^* - \ptx_1\| + \sqrt{2\QQ\tau}\}\|\loss_\tau\|_*\,.
\end{align*}
where we used \cref{lem:md_travel} on the last line.
Finally, the result follows by putting it all together.
\end{proof}

\begin{example}[Scale-Free Online Gradient Descent, MD style]\label{ex:md_ogd}
When $\breg_R(x, y) = \|x - y\|_2^2/2$ with $\|\cdot\| = \|\cdot\|_*= \|\cdot\|_2$
in a bounded decision set of diameter $\diam$,
then by choosing $\ptx_1$ at the center of $\mathcal{X}$,
using Jung's theorem we have
$\phi_1 \leq (\diam/\sqrt{2})^2/2$, so taking $\QQ = \diam^2/4$ gives,
taking $\loss_t = \nabla_t$,
\begin{align*}
    \regret_T(\ptx^*) \leq
    \diam\left(\sqrt{\sum_{t\in[T]} \|\nabla_t\|_2^2}
      + (2+\tfrac{1}{\sqrt{2}})\max_{t\in[T]}\|\nabla_t\|_2
    \right)
    \,.
\end{align*}
Let us call isoGD this instance of our algorithm.
As far as we are aware, this bound has the best leading factor for bounded-domain OGD with an adaptive learning rate when using a quadratic regularizer.
SOLO-FTRL has a leading factor of $\sqrt{2.75}$,
for AdaFTRL it is $5.3/2$,
while uncentered bounds (including that of Scale-Free Mirror Descent~\citep{orabona2018solo})
have at least a leading factor of $\sqrt{2}$ (\eg \citet{hazan2016oco}).
\citet{mcmahan2017survey} also obtains a centered bound with the optimal leading factor, but it is not adaptive and it is assumed that the largest gradient is known to tune a (non-adaptive) learning rate so as to avoid the off-by-one issue.

For non-quadratic regularizers and for bounded domains, it must be noted that the FreeGrad algorithm~\citep{mhammedi2020freerange} achieves a scale-free regret of $O(\|\ptx^* - \ptx_1\|\sqrt{V_T\ln (\|\ptx^*-\ptx_1\|\sqrt{V_T})} + \diam\max_t \|\loss_t\|)$, which is qualitatively better than our quadratic-regularizer bounds but features a (slowly) growing $\sqrt{\ln T}$ term.
In practice, the exponential component in the FreeRange update does make it move fast toward $\ptx^*$, but it also tends to make it take large strides back and forth once in the vicinity of $\ptx^*$.
Let us take a concrete example.
Assume $\ptxset$ is an unknown but bounded superset of $[-10, 10]$, and
let $\ptx^* = 3$ and $\ptx_1=-10$,
define the loss function $\loss(\ptx)= \ptx^*-\ptx$ if $\ptx < \ptx^*$ and 
$\loss(\ptx) = \exp(\ptx-\ptx^*)-1$ otherwise.
We consider FreeGrad with loss-clipping and restarting
with hints being the current maximum observed absolute gradient~\citep{mhammedi2020freerange}
(but without projection within an increasing domain).
For isoGD we take $\QQ=1$ agnostically so the leading factor of isoGD is  $|\ptx-\ptx^*|^2 = 13^2$.
See the results in \cref{fig:isogd_freerange} after $T=100\,000$.
Despite coming close to $x^*$ very early ($t<20$), FreeGrad incurs a large loss, restarts, moves more carefully then oscillates around $\ptx^*$ with large strides,
even after $10^5$ rounds, incurring a large cumulative regret, both in the loss function and in the distance to $\ptx^*$, even on average.
By contrast, isoGD moves `slowly' toward $\ptx^*$ but once in the vicinity of $\ptx^*$ ($t < 180$) it converges to $\ptx^*$ significantly faster than FreeGrad.
While this is of course just a single example, 
and that without doubt there are examples where FreeGrad significantly outperforms isoGD (say, probably, if $\|\ptx^*-\ptx_1\|$ is very large),
it does demonstrate that FreeGrad is not \emph{strictly} better than isoGD even for bounded domains.
\end{example}

\begin{figure}[htbp!]
    \centering
\includegraphics[width=0.49\textwidth]{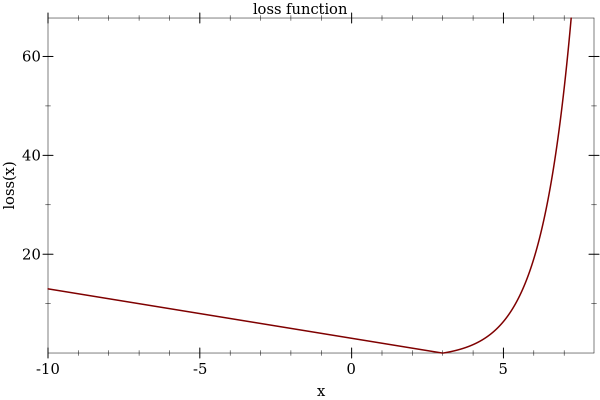}
\includegraphics[width=0.49\textwidth]{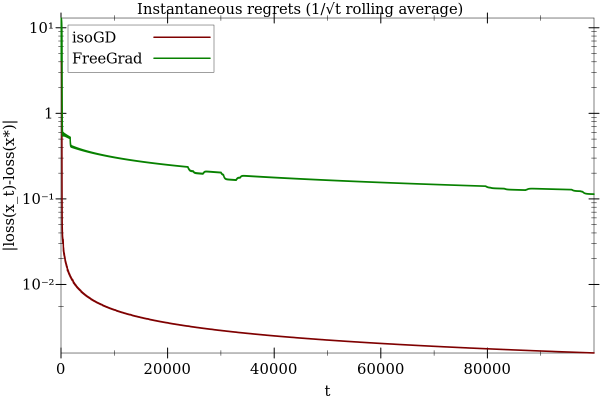}
\includegraphics[width=0.49\textwidth]{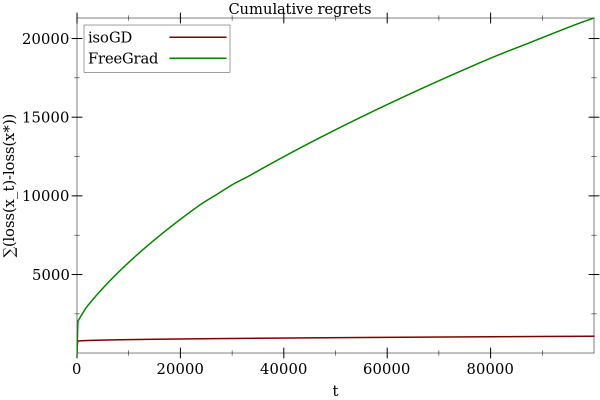}
\includegraphics[width=0.49\textwidth]{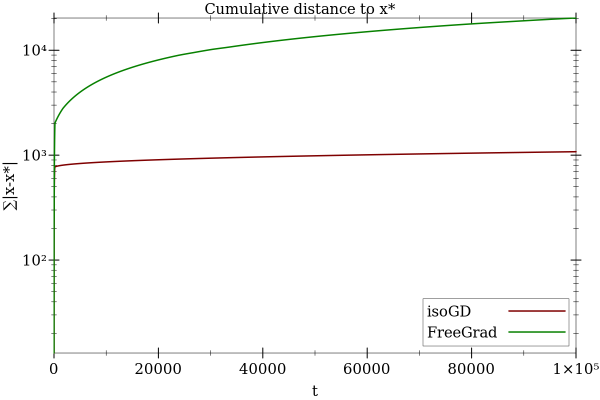}
\includegraphics[width=0.49\textwidth]{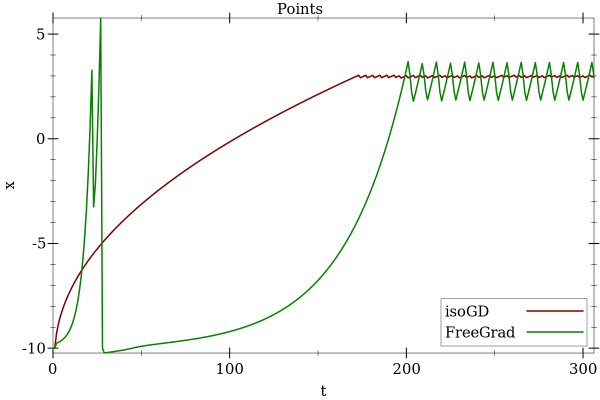}
\includegraphics[width=0.49\textwidth]{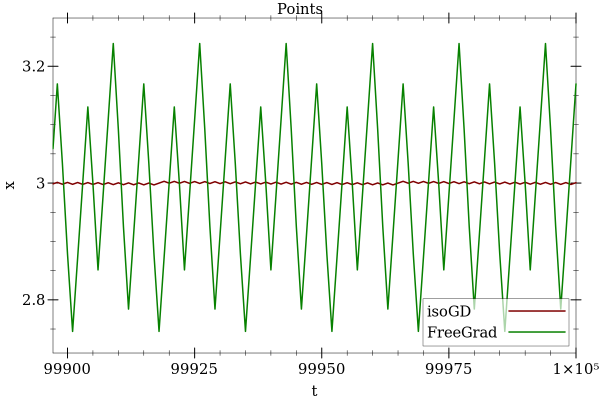}
\includegraphics[width=0.49\textwidth]{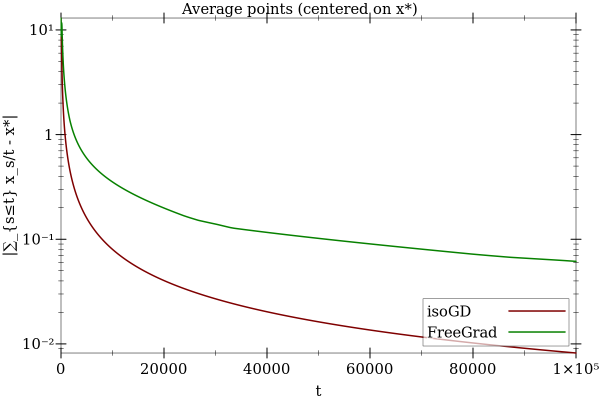}
    \caption{Behaviour of isoGD ($\QQ=1$) and FreeRange on performing gradient descent 
    on $\loss(\ptx)= \ptx^*-\ptx$ if $\ptx < \ptx^*$ and $\loss(\ptx) = \exp(\ptx-\ptx^*)-1$ otherwise.
    Note the log scales on some of the y-axis.}
    \label{fig:isogd_freerange}
\end{figure}

\begin{example}[MD on the simplex]\label{ex:md_simplex}
When the Bregman divergence is the the relative entropy $\breg_R(a, b) = \sum_{i\in[N]} a_i \ln (a_i/b_i)$, and $\ptxset=\probsimplex{N}$ is the probability simplex,
then 
$\|\cdot\| = \|\cdot\|_1$ and $\|\cdot\|_* = \|\cdot\|_{\infty}$,
$\phi_1 \leq \ln N$, 
and $\diam = 2$
so taking $\QQ = \ln N$ gives
\begin{align*}
    \regret_T(\ptx^*) \leq
    2\sqrt{\ln(N) \sum_{t\in [T]} \|\loss_t\|_\infty^2}
        + (4+ \sqrt{2\ln N})\max_{t\in[T]}\|\loss_t\|_\infty\,.
\end{align*}
See also \cref{apdx:adahedge} for a more adaptive so-called second order bound with excess losses like AdaHedge~\citep{derooij2014follow} but with a Mirror-Descent-style projection, and a loss translation lemma on the simplex
(\cref{lem:simplex_translate_losses}), which applies to the bound above.
\end{example}

\section{ISOML-PROD DETAILED PROOF}\label{apdx:mlu-prod}

\begin{proof}[Proof of \cref{thm:mlu-prod}]
Recall \cref{eq:mluprod_regret_simple}:
\begin{align*}
    \regret_{i, T} = \sum_{t\in[T]}(\bar\loss_t - \loss_{i, t}) \leq \left(1+\frac{\ln \ptx_{i, T+1}}{\ln N}\right)\Delta_{i, T}\,.
\end{align*}
We proceed to bound $\ptx_{i, T+1}$ and $\Delta_{i, T}$ in turn.

\paragraph{Bounding $\ptx_{i, T+1}$.}
For $t\notin\barT$, combining the online correction \cref{eq:generic_ob1_ocr}
with the update $\ptx_{i, t'}$ of \cref{eq:mlprod},
\begin{align}
    \sum_{j\in[N]} \ptx_{j, t+1} 
    &=
    \sum_{j\in[N]}\left[\ptx_{j, t'}
    \frac{\Delta_{j,t-1}}{\Delta_{j, t}}
    + \frac{\delta_{j, t}}{\Delta_{j, t}}\ptx_{j, 1}\right] \notag\\
    &=
    \sum_{j\in[N]}\left[\ptx_{j, t}(1+\rate_{j, t-1}(\bar\loss_t - \loss_{j,t}))
    \frac{\Delta_{j,t-1}}{\Delta_{j, t}}
    + \frac{\delta_{j, t}}{\Delta_{j, t}}\ptx_{j, 1}\right] \notag\\
    &\leq 
    \sum_{j\in[N]}\left[\ptx_{j, t}(1+\rate_{j, t-1}(\bar\loss_t - \loss_{j,t}))
    + \frac{\delta_{j, t}}{\Delta_{j, t}}\ptx_{j, 1}\right] \notag\\
    &= 
    \sum_{j\in[N]}\ptx_{j, t}
    + \underbrace{\sum_{j\in[N]}\ptx_{j, t}\rate_{j, t-1}\bar\loss_t - \sum_{j\in[N]}\ptx_{j, t}\rate_{j, t-1}\loss_{j,t}}_{=0 \text{ by \cref{eq:mlprod_barloss}}}
    + \sum_{j\in[N]}\frac{\delta_{j, t}}{\Delta_{j, t}} \ptx_{j, 1} \notag\\
    &\leq \sum_{j\in[N]}\ptx _{j, t} + \sum_{j\in[N]}\frac{\delta_{j, t}}{\Delta_{j, t}}\ptx_{j, 1}\,.
    \label{eq:mlprod_sum_xt_4281}
\end{align}
Observe that \cref{eq:mlprod_sum_xt_4281} also holds straightforwardly also for $t\in\barT$ since $\ptx_{t'}=\ptx_t$---this would not be the case if null updates were not performed on all coordinates at once.
Hence, by summing over $t\in[T]$ and telescoping we have, using $\ptx_{j, 1} = 1$,
\begin{align*}
    \sum_{j\in[N]}\ptx_{j, T+1} &\leq N + \sum_{j\in[N]} \sum_{t\in[T]} \frac{\delta_{j, t}}{\Delta_{j, t}}\,.
\end{align*}
(Note that if $\Delta_{j, t}=0$, then also $\delta_{j,t}=0$ and we take 
$\delta_{j, t}/\Delta_{j, t-1} = 0$, which is validated by \cref{eq:generic_ob1_ocr} since we have $\ptx_{t+1} = \ptx_{j, t'}$.)
Define $C_{j, T} = \sum_{t\in[T]} \frac{\delta_{j, t}}{\Delta_{j, t}}$, then 
\begin{align*}
    \ptx_{i, T+1} \leq \sum_{j\in[N]} \ptx_{j, T+1} \leq N+N\max_{j\in[N]}C_{j, T}\,.
\end{align*}

\paragraph{Bounding $\Delta_{i, T}$.}
Recall that $\tau = \max \barT$ is the last step where a null update is performed, and note that $\tau\geq 1$ when $T \geq 1$ unless the losses are 0.
Let $\mu\in[N]$ be such that $\rate_{\mu, \tau-1}|\bar\loss_t - \loss_{\mu, t}| > 1/2$ ($\mu$ necessarily exists by definition of $\barT$).
This implies that $\invrate_{\mu, \tau-1} < 2S$,
and thus $\Delta_{\mu, \tau-1} < 2S\ln N$.
Since null updates are performed on all coordinates at the same time with $\delta_{\cdot,t} = s_t$ for $t\in\barT$, then for all $i\in[N]$ we have
\begin{align}\label{eq:deltait_88}
    \sum_{t\in \barT} \delta_{i, t}
    \ =\ \sum_{t\in \barT} \delta_{\mu, t}
    \ \leq\ \Delta_{\mu, \tau}\ =\ \Delta_{\mu, \tau-1} + \delta_{\mu, \tau}\ \leq\ 2S\ln N + S\,.
\end{align}
Now let us focus on a particular coordinate $i$ (and we omit the index $i$ in $\hat{g}$ and other new definitions below). Define
\begin{align*}
    \forall t \notin\barT: \quad \forall y > 0:\hat{g}_t(y) &= \frac{\QQ r_{i, t}^2}{2y}\,,\quad \forall y \leq 0: \hat{g}_t(y) = \infty\,,&
    \forall t \in\barT: \hat{g}_t(.) &= 0\,.
\end{align*}
Let us define 
$\check\Delta_t$ to be like $\Delta_{i,t}$ but subtracting the terms $\delta_{i,\cdot}$ of the null update steps:
\begin{align*}
    \check\Delta_t &= \check\Delta_{t-1} + \delta_{i,t}\indicator{t\notin\barT}\,.
\end{align*}
Let $A_t = \sum_{s\in\barT, s\leq t }\delta_{i,s}$,
then $\check\Delta_t = \Delta_{i,t} - A_t$.
We proceed to express $\delta_{i, t}\indicator{t\notin\barT}$ as a function of $\check\Delta_t$.
For $t\in\barT$, we have $\delta_{i,t}\indicator{t\notin\barT} =0 = \hat{g}_t(\check\Delta_t)$.
For $t\notin\barT$, using \cref{cor:rate_log_approx} on the definition of $\delta_{i,t}$
(\cref{eq:isomlprod_deltat})
for $t\notin\barT$,
\begin{align*}
    \delta_{i,t} &\leq \frac{r_{i,t}^2/2}{\frac1{\rate_{t-1}}-|r_{i,t}|}
    \ =\  \frac{qr_{i,t}^2/2}{\Delta_{i, t-1}-q|r_{i,t}|} \\
    &=\  \hat{g}_t(\Delta_{i, t-1}-\QQ |r_{i, t}|)
    \ =\  \hat{g}_t(\Delta_{i, t}-\delta_{i,t}-\QQ |r_{i, t}|)
    \ =\ \hat{g}_t(\check\Delta_t+ A_t - \delta_{i,t} - \QQ |r_{i, t}|)
    \ \leq\ \hat{g}_t(\check\Delta_t - \delta_{i,t} - \QQ |r_{i, t}|)\,,
\end{align*}
where the last inequality is because $\hat{g}$ is monotone decreasing.
Therefore, by \cref{lem:isotuning_offset}, with the isotuning sequence $\tuple{\hat\Delta, \hat{g}}$,
\begin{equation*}
    \check\Delta_T \leq \hat\Delta_T + \max_{t\in[T]}\{\delta_{i,t} + \QQ |r_{i, t}|\}\,.
\end{equation*}
Then, using \cref{lem:Delta_sqrt_short} on $\hat\Delta_T$,
\begin{align*}
    \Delta_{i,T}& = \check\Delta_T + A_T \leq \sqrt{\QQ \sum_{t\notin\barT} r_{i, t}^2}
    + \max_{t\notin\barT}\{\delta_{i,t} + \QQ |r_{i, t}|\} + A_T
    \ \leq\ \sqrt{V_{i, T}\ln N}
    + S/5 + S\ln N  + 2S \ln N + S
\end{align*}
where $A_T$ is bounded in \cref{eq:deltait_88}, and 
$\delta_{i,t} = r_{i, t}\left(1-\frac{1}{u}\ln(1+u)\right)$,
with $u = \rate_{i, t-1}r_{i, t}$ and $|u| \leq 1/2$,
 can be bounded by $r_{i, t}/5\leq S/5$
since then $1-\frac1u\ln(1+u) \leq 1/5$.

\paragraph{Finally.}
We can now put it all together:
\begin{align*}
    \regret_{i, T} &\leq
    \left(
        1+\frac{\ln \left(N + N\max_{j\in[N]}C_{j, T}\right)}{\ln N}
    \right)\left(
        \sqrt{V_{i, T}\ln N} + 3S\ln N + \tfrac65 S
    \right) \\
    &\leq 
    \left(
        2+\frac{\ln \left(1 + \max_{j\in[N]}C_{j, T}\right)}{\ln N}
    \right)\left(
        \sqrt{V_{i, T}\ln N} + S(\tfrac65+3\ln N)
    \right)\,.
\end{align*}
Let $t_0 = \min\{t\in[T]:\Delta_{j, t} > 0\}$, then
for all $\theta\in\{t_0, \dots, T\}$,
\begin{align*}
    C_{j, T} = 
    \sum_{t\in[T]} \left(1-\frac{\Delta_{j, t-1}}{\Delta_{j, t}}\right)
     = \sum_{t=t_0+1}^T \left(1-\frac{\Delta_{j, t-1}}{\Delta_{j, t}}\right)
     = \sum_{t=t_0+1}^\theta \underbrace{\left(1-\frac{\Delta_{j, t-1}}{\Delta_{j, t}}\right)}_{\leq 1}
     + \sum_{t=\theta}^T \ln \frac{\Delta_{j, t}}{\Delta_{j, t-1}}
     \leq \theta - t_0 + \ln \frac{\Delta_{j, T}}{\Delta_{j, \theta}}
\end{align*}
where we used $1-1/x \leq \ln x$ and telescoping.
Hence,
\begin{align*}
    C_{j, T} \leq \min_{\theta\in\{t_0, \dots, T\}} \theta - t_0 + \ln \frac{\Delta_{j, T}}{\Delta_{j, \theta}}\,.
\end{align*}
From this we can deduce all the following simultaneously:
\begin{align*}
    C_{j, T} &\leq T\,, &&\text{with }\theta = T\,, \\
    C_{j, T} &\leq o(1) + \ln ST  &&\text{with } \theta = \argmin_{t\in[T]}\left\{\Delta_{j, t} \geq \frac{\Delta_{j, T}}{ST}\right\}
    \text{since } \Delta_{j, T} = O(\sqrt{ST}\ln T)\,, \\
    &&&\text{then } \frac{\Delta_{j, T}}{ST} = o(1), \text{ and thus } \theta = t_0 + o(1)\,, \\
    C_{j, T} &\leq \tau + \ln \frac{ST}{s_\tau}
    &&\text{ with } \theta = \tau = \max \barT
    \text{ since } \Delta_{j, \tau} \geq \delta_{j, \tau} = s_\tau \\
    &&&\text{ and } \Delta_{j, T} \leq ST \text{ since } \delta_{j, t} \leq S \text{ for all } t \in [T]\,.
    \qedhere
\end{align*}
\end{proof}

\begin{remark}\label{rmk:boa}
With a slight modification of \cref{alg:mlu-prod}, we can obtain a variant of 
the Bernstein Optimal Aggregation (BOA) algorithm \citep{wintenberger2016bernstein} (see \cref{foot:boa} in main text)
with online correction and null updates and prove the same bound as for isoML-Prod under the same constraints.
Replace \cref{eq:mlprod} with
\begin{align*}
    \forall i\in[N]: \ptx_{i, t'} 
    &= \ptx_{i, t}f(\rate_{i, t-1}r_{i, t})
    &\text{ with }
    f(y) &= \exp\left(y-\frac{y^2/2}{1-|y|}\right)\,.
\end{align*}
Recall that $\phi_i(\ptx) = -\ln \ptx_{i}$ and $\ptxset = [0, \infty)^N$.
To satisfy \cref{eq:generic_deltat_ob1} we take
\begin{align*}
    \delta_{i, t} = \delta^*_{i, t} = \positive{\frac{1}{\rate_{i, t-1}}(\phi_{i, t'} - \phi_{i, t}) + r_{i, t}}
    = \rate_{i, t-1}\frac{r_{i, t}^2/2}{1-\rate_{i, t-1}|r_{i, t}|}\,.
\end{align*}
We also need to show that \cref{eq:mlprod_sum_xt_4281} still holds:
We do this by reducing it to the isoML-Prod case by using $f(y) \leq 1+y$ for $|y| < 1$
(from \cref{lem:log_approx})
with $y=\rate_{i, t-1}r_{i, t}$;
though as for isoML-Prod we enforce $|y| \leq 1/2$ to bound $\max_t \delta_{i,t} \leq S$
(which is slightly worse than $S/5$ for isoML-Prod).
This constraint means that we need to perform null updates in the same conditions as for isoML-Prod, and we also still need null updates to be performed on all coordinates at the same time.
Therefore the bound of \cref{thm:mlu-prod} holds also for this variant of BOA.
Note that \cref{lem:log_approx} also saves a factor $\sqrt{2}$ compared to prior work
where $\exp(y - y^2)$ is used instead of $\exp(y-y^2/(2(1-|y|)))$.
\end{remark}

\begin{remark}\label{rmk:isoprod_x*}
It is straightforward to adapt the proof to allow for different initial weights $\ptx_{i, 1}$ that may depend on $i$, as well as extending to a growing number of experts~\citep{mourtada2017growing}.
For example, if we can ensure that $\sum_{i\in[N]} \ptx_{i, 1} \leq 1$ (and still $\ptx_{i, 1}\geq 0$), then we would have $\phi_{i,1} = \ln\frac{1}{\ptx_{i, 1}}$ (instead of 0 in \cref{eq:mluprod_regret_simple})
but from \cref{eq:mlprod_sum_xt_4281} we have by telescoping
\begin{align*}
    \ptx_{i, T+1} \leq \sum_{j\in[N]} \ptx_{j, 1} + \sum_{j\in[N]} \ptx_{j, 1} C_{j, T}
    \leq 1 + \innerprod{\ptx_1, C_{\cdot,T}}
\end{align*}
and setting $\QQ_i = \ln \frac{1}{\ptx_{i, 1}} = \phi_{i, 1}$ we obtain the regret bound,
using $\regret_T(\ptx^*) = \innerprod{\ptx^*,(\regret_{i, T})_{i\in[N]}}$ on the simplex,
\begin{align*}
    \regret_T(\ptx^*) &\leq 
    \sum_i \ptx^*_i
    \left(
        1+\frac{\phi_{i, 1}}{\QQ_i}+\frac{\ln (1 + \innerprod{\ptx_1, C_{\cdot,T}})}{\QQ_i}
    \right)\left(
        \sqrt{\QQ_i V_{i, T}} + 3S\QQ_i + 2S
    \right)  \\
    &\leq 
     \left(
        2+\frac{\ln (1 + \innerprod{\ptx_1, C_{\cdot,T}})}{\min_i \ln\frac{1}{\ptx_{i, 1}}}
    \right)\left(
        \sqrt{\text{H}(\ptx^*, \ptx_1)\sum_i \ptx^*_i V_{i, T}} + 3S\text{H}(\ptx^*, \ptx_1) + 2S
    \right) \,,
\end{align*}
where $\text{H}(\ptx^*, \ptx_1) = \sum_i \ptx^*_i\ln\frac{1}{\ptx_{i, 1}}$ is the cross-entropy of $\ptx^*$ and $\ptx_1$ (and is equal to the relative entropy $\RE{\ptx^*}{\ptx_1}$ plus the constant entropy term $\text{H}(\ptx^*, \ptx^*)$),
and we used Cauchy-Schwarz on $\sum_i \ptx^*_i \sqrt{\QQ_i V_{i, T}} = \sum_i \sqrt{\ptx^*_i \QQ_i} \sqrt{\ptx^*_i V_{i, T}}
\leq \sqrt{\sum_i \ptx^*_i \QQ_i}\sqrt{\sum_i \ptx^*_iV_{i, T}}$.
Finally, note that $C_{j, T}$ is independent of $\QQ_j$.
We recover exactly the bound of \cref{thm:mlu-prod} when $\ptx_{i, 1} = 1/N$.
\end{remark}

\begin{lemma}\label{lem:log_approx}
For all $x\in(-1, 1)$, 
\begin{equation*}
    \ln(1+x) \geq x - \frac{x^2}{2(1-|x|)}\,.
    \qedhere
\end{equation*}
\end{lemma}
\begin{proof}
For $x \in [0, 1)$, define $g_+(x) = \tfrac12x^2 + (1-x)(\ln(1+x) - x)$.
Then $g_+'(x) = x - (\ln(1+x) - x) + \frac{1-x}{1+x} -(1-x) \geq 0$,
hence $g_+$ is increasing on $[0, 1)$ and since $g_+(0)=0$ then $g_+(x) \geq 0$ for all $x \in[0, 1)$.

Similarly, for $x \in (-1, 0]$, define $g_-(x) = \tfrac12 x^2 + (1+x)(\ln(1+x) - x)$.
Then $g_-'(x) = x + (\ln(1+x) - x) + 1 -(1+x)=\ln(1+x)-x \leq 0$
and since $g_-(0)=0$ then $g_-(x) \geq 0$ for all $x\in(-1, 0]$, which concludes the proof.
\end{proof}

{
\renewcommand{\themytheorem}{\ref{lem:log_approx}a}
\begin{corollary}\label{cor:rate_log_approx}
For all $\rate x\in(-1, 1)$ with $\rate > 0$,
\begin{align*}
    x-\frac1{\rate}\ln(1+\rate x) \ \leq\  \frac{x^2/2}{\frac1{\rate}- |x|}\,.
    &\qedhere
\end{align*}
\end{corollary}
\begin{proof}
Follows straightforwardly from \cref{lem:log_approx}.
\end{proof}
}
\addtocounter{mytheorem}{-1}
\section{ISOHEDGE: ADAHEDGE, MIRROR-DESCENT STYLE}\label{apdx:adahedge}

As mentioned before, the present paper is building upon the balancing of the regret idea of \citet{derooij2014follow} for AdaHedge.
We show how to build a variant of AdaHedge that uses Mirror Descent-style projection with online correction based on \cref{alg:generic_ob1} without null updates,
and straightforwardly recover the same regret bound with our tools, with a slight generalization.

We assume $\ptxset = \probsimplex{N}$.
Recall that the loss of the algorithm is $\bar\loss_t = \sum_{i\in[N]}\ptx_{i, t}\loss_{i, t}$
with $\loss_t\in\Reals^N$
and the instantaneous regret is $r_t = \innerprod{\ptx_t - \ptx^*, \loss_t}$.
Define, for some unknown $T\in\Naturals$,
\begin{align*}
    L & = \max_{t\in[T], i\in[N]} |\loss_{i, t}-m_t|\,,&
    S &= \max_{t\in[T], i\in[N]} |\bar\loss_t-\loss_{i, t}|\,,  \\
    u_t &= \sum_i \ptx_{i, t} (\loss_{i, t}-m_t)^2\,, & 
    U_T &= \sum_{t\in[T]} u_t\,, &
    v'_t &= \sum_i \ptx_{i, t}(\bar\loss_t-\loss_{i, t})^2\,, &
    V'_t &= \sum_{t\in[T]} v'_t\,.
\end{align*}
\citet{derooij2014follow} prove the scale-free bound of AdaHedge by first analyzing the regret for $[0, 1]$ losses and then showing that the algorithm outputs predictions that are invariant to any rescaling of the losses.
\begin{theorem}[Theorem 6, \citet{derooij2014follow}, rescaled]
The regret of AdaHedge is bounded by
\begin{align*}
    \regret_{i, T} \leq 2\sqrt{V'_T \ln N} + S'\left(2+\tfrac43\ln N\right)\,.
\end{align*}
where $S' = \max_{t\in[T]}(\max_{i\in[N]}\loss_{i, t} - \min_{j\in[N]}\loss_{j, t})$.
\end{theorem}
This regret bound has the nice property of being invariant to any translation of the losses.
However, surprisingly, it turns out that on the simplex non-translation invariant bounds have better properties since they can always be made translation invariant,
while the converse may not hold:
\begin{lemma}[Loss translation on the simplex]\label{lem:simplex_translate_losses}
Let $\ptxset = \probsimplex{N}$.
For all $\ptx^*\in\ptxset$, if an algorithm $\mathcal{A}$ sequentially predicts $\ptx_t$ and has a regret of the form 
\begin{align*}
    \regret_T(\ptx^*) = \sum_t \innerprod{\ptx_t-\ptx^*, \loss_t}\ \leq\ f(\loss_1, \loss_2, \dots)
\end{align*}
for all sequences of losses $\loss_1, \loss_2, \dots\in\Reals^N$ then,
for all sequentially available $m_t\in\Reals$,
algorithm $\mathcal{A}$ fed with the translated losses $\tilde\loss_t = \loss_t - m_t$
achieves the regret
\begin{align*}
    \regret_T(\ptx^*)= \sum_t \innerprod{\ptx_t-\ptx^*, \loss_t} \leq f(\loss_1-m_1, \loss_2-m_2, \dots)\,.&\qedhere
\end{align*}
\end{lemma}
\begin{proof}
Simply observe that on the simplex the regret is translation invariant:
$\innerprod{\ptx_t - \ptx^*, \loss_t -m_t} = \innerprod{\ptx_t - \ptx^*, \loss_t}$.
\end{proof}
In particular when $m_t = \bar\loss_t$ the regret bound becomes translation invariant.
See also \citet{chen2021impossible} for a  coordinate-wise generalization of this property and additional discussion of its benefits.

We now describe the isoHedge algorithm.
We use \cref{ass:generic_isotuning}, and we take and $\phi(\ptx) = \RE{\ptx^*}{\ptx}$.
Define the update rule, for all $i\in[N]$ and for all $t\in\Naturals$:
\begin{align*}
    \ptx_{i, 1} &= \frac{1}{N}\,,\\
    \ptx_{i, t'} &= \ptx_{i, t}\frac{\exp(-\rate_{t-1}\loss_{i, t})}{\sum_{j\in[N]} \ptx_{j, t}\exp (-\rate_{t-1}\loss_{j, t})}&
    \text{(AdaHedge update and `projection', MD style)}
\end{align*}
and we also use the online correction of \cref{eq:generic_ob1_ocr}.
Null updates are unnecessary in this case because the update is always valid 
and $\delta_t$ will always be bounded by the largest loss.
Let $\QQ = \ln N$, and 
we define $\delta_t$ as tightly as possible from \cref{eq:generic_deltat_ob1}:
\begin{align}\label{eq:deltat_isohedge}
    \delta_t  &= \delta^*_t = \positive{\frac{1}{\rate_{t-1}}\ln\frac{\ptx_{i, t}}{\ptx_{i, t'}}+ r_t}
    =
    \frac{1}{\rate_{t-1}}\ln\left(\sum_{j\in[N]} \ptx_{j, t}\exp(-\rate_{t-1}\loss_{j, t})\right)
    + \bar\loss_t\,.
\end{align}

See \cref{alg:adahedge/oc}.

\begin{algorithm}[htbp!]
\begin{lstlisting}
def isoHedge():
  $\ptx_1 = \left(\frac{1}{N}, \dots, \frac{1}{N}\right)$
  $\Delta_0 = 0$
  for t = 1, 2, ...:
    predict $\ptx_t$ ; observe $\loss_t$
    choose $m_t\in\Reals$ # can be $\lstcommentcolor{\bar\loss_t}$
    $\rate = \QQ/\Delta_{t-1}$  # $\lstcommentcolor{\rate_{t-1}}$
    $Z = \sum_{j\in[N]} \ptx_{j, t}\exp(\rate(m_t - \loss_{j, t}))$
    for $i\in[N]$:     
      $\ptx_{i, t'} = \frac{1}{Z}\ptx_{i, t}\exp(\rate(m_t - \loss_{i, t}))$ # AdaHedge update, MD style, projection = normalization

    $\delta_t = \frac{1}{\rate}\ln Z + \bar\loss_t$  # AdaHedge's mixability gap
    $\Delta_t = \Delta_{t-1} + \delta_t$ # isotuning
    $\ptx_{t+1} = \ptx_{t'}\frac{\Delta_{t-1}}{\Delta_t} + \frac{\delta_t}{\Delta_t}\ptx_1$  # online correction
\end{lstlisting}
\caption{Instance of \cref{alg:generic_ob1} for AdaHedge with Mirror-Descent-style projection, online correction and isotuning (no null updates).
Note that if $\Delta_{t-1}=0$, the limit $\invrate_{t-1}\to0$ gives
$\delta_t = \max_{i\in[N]} (\bar \loss_t - \loss_{i, t})$, and
the online correction implies that $\ptx_{t+1} = \ptx_1$.
}
\label{alg:adahedge/oc}
\end{algorithm}

\begin{theorem}\label{thm:adahedge/oc}
For all sequentially chosen $\{m_t\}_t\in\Reals^T$ the regret of isoHedge in \cref{alg:adahedge/oc} after $T$ rounds is at most
\begin{equation*}
    \regret_T(\ptx^*) \leq 
    2\sqrt{U_t \ln N} + 2S+\tfrac23L\ln N\,.
\end{equation*}
In particular with $m_t = \bar\loss_t$, the regret is bounded by
\begin{equation*}
    \regret_T(\ptx^*) \leq 
    2\sqrt{V'_t \ln N} + 2S+\tfrac23S\ln N\,.
    \qedhere
\end{equation*}
\end{theorem}
\begin{proof}
We prove for $m_t=0,\forall t$ and then apply \cref{lem:simplex_translate_losses}.

As usual, we need to express $\delta_t$ as a function of $\Delta_t$.
Define $\hat{g}_t(x) = \QQ u_t/(2x)$ for $x > 0$ and $\hat{g}_t(x) = \infty$ otherwise.
From \cref{lem:exp_denom}, we know that $e^x \leq 1 + x + \frac{x^2/2}{1-x/3}$
for $x < 3$ and so, with $\ln x \leq x-1$, 
\begin{align*}
    \delta_t 
    &\leq
    \frac{1}{\rate_{t-1}}\left(-1+\sum_{j\in[N]} \ptx_{j, t}\exp(-\rate_{t-1}\loss_{j, t})\right)
    + \bar\loss_t\\
    &\leq
    \frac{1}{\rate_{t-1}}
    \left(-1+\sum_j \ptx_{j, t}\left(1-\rate_{t-1}\loss_{j, t} + \frac{\rate_{t-1}^2\loss_{j, t}^2/2}{1+ \rate_{t-1}\loss_{j, t}/3}\right)\right)
    +\bar\loss_t\\
    &= 
    \sum_j \ptx_{j, t}\frac{\rate_{t-1}\loss_{j, t}^2/2}{1+ \rate_{t-1}\loss_{j, t}/3}\\
    &\leq
    \frac{\rate_{t-1} u_t/2}{1-\rate_{t-1} L/3}
    \ =\  \frac{\QQ u_t/2}{\Delta_{t-1} - \QQ L/3}
    \ =\  \frac{\QQ u_t/2}{\Delta_t - \delta_t - \QQ L/3}
\end{align*}
which is valid at least when $\Delta_t - \delta_t - \QQ L/3 > 0$,
and thus for all $t\in[T]$ we have $\delta_t \leq \hat{g}_t(\Delta_t - \delta_t - \QQ L/3)$.
Therefore, using \cref{lem:isotuning_offset},
let $\tuple{\hat\Delta, \hat{g}}$ be an isotuning sequence, then
using \cref{lem:Delta_sqrt_short},
\begin{align*}
    \Delta_T\ \leq\ \hat\Delta_T + \max_{t\in[T]}\{\delta_t + \QQ L/3\}
    \ \leq\  \sqrt{\QQ \sum_{t\in[T]} u_t} + S + \QQ L/3
\end{align*}
where $\delta_t \leq S$ is from \cref{eq:deltat_isohedge} (by bounding the average by the maximum).
The result follows from \cref{lem:generic_ob1}, since with $\phi_1 =\ln N = \QQ$ we have $\regret_{i, T} \leq 2\Delta_T$.
\end{proof}

Observe that by contrast to the original analysis of AdaHedge, our analysis applies directly to losses in $\Reals^N$.
Furthermore, the only application-specific difficulty is to bound $\delta_t$.

\begin{remark}
One difference between isoHedge (when $m_t=\bar\loss_t$) and AdaHedge is that for $t=2$ AdaHedge predicts at a corner of the probability simplex like Follow The Leader (FTL)~\citep{derooij2014follow},
while isoHedge predicts at the `center' $(1/N, 1/N, \dots)$.
This latter behaviour is also the choice made by SOLO FTRL~\citep{orabona2018solo},
which deals with both bounded and unbounded domains---where playing at a corner is not always possible or even desirable.
In particular, when losses are gradients, playing at a corner may lead to unbounded or infinite losses, \eg
$\loss_{i, t} = \frac{\partial (-\log \innerprod{\ptx_t, f_t})}{\partial \ptx_{i, t}}$ for some $f_t$.
\end{remark}

The following results are technical lemmas used in the proof of \cref{thm:adahedge/oc}. A similar result can be extracted from the proofs of \citet{derooij2014follow}.
\begin{lemma}\label{lem:exp_quadratic_tight}
For all $x\in\Reals$, $(e^x - x - 1)(1-x/3) \leq x^2/2$.
\end{lemma}
\begin{proof}
Let 
\begin{align*}
    f(x) &= (e^x - x - 1)(1-x/3) - x^2/2 \\
         &= e^x(3-x)/3 + (x+1)(x/3-1) - x^2/2\,,\\
    f'(x) &= e^x(2-x)/3 -x/3 -2/3\,, \\
    f''(x) &= e^x(1 -x)/3 -1/3\,, \\
    f'''(x) &= -xe^x/3\,.
\end{align*}
Notice that $f(0)=f'(0) =f''(0) = 0$ and all functions above are continuous on $\Reals$.
Therefore a Taylor expansion at 0 with Lagrange remainder gives $f(x) = -x^3ye^y/3$
for some $y$ between $x$ and 0 and thus $f(x) \leq 0$ for all $x \in\Reals$,
which proves the result.
\end{proof}

\begin{corollary}\label{lem:exp_denom}
For all $x < 3$, $e^x \leq 1 + x + \frac{x^2/2}{1-x/3}$.
\end{corollary}
\begin{proof}
Follows from \cref{lem:exp_quadratic_tight}.
\end{proof}

\section{SOFT-BAYES REVISITED}\label{apdx:soft-bayes}

Soft-Bayes~\citep{orseau2017softbayes} is an algorithm for the universal portfolio problem~\citep{Cov91}, with running time $O(N)$ per step for $N$ experts and a regret bound of $O(\sqrt{NT\ln N})$.
A variant is also given with a better bound of $O(\sqrt{mT\ln N})$ where $m$ is the number of experts in hindsight that have been `good' at least once after $T$ steps.
While not changing the computational complexity of the algorithm, this variant requires keeping track of these $m$ experts online with an array of size $N$, thus making the algorithm less simple, and making the bound dependent on the specifics of the tracking algorithm.
At the core, the reason for this explicit dependency on $m$ is because a variant of the
`optimal learning rate in hindsight' is used instead of isotuning (see \cref{apdx:seq-opt}).

We show how isotuning can be used to obtain a regret bound of $O(\sqrt{mT\ln N})$ \emph{without} having to keep track of $m$, just by using the prescribed isotuning learning rate.
We do not need null updates as we can bound the additive term $\max_t \delta_t = O(\ln(NT))$ directly.

We are in the online convex optimization setting,
where $\ptxset$ is the probability simplex of $N$ coordinates,
and $\loss_t(\ptx) = -\ln \innerprod{\ptx, p_t}$
where $p_t \in [0, \infty)^N$.

Let 
\begin{align}\label{eq:softbays_N_T}
    \mathcal{N}_T &\in \argmin_{S\subseteq[N]}\left\{|S| \ \middle|\ \forall t\in[T]:
S\cap\argmax_{i\in[N]}\{p_{i, t}\} \neq \emptyset\right\}
\end{align}
be one of the smallest subsets of $[N]$ such that at each step $t$ at least one 
of the best experts at step $t$ is in this subset.

The core update of the Soft-Bayes algorithm is as follows
(note that the \emph{notation} for the index of the learning rate is offset by one 
compared to the notation of the original paper, to be aligned with the notation of the current paper).
For all $i\in[N]$ and $t\in[T]$, 
\begin{align}
    \ptx_{i, 1} &= \frac1N\,,\notag\\
    \bar p_t &= \innerprod{\ptx_t, p_t}\,, 
    &\text{(mixture prediction)} \notag\\
    \ptx_{i, t'} &= \ptx_{i, t}\left(1-\rate_{t-1} + \rate_{t-1} \frac{p_{i, t}}{\bar p_t}\right)\,, 
    &\text{(Soft-Bayes update)} \label{eq:soft_bayes_update}
\end{align}
The online correction of \cref{eq:generic_ob1_ocr} is also applied (as in the original work), but not null updates.
See \cref{alg:soft-bayes}, and note that by contrast to the original algorithm, it does not require keeping track of $|\mathcal{N}_T|$.

\begin{algorithm}
\begin{lstlisting}
def isoSoft-Bayes($N\geq 2$):
  $\Delta_0 = 2\ln N$
  $\ptx_1 = \{\frac1N, \frac1N, \dots, \frac1N\}$
  for $t = 1, 2, \dots$:
    predict $\ptx_t$; observe $p_t$  # loss is $\lstcommentcolor{-\ln \bar p_t}$
    $\bar p_t = \innerprod{\ptx_t, p_t}$
    $\rate = \frac{\ln N}{\Delta_{t-1}}$ # $\lstcommentcolor{\rate_{t-1}}$
    $\forall i\in[N]: \ptx_{i, t'} = \ptx_{i, t}\left(1-\rate+\rate\frac{p_{i, t}}{\bar p_t}\right)$ # Soft-Bayes update
    $\delta_t = \max_{i\in[N]}\ln\left(1+\frac{\rate}{1-\rate}\frac{p_{i, t}}{\bar p_t}\right)$
    $\Delta_{t} = \Delta_{t-1} + \delta_t$ # isotuning
    $\ptx_{t+1} = \ptx_{t'}\frac{\Delta_{t-1}}{\Delta_t} + \frac{\delta_t}{\Delta_t} \ptx_1$ # online correction
\end{lstlisting}
\caption{The Soft-Bayes algorithm with isotuning with regret $O(\sqrt{m T \ln N})$ where $m= |\mathcal{N}_T|$ is the (unknown) number of `good' experts.
Null updates are not used as the additive term $\max_t \delta_t$ is well bounded.}
\label{alg:soft-bayes}
\end{algorithm}

Define for all $t\in[T]$ (note that this slightly differs from \cref{ass:generic_isotuning}),
\begin{align}\label{eq:softbayes_defs}
    \QQ &= \ln N\,, &
    \Delta_t &= \frac{\QQ}{\rate_t}\,, &
    \Delta_0 &= 2\QQ\,, &
    \Delta_t &= \Delta_{t-1} + \delta_t\,, &
    &\delta_t = 
    \max_{i\in[N]}\ln\left(1+\frac{\rate_{t-1}}{1-\rate_{t-1}}\frac{p_{i, t}}{\bar p_t}\right)\,.
\end{align}  

Our proof will make use of the central lemma of the original paper:
\begin{lemma}[Lemma 4, \citet{orseau2017softbayes}]\label{lem:softbayes}
For all  $\ptx\in\ptxset, b\in[0,\infty)^N$, and $\rate\in(0, 1)$,
\begin{equation*}
    \ln \sum_{i\in[N]} \ptx_i b_i 
    - \frac{1}{\rate}\sum_{i\in[N]} \ptx_i \ln(1-\rate + \rate b_i)
    \ \leq\ \max_{i\in[N]} \ln \left(1+\frac{\rate}{1-\rate}b_i\right)\,.
    \qedhere
\end{equation*}
\end{lemma}

Recall that $\regret_T(\ptx^*) = \sum_{t\in[T]} (\ln \innerprod{\ptx^*, p_t}
- \ln \innerprod{\ptx_t, p_t})$.

\begin{theorem}\label{thm:soft-bayes}
The regret of isoSoft-Bayes (\cref{alg:soft-bayes}) for all $\ptx^*\in\ptxset$ is bounded for all $T\in\Naturals$
by
\begin{align*}
    \regret_T(\ptx^*) \leq 
    \ \min_{m\in \{|\mathcal{N}_T|,\dots\ N\}}
    \ 4\sqrt{m T \ln N}
    \ +\  2\ln(m\sqrt{T}\ln N)
    \ +\ 8 m\ln\frac{N}{m}\ +\ 4\ln N
    \ +\ O(1)\,,
\end{align*}
where $\mathcal{N}_T$ is defined in \cref{eq:softbays_N_T}.
\end{theorem}

We need a few intermediate results to prove the theorem.
The proofs are in \cref{apdx:softbayes_lemma_proofs}.
First, we need to make sure that this definition of $\delta_t$ can be used with \cref{lem:generic_ob1}.

\begin{lemma}[$\delta_t \geq \delta^*_t$]\label{lem:softbayes_deltat_valid}
The definition of $\delta_t$ satisfies \cref{eq:generic_deltat_ob1}.
\end{lemma}

Now we bound $\delta_t$ as a function of $\Delta_{t-1}$.

\begin{lemma}[$\delta_t$ as a function of $\Delta_{t-1}$]\label{lem:softbayes_deltat_Deltat-1}
For all $\mathcal{N}'_T$ such that $\mathcal{N}_T\subseteq \mathcal{N}'_T\subseteq [N]$,
let $m = |\mathcal{N}'_T|$,
let $a = m\ln (N/m)$,
and let $b_t = \sum_{i\in \mathcal{N}'_T} \ln \frac{\ptx_{i, t+1}}{\ptx_{i, t}}$.
Then for all $t\in[T]$
such that $\Delta_{t-1} > a$,
\begin{align*}
    \delta_t &\leq 
    \left(1+\frac{a}{\Delta_{t-1}-a}\right)\left(
    b_t + \frac{m\QQ}{\Delta_{t-1}-\QQ}
    \right)\,.
    \qedhere
\end{align*}
\end{lemma}

Then, we bound the term $\max_{t\in[T]} \delta_t$ that will appear similarly to \cref{lem:generic_ob1_DeltaT}.

\begin{lemma}[Bounding $\max_t \delta_t$]\label{lem:softbayes_max_deltat}
For all $T\in\Naturals$, for all $t\in[T]$,
\begin{align*}
    \delta_t\ \leq\ \ln\left(1+2\Delta_T\right) + q\,.
    &\qedhere
\end{align*}
\end{lemma}

We can now prove the theorem.

\begin{proof}[Proof of \cref{thm:soft-bayes}]
Define, for all $t\in[T]$,
with the notation of \cref{lem:softbayes_deltat_Deltat-1}:
\begin{align*}
    \forall x>0:\ \hat{g}_t(x) &= \left(1+\frac{a}{x}\right)\left(
    b_t + \frac{m\QQ}{x}\right)\,, &
    \text{and }\forall x\leq 0:\ \hat{g}_t(x) &= \infty\,.
\end{align*}
Define $\Delta_t' = \Delta_t - \Delta_0 = \sum_{s\in[t]}\delta_s$.
To use \cref{lem:isotuning_offset},
we now express $\delta_t$ as a function
\footnote{We could actually also define $\hat g_0(x)=\Delta_0$ and use \cref{lem:isotuning_offset} directly on $\Delta_t$, at the expense of an additional $\Delta_0 =2q$.}
of $\Delta'_t$.
Note that
for the case $\Delta_{t-1} \leq a$ we have $\hat{g}_t(\Delta_{t-1}-a) = \infty\geq \delta_t$, therefore,
from \cref{lem:softbayes_deltat_Deltat-1},
for all $t\in[T]$, 
\begin{align*}
    \delta_t &\ \leq\ \hat{g}_t(\Delta_{t-1} -\max\{a, \QQ\}) \\
    &\ =\ \hat{g}_t(\Delta_t - \delta_t - \max\{a, \QQ\}) \\
    &\ =\ \hat{g}_t(\Delta'_t - \delta_t - \max\{a, \QQ\} + \Delta_0)\,.
\end{align*}
Let $\tuple{\hat\Delta, \hat{g}}$ be an isotuning sequence, then
by \cref{lem:isotuning_offset} where we take  $\Delta_t\leadsto  \Delta'_{t}$,
\begin{align}
    &&\Delta_T - \Delta_0 = \Delta'_T &\leq \hat\Delta_T + \max_{t\in[T]}\positive{\delta_t +\max\{a, \QQ\}-\Delta_0} \notag\\
    &\Leftrightarrow& \Delta_T &\leq 
    \hat\Delta_T + \max\Big\{\max_{t\in[T]}\delta_t +\max\{a, \QQ\},\ \Delta_0\Big\} \notag\\
    &&&\leq \hat\Delta_T + \max\{\ln(1+2\Delta_T) + q + \max\{a,q\}, 2q\}\notag\\
    &&&\leq \hat\Delta_T + \ln(1+2\Delta_T) + a + 2q\,,
    \label{eq:softbayes_DeltaT_hatDeltaT}
\end{align}
where we used \cref{lem:softbayes_max_deltat} on the third line.
By \cref{thm:isotune} (see \cref{eq:isotuning_upper_bound}),
with $B_T = \sum_{t\in[T]} b_t$,
\begin{align*}
    \hat\Delta_T \ \leq\ \inf_{y > 0}
    \left\{
    y + \sum_{t\in[T]}
    \left(1+\frac{a}{y}\right)
    \left(b_t + \frac{m\QQ}{y}\right)
    \right\}
    \ = \ \inf_{y > 0}
    \left\{
    y + 
    \frac{aB_T}{y} +\frac{m\QQ T}{y} +
    \frac{am\QQ T}{y^2}
    \right\}
    + B_T
    \ \leq\  2\sqrt{m\QQ T} + 3a \,,
\end{align*}
where the last inequality is because
by telescoping we have $B_T = \sum_{i\in \mathcal{N}'_T} \ln \frac{\ptx_{i, T+1}}{\ptx_{i, 1}} \leq m\ln\frac{N}{m} = a$,
and choosing $y =\max\{a, \sqrt{m\QQ T}\}$.
Finally, combining with \cref{eq:softbayes_DeltaT_hatDeltaT} gives:
\begin{align*}
    \Delta_T \leq 2\sqrt{m T \ln N} + \ln(1+2\Delta_T) + 
    4 m\ln\frac{N}{m} + 2\ln N\,,
\end{align*}
where $\ln(1+\Delta_T) \leq \ln(m\sqrt{T}\ln N) +O(1)$,
and the result follows by $\regret_T(\ptx^*)\leq 2\Delta_T$ from \cref{lem:generic_ob1}
(for which \cref{lem:softbayes_deltat_valid} satisfies a condition)
with $\QQ = \ln N \geq \phi_1 = \RE{\ptx^*}{\ptx_1}$.
\end{proof}

\begin{remark}[Dealing with the case $\Delta_{t-1} \leq a$]
(i) One may want to avoid the case $\Delta_{t-1} \leq a$ altogether
 by initializing $\Delta_0 = a+2q$ for example,
 but recall that $a = m\ln (N/m)$ is unknown to the algorithm, 
so it is not possible to initialize $\Delta_0$ in this way.
(ii) Setting $\Delta_0 = N \ln N > a$ is possible, but leads to an additive
$N\ln N$ term in the regret bound, which is not desirable for cases where $m \ll N$.
(iii) The next intuitive idea is to find the first step $\tau$ such that $\Delta_{\tau-1} > a$,
then use $\Delta_T = \Delta_{\tau-1} + \sum_{t=\tau}^T \delta_t$,
use $\Delta_{\tau-1}=\Delta_{\tau-2}+\delta_{\tau-1} \leq a + \max_t \delta_t$, and then proceed to bound $\sum_{t=\tau}^T \delta_t$ --- which would make bounding $\sum_{t=\tau}^T b_t$ difficult as it is not bounded by $ a$ in general, by contrast to $\sum_{t=1}^T b_t = B_T$.
(iv) But observe instead how `vacuously' bounding $\delta_t \leq \infty$ when $\Delta_{t-1} \leq a$ simplifies the whole process!
\end{remark}

\begin{remark}
It is possible to take $m < |\mathcal{N}_T|$ if some of the `good' experts
are good only a few times, and bound the corresponding $\delta_t$ outliers separately,
since $\delta_t \leq \ln O(NT)$.
\end{remark}

\begin{remark}
The leading factor 4 is likely due to using the convenient \cref{thm:isotune}, but it seems plausible that a more careful analysis would remove a factor $\sqrt{2}$.
\end{remark}

\subsection{Proofs of the Lemmas}\label{apdx:softbayes_lemma_proofs}

\begin{proof}[Proof of \cref{lem:softbayes_deltat_valid}]
Define $\phi(\ptx) = \sum_{i\in[N]}\ptx^*_i\ln\frac{\ptx^*_i}{\ptx_i}$
and thus $\phi_1 \leq \ln N = q$.
\begin{align*}
    \delta^*_t 
    &= \positive{\frac1{\rate_{t-1}}\sum_{i\in[N]}\ptx^*_i\ln\frac{\ptx_{i, t'}}{\ptx_{i, t}}
    + \ln\frac{\innerprod{\ptx^*, p_t}}{\bar p_t}} \\
    &= \positive{-\frac{1}{\rate_{t-1}}
    \sum_{i\in[N]}\ptx^*_i\ln\left(1-\rate_{t-1} + \rate_{t-1} \frac{p_{i, t}}{\bar p_t}\right)
    +
    \ln \sum_{i\in[N]}\ptx^*_i\frac{p_{i, t}}{\bar p_t}} \\
    &\leq \max_{i\in[N]}\ln\left(1+\frac{\rate_{t-1}}{1-\rate_{t-1}}\frac{p_{i, t}}{\bar p_t}\right) = \delta_t
\end{align*}
where we used \cref{lem:softbayes}
with $\ptx\leadsto\ptx^*$ and $b_i\leadsto\frac{p_{i,t}}{\bar p_t}$
on the last line.
\end{proof}

\begin{proof}[Proof of \cref{lem:softbayes_deltat_Deltat-1}]
The following holds for every $\mathcal{N}'_T\subseteq [N]$ such that $\mathcal{N}_T \subseteq \mathcal{N}'_T$.
Let $m = |\mathcal{N}'_T|$ 
and $a = m\ln \frac N m$.

Now, we seek to bound $\delta_t$ as a function of $\Delta_t$.
From \cref{lem:ocr_convexity} and the concavity of the log function, then using the update rule \cref{eq:soft_bayes_update} on the second line
along with $\ln(1+\rate - \rate x)= \ln(1+\frac{\rate}{1-\rate}x) +\ln(1-\rate)$, for all $i\in[N]$ we have:
\begin{align*}
    && \Delta_t \ln \ptx_{i, t+1} &\geq \Delta_{t-1}\ln \ptx_{i, t'} + \delta_t \ln \ptx_{i, 1} \\
    &\Leftrightarrow&
    \Delta_{t-1} \ln \ptx_{i, t+1} + \delta_t\ln \ptx_{i, t+1}
    &\geq  \Delta_{t-1}\ln \ptx_{i, t} +  \Delta_{t-1}\ln\left(1+\frac{\rate_{t-1}}{1-\rate_{t-1}}\frac{p_{i, t}}{\bar p_t}\right)
    + \Delta_{t-1}\ln(1-\rate_{t-1}) + \delta_t\ln \ptx_{i, 1} \\
    &\Leftrightarrow&
    \ln\left(1+\frac{\rate_{t-1}}{1-\rate_{t-1}}\frac{p_{i, t}}{\bar p_t}\right)
    &\leq
    \ln \frac{\ptx_{i, t+1}}{\ptx_{i, t}}
    +\frac{\delta_t}{\Delta_{t-1}}\ln\frac{\ptx_{i, t+1}}{\ptx_{i, 1}}
    -\ln(1-\rate_{t-1})
    \,, 
\end{align*}
where the last line follows by dividing both sides by $\Delta_{t-1}$ and reorganizing.
Observe that
$-\ln(1-\rate_{t-1}) = \ln(1+ \frac{\rate_{t-1}}{1-\rate_{t-1}}) \leq \frac{\rate_{t-1}}{1-\rate_{t-1}} = \frac{\QQ}{\Delta_{t-1}-\QQ}$,
and recall that $\Delta_{t-1} - q > 0$.
Hence,
using the definitions of $\delta_t$, of $\mathcal{N}'_T$ and of $\mathcal{N}_T$,
and loosely upper bounding $\delta_t$:
\begin{align}
    \delta_t\ \leq\ \sum_{i\in \mathcal{N}'_T}
    \ln\left(1+\frac{\rate_{t-1}}{1-\rate_{t-1}}\frac{p_{i, t}}{\bar p_t}\right)
    \ \leq 
    \ 
    \underbrace{\sum_{i\in \mathcal{N}'_T}
    \ln \frac{\ptx_{i, t+1}}{\ptx_{i, t}}}_{b_t}
    + 
    \frac{\delta_t}{\Delta_{t-1}} 
    \underbrace{\sum_{i\in \mathcal{N}'_T}\ln\frac{\ptx_{i, t+1}}{\ptx_{i, 1}}}_{\leq m\ln\frac{N}{m} = a}
    + m \frac{\QQ}{\Delta_{t-1}-\QQ}\,. \label{eq:sb_deltat_122}
\end{align}
Then 
\begin{align*}
    \delta_t\ \leq\ b_t + \frac{\delta_t}{\Delta_{t-1}}a + m \frac{q}{\Delta_{t-1}-q}
\end{align*}
that is,
\begin{align*}
    \delta_t\left(1-\frac{a}{\Delta_{t-1}}\right)
    &\leq  b_t + \frac{m\QQ}{\Delta_{t-1}-\QQ}\,.
\end{align*}
and since $1- a/\Delta_{t-1} > 0$  by assumption, the result follows by rearranging.
\end{proof}

\begin{proof}[Proof of \cref{lem:softbayes_max_deltat}]
Recall the definition of $\delta_t$ in \cref{eq:softbayes_defs}.
Using $\bar p_t \geq \ptx_{i, t} p_{i, t}$ for all $i$,
and $\rate_t /(1-\rate_t) \leq 2\rate_t$ (since $\rate_0=\Delta_0/\QQ = \tfrac12$), then 
\begin{align*}
    \delta_t 
    \leq \max_i \ln\left(1+\frac{\rate_{t-1}}{1-\rate_{t-1}}\frac{1}{\ptx_{i, t}}\right)
    \leq \max_i \ln\left(1+2\frac{\QQ}{\Delta_{t-1}}\frac{1}{\ptx_{i, t}}\right)
\end{align*}
Furthermore, since $\max_i p_{i,t} / \bar p_t \geq 1$, then
$\delta_t \geq \ln(1+\rate_{t-1}/(1-\rate_{t-1}))\geq \rate_{t-1}$.
Also, from the online correction \cref{eq:generic_ob1_ocr},
multiplying by $\Delta_t$ on both side, we obtain
\begin{align}\label{eq:sb_Deltax}
    \Delta_t\ptx_{i, t+1}\ =\ \Delta_{t-1}\ptx_{i, t'} + \delta_t\ptx_{i, 1}
    \ \geq\ \delta_t\ptx_{i, 1}\ \geq\ \rate_{t-1}\ptx_{i, 1} 
    \ =\ \frac{\QQ}{N\Delta_{t-1}}\,,
\end{align}
and by reordering we have
$\frac{\QQ}{\Delta_t\ptx_{i,t+1}} \leq N\Delta_{t-1}\,,$
and so also $\frac{\QQ}{\Delta_{t-1}\ptx_{i,t}} \leq N\Delta_{t-2}$.
Then
\begin{align*}
    \delta_t\ \leq\ \ln\left(1+2N\Delta_{t-2}\right)\ \leq\ \ln(1+2\Delta_T) + \ln N\,.
    &\qedhere
\end{align*}
\end{proof}

\section{STRENGTHENING EXISTING RESULTS}\label{apdx:campolongo}

\citet[Lemma 6.1]{campolongo2020temporal} prove the following lemma (slightly reformulated, but equivalent):
\begin{lemma}[{\citet[Lemma 6.1]{campolongo2020temporal}}]
Let $T\in\Naturals$, and $\Delta_0 = 0$c.
For all $t\in[T]$,
let $a_t \geq 0$ and $\Delta_t = \Delta_{t-1} + \delta_t$
with $\delta_t \leq \min\{b a_t, ca_t^2/(2\Delta_{t-1})\}$.
Then $\Delta_T \leq \sqrt{(b^2+c)\sum_{t\in[T]}a_t^2}$.
\end{lemma}
The original proof is simple enough, but the result is a somewhat loose.
Using the isotuning tools, we can readily tighten the bound to
$\Delta_T \leq \sqrt{c\sum_{t\in[T]}a_t^2} + \max_{t\in[T]}ba_t$
--- see \cref{lem:campolongo_tighter} below.
With $A= \max_t a_t$,
this improves the bound from $\Delta_T = O((b+\sqrt{c})A\sqrt{T})$
to $\Delta_T = O(A\sqrt{cT} + bA)$.

Let us take this opportunity to generalize the bound as well
(note that the meanings of $a_t$ and $b_t$ differ from above).
\begin{lemma}\label{lem:campolongo_tighter}
Let $T\in\Naturals$, and $\Delta_0 = 0$.
For all $t\in[T]$,
let $a_t, b_t \geq 0$ and $\Delta_t = \Delta_{t-1} + \delta_t$
with $\delta_t \leq \min\{b_t, a_t/\Delta_{t-1}\}$.
Then $\Delta_T \leq \sqrt{2\sum_{t\in[T]}a_t} + \max_{t\in[T]}b_t$.
\end{lemma}
\begin{proof}
For all $t\in[T]$, define $\hat g_t(x) = a_t/x$ for $x > 0$ 
and $\hat g_t(x) = \infty$ for $x\leq 0$,
and let $\tuple{\hat\Delta, \hat g}$ be an isotuning sequence.
Then since $\delta_t \leq \hat g_t(\Delta_{t-1})$,
by \cref{lem:generic_ob1_DeltaT},
$\Delta_T \leq \hat\Delta_T + \max_{t\in[T]} \delta_t$.
The result follows by using \cref{lem:Delta_sqrt_short} on $\hat\Delta_T$,
and the upper bound on $\delta_t$.
\end{proof}

Similarly, \Cref{lem:campolongo_tighter} can also be used to improve the analysis of AdaFTRL \citep[Appendix B]{orabona2018solo} --- without changing the algorithm --- to
\begin{align*}
    \regret_T(\ptx^*)\ \leq\ (1+R(\ptx^*))\left(\sqrt{\frac1{\lambda}\sum_{t\in[T]} \|\loss_t\|_*^2} + \diam \max_{t\in[T]} \|\loss_t\|_*\right)
\end{align*}
instead of the original bound
\begin{align*}
    \regret_T(\ptx^*)\ \leq\ \sqrt{3}(1+R(\ptx^*))\max\left\{\diam, \frac1{\sqrt{2\lambda}}\right\}\sqrt{\sum_{t\in[T]} \|\loss_t\|_*^2}\ .
\end{align*}
Indeed, the authors prove that $\regret_T(\ptx^*)\leq (1+R(\ptx^*))\Delta_T$
with 
\begin{align*}
    \Delta_t \leq \Delta_{t-1} + \min\left\{\diam\|\loss_t\|_*, \ \frac{\|\loss_t\|_*^2}{2\lambda}\right\}
\end{align*}
so, now, applying \cref{lem:campolongo_tighter} gives the result immediately.

\clearpage
\section*{Table of Notation}\label{sec:notation_table}

\begin{tabular}{l|l}
Notation & Meaning \\
\hline
$\Nonnegints$ & \{0, 1,\dots\} \\
$\Naturals$ & \{1, 2,\dots\} \\
$[n]$ & $\{1, 2, \dots, n\}$ \\
$\positive{x}$ & $\max\{0, x\}$ \\
$\indicator{test}$ & indicator function \\
$\iso(\cdot, \cdot) $ & iso operator (\cref{sec:isotuning})\\
$T$ & some (unknown) number of rounds to play\\
$t$ & some round index, usually $t \in[T]$\\
$N $     & number of coordinates/experts\\
$\ptxset $   & convex subset of $\Reals^N$ \\
$\probsimplex{N}$ & probability simplex with $N$ coordinates \\
$\ptx^* $     & some competitor point in $\ptxset$ \\
$\ptx_t $     & point in $\ptxset$ played by the algorithm \\
$\ptx_{t'} $     &  algorithm-specific update of $\ptx_t$ before online correction\\
$\loss_t(\ptx)$ & OCO: loss function at round $t$\\
$\nabla_t$ & OCO: $\nabla \loss_t$ gradient of the loss function at round $t$\\
$\loss_t$ &  OLO: in $\Reals^N$, loss vector at round $t$; corresponds to $\nabla_t$ in OCO\\
$\loss_{i, t}$ &  in $\Reals$, loss of coordinate/expert $i$ at round $t$\\
$\bar\loss_t$  & $\sum_i \ptx_{i, t}\loss_{i, t}$ loss of the expert algorithm at round $t$\\
$G $ & OCO: $\max_{t\in[T]} \|\nabla_t\|_*$, largest observed gradient\\
$L $ & OLO: $\max_{t\in[T]} \|\loss_t\|_*$, largest observed loss \\
$L_t$ & FTRL: $\sum_{s\in[t]} \loss_s$, cumulative loss\\
$\regret_T(\ptx^*) $     & cumulative regret of the algorithm compared to $\ptx^*\in\ptxset$ \\
$r_t $     &  instantaneous regret of the algorithm, usually $\innerprod{\ptx_t-\ptx^*, \loss_t}$ (OLO) or $\loss_t(\ptx_t) - \loss_t(\ptx^*)$ (OCO)\\
$r_{i, t}$ & $\bar\loss_t - \loss_{i, t}$ instantaneous regret of the algorithm against expert $i$ at round $t$ \\
$s_t$ & $\max_{i\in[N]} |r_{i, t}|$ \\
$S$ & $\max_{t\in[T]} s_t$ \\
$V_{i, T}$ & $\sum_{t\in[T]} r_{i, t}^2$  \\ 
$\diam $     & $\max_{(a, b)\in\ptxset^2} \|a-b\|$, diameter of $\ptxset$ according to $\|\cdot\|$\\
$\rate$ & learning rate, $\geq 0$\\
$\QQ$ & positive parameter of the algorithms\\
$\Delta_t $     & $\Delta_0 + \sum_{s\in[t]}\delta_s$; usually $\Delta_0 = 0$ \\
$\delta^*_t$    & algorithm-specific quantity, \cref{eq:delta*t}\\
$\delta_t $     & algorithm-specific quantity, usually satisfies \cref{eq:generic_deltat_ob1} \\
$\hat{g}_t(\cdot)$ & non-increasing continuous positive function such that usually $\delta_t \leq \hat{g}_t(\Delta_{t-1})$ \\
$\hat\Delta_t$ & defined such that $\tuple{\hat\Delta,\hat{g}}$ is an isotuning sequence \\
$\|\cdot\| $     & some norm\\
$\|\cdot\|_*$     & dual norm of $\|\cdot\|$ \\
$R $     &  regularizer, usually 1-strongly convex w.r.t. some norm $\|\cdot\|$ and non-negative \\
$\breg_R $     & Bregman divergence \\
$\RE{\cdot}{\cdot} $     & relative entropy \\
$\phi $     & some convex function \\
$\barT $ & subset of $[T]$ where a null update is performed (\cref{sec:null_updates})\\
$\tau $ & $\max \barT$, last round before $T$ a null update is performed (\cref{sec:null_updates})\\
\end{tabular} 
\end{document}